\documentclass[11pt]{article}

\usepackage[utf8]{inputenc} 

\pdfoutput=1

\usepackage{algorithm}
\usepackage[noend]{algpseudocode}
\usepackage{amsmath}
\usepackage{amssymb}
\usepackage{amsthm,amsfonts}
\usepackage{array}
\usepackage{authblk}
\usepackage{babel}
\usepackage{booktabs}
\usepackage{caption}
\usepackage{colortbl}
\usepackage{enumitem}
\usepackage[T1]{fontenc}    
\usepackage{framed}
\usepackage{fullpage}
\usepackage{graphicx}
\usepackage[colorlinks, linkcolor=red, anchorcolor=blue, citecolor=blue]{hyperref}
\usepackage{indentfirst}
\usepackage[utf8]{inputenc}
\usepackage{lipsum}        
\usepackage{makecell}
\usepackage{mathrsfs}
\usepackage{mathtools}
\usepackage{xparse}

\usepackage[framemethod=default]{mdframed}
\usepackage{microtype}     
\usepackage{multicol}
\usepackage{multirow}
\usepackage[numbers]{natbib}
\usepackage{nicefrac}       
\usepackage{tablefootnote}
\usepackage{url}            
\usepackage{wrapfig,lipsum}
\usepackage{xcolor}         
\usepackage{xspace}

\usepackage{doi}
\usepackage{tikz}
\usetikzlibrary{decorations.pathreplacing}
\usepackage{amsmath,amsfonts,bm}

\newcommand{\RR}{{\mathbb R}}

\def\1{\bm{1}}

\def\rvv{{\mathbf{v}}}

\def\rvx{{\mathbf{x}}}
\def\rvy{{\mathbf{y}}}
\def\rvz{{\mathbf{z}}}

\def\vzero{{\bm{0}}}
\def\vone{{\bm{1}}}
\def\vmu{{\bm{\mu}}}
\def\vtheta{{\bm{\theta}}}

\def\vb{{\bm{b}}}

\def\ve{{\bm{e}}}

\def\vs{{\bm{s}}}

\def\vv{{\bm{v}}}

\def\vx{{\bm{x}}}
\def\vy{{\bm{y}}}
\def\vz{{\bm{z}}}
\def\vmu{{\bm{\mu}}}

\def\mB{{\bm{B}}}

\def\mI{{\bm{I}}}

\DeclareMathAlphabet{\mathsfit}{\encodingdefault}{\sfdefault}{m}{sl}
\SetMathAlphabet{\mathsfit}{bold}{\encodingdefault}{\sfdefault}{bx}{n}

\def\gA{{\mathcal{A}}}
\def\gB{{\mathcal{B}}}

\def\gF{{\mathcal{F}}}

\def\gM{{\mathcal{M}}}

\def\gP{{\mathcal{P}}}

\def\gT{{\mathcal{T}}}

\def\gX{{\mathcal{X}}}

\newcommand{\grad}{\ensuremath{\nabla}}

\newcommand{\E}{\mathbb{E}}

\newcommand{\R}{\mathbb{R}}

\newcommand{\KL}[2]{\mathrm{KL}\left(#1 \big\| #2\right)}
\newcommand{\FI}[2]{\mathrm{FI}\left(#1 \big\| #2\right)}
\newcommand{\TVD}[2]{\mathrm{TV}\left(#1, #2\right)}

\newcommand{\der}{\mathrm{d}}

\DeclareMathOperator{\Tr}{Tr}

\makeatletter
\def\thickhline{%
  \noalign{\ifnum0=`}\fi\hrule \@height \thickarrayrulewidth \futurelet
   \reserved@a\@xthickhline}
\def\@xthickhline{\ifx\reserved@a\thickhline
               \vskip\doublerulesep
               \vskip-\thickarrayrulewidth
             \fi
      \ifnum0=`{\fi}}
\makeatother

\newlength{\thickarrayrulewidth}
\newtheorem{lemma}{Lemma}[section]
\newtheorem{theorem}[lemma]{Theorem}
\newtheorem{corollary}[lemma]{Corollary}

\newtheorem{remark}{Remark}

\newtheorem{definition}{Definition}

\algnewcommand{\IfThenElse}[3]{
\State \algorithmicif\ #1\ \algorithmicthen\ #2\ \algorithmicelse\ #3}
\algnewcommand{\IfThen}[2]{
\State \algorithmicif\ #1\ \algorithmicthen\ #2}

\newcommand{\rbkwx}{\rvx^{\gets}}

\newcommand{\bkwp}{p^{\gets}}

\newcommand*\samethanks[1][\value{footnote}]{\footnotemark[#1]}

\title{Reverse Transition Kernel: A Flexible Framework to Accelerate Diffusion Inference}

\author[$\dagger$]{\normalsize Xunpeng Huang\thanks{Mail to \href{xhuangck@connect.ust.hk}{xhuangck@connect.ust.hk}, \href{dzou@cs.hku.hk} {dzou@cs.hku.hk}}}
\author[$\S$]{Difan Zou\samethanks}
\author[$\dagger$]{Hanze Dong}
\author[$\S$]{Yi Zhang}
\author[$\P$]{Yian Ma}
\author[$\ddag$]{Tong Zhang}

\affil[$\dagger$]{Hong Kong University of Science and Technology}
\affil[$\S$]{The University of Hong Kong}
\affil[$\P$]{University of California San Diego}
\affil[$\ddag$]{University of Illinois Urbana-Champaign}

\begin{document}

\date{}
\maketitle

\begin{abstract}
    To generate data from trained diffusion models, most inference algorithms, such as DDPM \citep{ho2020denoising}, DDIM \citep{song2020denoising}, and other variants, rely on discretizing the reverse SDEs or their equivalent ODEs. In this paper, we view such approaches as decomposing the entire denoising diffusion process into several segments, each corresponding to a reverse transition kernel (RTK) sampling subproblem. Specifically, DDPM uses a Gaussian approximation for the RTK, resulting in low per-subproblem complexity but requiring a large number of segments (i.e., subproblems), which is conjectured to be inefficient. To address this, we develop a general RTK framework that enables a more balanced subproblem decomposition, resulting in $\tilde O(1)$ subproblems, each with strongly log-concave targets. We then propose leveraging two fast sampling algorithms, the Metropolis-Adjusted Langevin Algorithm (MALA) and Underdamped Langevin Dynamics (ULD), for solving these strongly log-concave subproblems. This gives rise to the RTK-MALA and RTK-ULD algorithms for diffusion inference. In theory, we further develop the convergence guarantees for  RTK-MALA and RTK-ULD in total variation (TV) distance: RTK-ULD can achieve $\epsilon$ target error within $\tilde{\mathcal O}(d^{1/2}\epsilon^{-1})$ under mild conditions, and  RTK-MALA enjoys a $\mathcal{O}(d^{2}\log(d/\epsilon))$   convergence rate under slightly stricter conditions. These theoretical results surpass the state-of-the-art convergence rates for diffusion inference and are well supported by numerical experiments.

\end{abstract}

\section{Introduction}

Generative models have become a core task in modern machine learning, where the neural networks are employed to learn the underlying distribution from training examples and generate new data points.
Among various generative models, denoising diffusions have produced state-of-the-art performance in many domains, including image and text generation~\citep{dhariwal2021diffusion, austin2021structured,ramesh2022hierarchical,saharia2022photorealistic}, text-to-speech synthesis~\citep{popov2021grad}, and scientific tasks~\citep{trippe2023diffusion, watson2023novo, boffi2023probability}.
The fundamental idea involves incrementally adding noise and gradually transform the data distribution to a prior distribution that is easier to sample from, e.g., Gaussian distribution.
Then, diffusion models parameterize and learn the score of the noised distributions to progressively denoise samples from priors and recover the data distribution~\citep {vincent2011connection, song2019generative}.

Under this paradigm, generating data in denoising diffusion models involves solving a series of sampling subproblems, i.e., generating samples from the distribution after one-step denoising. 
To this end, DDPM~\citep{ho2020denoising}, one of the most popular sampling methods in diffusion models, has been developed for this purpose. 
DDPM uses Gaussian processes with carefully designed mean and covariance to approximate the solutions to these sampling subproblems. 
By sequentially stacking a series of Gaussian processes, DDPM successfully generates high-quality samples that follow the data distribution. The empirical success of DDPM has immediately triggered various follow-up work \citep{song2020score,lu2022dpm}, aiming to accelerate the inference process and improve the generation quality. Alongside rapid empirical research on diffusion models and DDPM-like sampling algorithms, theoretical studies have emerged to analyze the convergence and sampling error of DDPM. 
In particular,~\cite{lee2022convergence,li2023towards,chen2022sampling, chen2023improved, benton2024nearly,chen2024probability}  have established polynomial convergence bounds, in terms of dimension $d$ and target sampling error $\epsilon$, for the generation process under various assumptions.
A typical bound under minimal data assumptions on the score of the data distribution is provided by~\cite{chen2022sampling,benton2024nearly}, which establishes an $\tilde{\mathcal{O}}(d\epsilon^{-2})$ score estimation guarantees to sample from data distribution within $\epsilon$-sampling error in the Total Variation (TV) distance. 



In essence, the denoising diffusion process can be approached through various decompositions of sampling subproblems, where the overall complexity depends on the number of these subproblems multiplied by the complexity of solving each one.  Within this framework, DDPM can be regarded as a specific solver for the denoising diffusion process that heavily prioritizes the simplicity of subproblems over their quantity. In particular, it adopts simple one-step Gaussian approximations for the subproblems, with $\mathcal{O}(1)$ computation complexity, but needs to deal with a relatively large number---approximately $O(d\epsilon^{-2})$---of target subproblems to ensure the cumulative sampling error is bounded by $\epsilon$ in TV distance. This imbalance raises the question of whether the DDPM-like approaches stand as the most efficient algorithm, considering the extensive potential subproblem decompositions of the denoising diffusion process. 
We therefore aim to: 



\begin{quote}
\emph{accelerate the inference of diffusion models via a more balanced subproblem decomposition in the denoising process.}
\end{quote}

In this work, we propose a novel framework called reverse transition kernel (RTK) to achieve exactly that. Our approach considers a generalized subproblem decomposition of the denoising process, where the difficulty of each sampling subproblem and the total number of subproblems are determined by the step size parameter $\eta$. 
Unlike DDPM, which requires setting $\eta = \epsilon^2$, resulting in approximately  $\tilde {\mathcal O}(1/\eta) = \tilde{\mathcal O}(1/\epsilon^2)$ subproblems, our framework allows $\eta$ to be feasible in a broader range. Furthermore, we demonstrate that a more balanced subproblem decomposition can be attained by carefully selecting $\eta = \Theta(1)$ as a constant, resulting in approximately $\tilde{\mathcal{O}}(1)$ sampling subproblems, with each target distribution being strongly log-concave. This nice property further enables us to efficiently solve the sampling subproblems using well-established acceleration techniques, such as Metropolis Hasting step and underdamped discretization,
without encountering many subproblems. Consequently, our proposed framework facilitates the design of provably faster algorithms than DDPM for performing diffusion inference. Our contributions can be summarized as follows.


\begin{itemize}[leftmargin=*]
    \item  We propose a flexible framework that enhances the efficiency of diffusion inference by balancing the quantity and hardness of RTK sampling subproblems used to segment the entire denoising diffusion process. Specifically, we demonstrate that with a carefully designed decomposition, the number of sampling subproblems can be reduced to approximately $\tilde{\mathcal{O}}(1)$, while ensuring that all RTK targets exhibit strong log-concavity. This capability allows us to seamlessly integrate a range of well-established sampling acceleration techniques, thereby enabling highly efficient algorithms for diffusion inference.
    \item Building upon the developed framework, we implement the RTK using the Metropolis-Adjusted Langevin Algorithm (MALA), making it the first attempt to adapt this highly accurate sampler for diffusion inference. Under slightly stricter assumptions on the estimation errors of the energy difference and score function, we demonstrate that RTK-MALA can achieve linear convergence with respect to the sampling error $\epsilon$, specifically $\mathcal{O}(\log(1/\epsilon))$, which significantly outperforms the $\tilde{\mathcal{O}}(1/\epsilon^2)$ convergence rate of DDPM \cite{chen2022sampling,benton2024nearly}. Additionally, we consider the practical diffusion model where only the score function is accessible and develop a score-only RTK-MALA algorithm. We further prove that the score-only RTK-MALA algorithm can achieve an error $\epsilon$ with a complexity of $\tilde{\mathcal{O}}(\epsilon^{-2/(u-1)} \cdot 2^u)$, where $u$ can be an arbitrarily large constant, provided the energy function satisfies the $u$-th order smoothness condition.
    
    
    \item We further implement Underdamped Langevin Dynamics (ULD) within the RTK framework. The resulting RTK-ULD algorithm achieves a state-of-the-art complexity of  \(\tilde{\mathcal{O}}(d^{1/2}\epsilon^{-1})\) for both $d$ and $\epsilon$ dependence under minimal data assumptions, i.e., Lipschitz condition for the ground truth score function. Compared with the $\tilde{\mathcal{O}}(d\epsilon^{-2})$ complexity guarantee for DDPM, it improves the complexity with an $\tilde{\mathcal{O}}(d^{1/2}\epsilon^{-1})$ factor. This result also matches the state-of-the-art convergence rate of the ODE-based methods \citep{chen2024probability}, though those methods require Lipschitz conditions for both the ground truth score function and the score neural network. 

\end{itemize}
\section{Preliminaries}
\label{sec:pre}

In this section, we first introduce the notations used in subsequent sections. 
Then, we present several distinct Markov processes to demonstrate the procedures for adding noise to existing data and generating new data. 
Besides, we specify the assumptions required for the target distribution in our algorithms and analysis.

\noindent\textbf{Notations.}
We say a complexity $h\colon \R\rightarrow\R$ to be  $h(n)=\mathcal{O}(n^k)$ or $h(n)=\tilde{\mathcal{O}}(n^k)$ if the complexity satisfies $h(n)\le c\cdot n^{k}$ or $h(n)\le c\cdot n^{k}[\log (n)]^{k^\prime}$ for absolute contant $c, k$ and $k^\prime$.
We use the lowercase bold symbol $\rvx$ to denote a random vector, and the lowercase italicized bold symbol $\vx$ represents a fixed vector. 
The standard Euclidean norm is denoted by $\|\cdot\|$.
The data distribution is presented as $p_*\propto\exp(-f_*)$.
Besides, we define two Markov processes $\R^d$, i.e.,
\begin{equation*}
    \left\{\rvx_t\right\}_{t\in[0,T]},\quad  \left\{\rbkwx_{k\eta}\right\}_{k\in\{0,1,\ldots, K\}},\quad \mathrm{where}\quad T=K\eta.
\end{equation*}
In the above notations, $T$ presents the mixing time required for the data distribution to converge to specific priors, $K$ denotes the iteration number of the generation process, and $\eta$ signifies the corresponding step size. 
Further details of the two processes are provided below.

\noindent\textbf{Adding noise to data with the forward process.}
The first Markov process $\{\rvx_t\}$ corresponds to generating progressively noised data from $p_*$.
In most denoising diffusion models, $\{\rvx_t\}$ is an Ornstein–Uhlenbeck (OU) process shown as follows
\begin{equation}
    \label{sde:ou}
    \der \rvx_t = -\rvx_t \der t + \sqrt{2}\der \mB_t \quad \mathrm{where}\quad \rvx_0\sim p_*\propto \exp(-f_*).
\end{equation}
If we denote underlying distribution of $\rvx_t$ as $p_t\propto \exp(-f_t)$ meaning $f_0 = f_*$,
the forward OU process provides an analytic form of the transition kernel, i.e.,
\begin{equation}
    \label{def:forward_trans_ker}
    p_{t^\prime|t}(\vx^\prime|\vx) = \frac{p_{t^\prime, t}(\vx^\prime, \vx)}{p_t(\vx)} =   \left(2\pi \left(1-e^{-2(t^\prime-t)}\right)\right)^{-d/2}
     \cdot \exp \left[\frac{-\left\|\vx^\prime -e^{-(t^\prime-t)}\vx \right\|^2}{2\left(1-e^{-2(t^\prime-t)}\right)}\right]
\end{equation}
for any $t^\prime\ge t$, where $p_{t^\prime,t}$ denotes the joint distribution of $(\rvx_{t^\prime}, \rvx_t)$.
According to the Fokker-Planck equation, we know the stationary distribution for SDE.~\eqref{sde:ou} is the standard Gaussian distribution.

\noindent\textbf{Denoising generation with a reverse SDE.}
Various theoretical works~\citep{lee2022convergence,li2023towards,chen2022sampling, chen2023improved, benton2024nearly} based on DDPM~\citep{ho2020denoising} consider the generation process of diffusion models as the reverse process of SDE.~\eqref{sde:ou} denoted as $\{\rbkwx_t\}$.
According to the Doob’s $h$-Transform, the reverse SDE, i.e., $\{\rbkwx_t\}$, follows from
\begin{equation}
    \label{def:rev_ou_sde}
    \der \rbkwx_t = \left(\rbkwx_t + 2\grad\ln p_{T-t}(\rbkwx_t)\right)\der t + \sqrt{2}\der\mB_t,
\end{equation}
whose underlying distribution $p^\gets_{t}$ satisfies $p_{T-t} = p^\gets_{t}$.
Similar to the definition of transition kernel shown in Eq.~\ref{def:forward_trans_ker}, we define $p^\gets_{t^\prime|t}(\vx^\prime|\vx) = p^\gets_{t^\prime, t}(\vx^\prime,\vx)/p^\gets_t(\vx)$ for any $t^\prime\ge t\ge 0$ and name it as reverse transition kernel (RTK).

To implement SDE.~\eqref{def:rev_ou_sde}, diffusion models approximate the score function $\grad\ln p_{t}$ with a parameterized neural network model,  denoted by $\vs_{\vtheta,t}$, where $\vtheta$ denotes the network parameters.
Then, SDE.~\eqref{def:rev_ou_sde} can be practically implemented by
\begin{equation}
    \label{def:rev_ou_sde_prac}
    \der \overline{\rvx}_t = \left(\overline{\rvx}_t + 2 \vs_{\theta, T-k\eta}(\overline{\rvx}_{k\eta})\right)\der t + \sqrt{2}\der\mB_t\quad \mathrm{for}\quad t\in[k\eta, (k+1)\eta)
\end{equation}
with a standard Gaussian initialization, $\overline{\rvx}_0 \sim \mathcal{N}(\vzero,\mI)$.
 Eq.~\eqref{def:rev_ou_sde_prac} has the following closed solution
\begin{equation}
    \label{def:rev_ou_sde_prac_closed}
    \overline{\rvx}_{(k+1)\eta} = e^{\eta}\cdot \overline{\rvx}_{k\eta} - 2(1-e^{\eta})\vs_{\theta, T-k\eta}(\overline{\rvx}_{k\eta}) + \sqrt{e^{2\eta}-1}\cdot\xi\quad \mathrm{where}\quad \xi\sim\mathcal{N}(\vzero,\mI).
\end{equation}
which is exactly the DDPM algorithm.

\noindent\textbf{DDPM approximately samples the reverse transition kernel.}
 DDPM can also be viewed as an approximated sampler for RTK, i.e., $p^\gets_{t^\prime|t}(\vx^\prime|\vx)$ for some $t'>t$. In particular, the update of DDPM at the iteration $k$ applies the Gaussian process 
\begin{equation}
    \label{def:ddpm_trans_ker}
    \overline{p}_{(k+1)\eta|k\eta}(\cdot|\vx)  = \mathcal{N}\left(e^{\eta}\vx - 2(1-e^{\eta})\vs_{\theta,T-k\eta}(\vx), (e^{2\eta}-1)\cdot\mI\right).
\end{equation}
to approximate the distribution $p^\gets_{(k+1)\eta|k\eta}(\cdot|\vx)$ \citep{ho2020denoising}. Specifically, by the chain rule of KL divergence, the gap between the data distribution $p_*$ and the generated distribution  $\overline{p}_T$ satisfies
\begin{equation}\label{eq:chainrule_KL}
    \small
    \KL{\overline{p}_T}{p_*} \le \KL{\overline{\rvx}_0}{p^\gets_0} + \sum_{k=0}^{K-1} \E_{\overline{\rvx}\sim \overline{p}_{k\eta}}\left[\KL{\overline{p}_{(k+1)\eta|k\eta}(\cdot|\overline{\rvx})}{p^\gets_{(k+1)\eta|k\eta}(\cdot|\overline{\rvx})}\right],
\end{equation}
where $K = T/\eta$ is the total number of iterations. For DDPM, to guarantee a  small sampling error, we need to use a small step size $\eta$ to ensure that $\overline{p}_{(k+1)\eta|k\eta}$ is sufficiently close to $p^\gets_{(k+1)\eta|k\eta}$. Then, the required iteration numbers $K=T/\eta$ will be large and dominate the computational complexity. In \citet{chen2022sampling,chen2022improved}, it was shown that one needs to set $\eta = \tilde{\mathcal{O}}(\epsilon^2)$ to achieve $\epsilon$ sampling error in TV distance (assuming no score estimation error) and the total complexity is $K=\tilde{\mathcal{O}}(1/\epsilon^2)$.

\noindent \textbf{Intuition for General Reverse Transition Kernel.}
As previously mentioned, DDPM approximately solves RTK sampling subproblems using a small step size $\eta$. While this allows for efficient one-step implementation, it necessitates a large number of RTK sampling problems. This naturally creates a trade-off between the quantity of RTK sampling problems and the complexity of solving them. To address this, one can consider a larger step size $\eta$, which results in a relatively more challenging RTK sampling target $p^\gets_{(k+1)\eta|k\eta}$ and a reduced number of sampling problems ($K=T/\eta$). By examining a general choice for the step size $\eta$, the generation process of diffusion models can be depicted through a comprehensive framework of reverse transition kernels, which will be explored in depth in the following section. This framework enables the design of various decompositions for RTK sampling problems and algorithms to solve them, resulting in an extensive family of generation algorithms for diffusion models (that encompasses DDPM). Consequently, this also offers the potential to develop faster algorithms with lower computational complexities, e.g., applying fast sampling algorithms for sampling the RTK, i.e., $p^\gets_{(k+1)\eta|k\eta}$ with a reasonably large $\eta$.

\noindent \textbf{General Assumptions.} Similar to the analysis of DDPM~\citep{chen2022sampling,chen2023improved}, we make the following  assumptions on the data distribution $p_*$ that will be utilized in the theory.

\begin{enumerate}[label=\textbf{[A{\arabic*}]}]
    \item \label{ass:lips_score} For all $t\ge 0$, the score $\grad \ln p_{t}$ is $L$-Lipschitz.
    \item \label{ass:second_moment} The second moment of the target distribution $p_*$ is upper bounded, i.e., $\mathbb{E}_{p_*}\left[\left\|\cdot\right\|^{2}\right]=m_2^2$.
\end{enumerate}
Assumption~\ref{ass:lips_score} is standard one in diffusion literature and has been used in many prior works~\citep{block2020generative, chen2022improved, lee2022convergence, chen2024probability}.
Moreover, we do not require the isoperimetric conditions, e.g., the establishment of the log-Sobolev inequality and the Poincar\'e inequality for the data distribution $p_*$ as~\cite{lee2022convergence}, and the convex conditions for the energy function $f_*$ as~\cite{block2020generative}.
Therefore, our assumptions cover a wide range of highly non-log-concave data distributions.
We emphasize that Assumption~\ref{ass:lips_score} may be relaxed only to assume the target distribution is smooth rather than the entire OU process, based on the technique in \cite{chen2023improved} (see rough calculations in their Lemmas 12 and 14).
We do not include this additional relaxation in this paper to clarify our analysis.
Assumption~\ref{ass:second_moment} is one of the weakest assumptions being adopted for the analysis of posterior sampling.

\section{General Framework of Reverse Transition Kernel}
\label{sec:framework}
This section introduces the general framework of Reverse Transition Kernel (RTK).
As mentioned in the previous section, the framework is built upon the general setup of segmentation: each segment has length $\eta$; within each segment, we generate samples according to the RTK target distributions. Then, the entire generation process in diffusion models can be considered as the combination of a series of sampling subproblems. In particular, the inference process via RTK is displayed in Alg. \ref{alg:rtk}. 


\begin{algorithm}
    \caption{\sc Inference with Reverse Transition Kernel (RTK)}
    \label{alg:rtk}
    \begin{algorithmic}[1]
            \State {\bfseries Input:} Initial particle $\hat{\rvx}_0$ sampled from the standard Gaussian distribution, Iteration number $K$, Step size $\eta$, required convergence accuracy $\epsilon$;
            \For{$k=0$ to $K-1$}
                \State \label{step:inner_sampler} Draw sample $\hat{\rvx}_{(k+1)\eta}$ with MCMCs from $\hat{p}_{(k+1)\eta|k\eta}(\cdot|\hat{\rvx}_{k\eta})$ which is closed to the ground-truth reverse transition kernel, i.e.,
                \begin{equation}
                    \label{def:inner_target_dis}
                    p^\gets_{(k+1)\eta|k\eta}(\vz|\hat{\rvx}_{k\eta})\propto \exp\left(-g(\vz)\right)\coloneqq \exp\left( -f_{(K-k-1)\eta}(\vz)-\frac{\left\|\hat{\rvx}_{k\eta} - \vz\cdot e^{-\eta}\right\|^2}{2(1-e^{-2\eta})}\right).
                \end{equation}
            \EndFor
            \State {\bfseries return} $\hat{\rvx}_{K}$.
    \end{algorithmic}
\end{algorithm}

\noindent\textbf{The implementation of RTK framework.\quad}
We begin with a new Markov process $\{\hat{\rvx}_{k\eta}\}_{k=0,1,\ldots,K}$ satisfying $K\eta = T$, where the number of segments $K$, segment length $\eta$, and length of the entire process $T$ correspond to the definition in Section~\ref{sec:pre}.
Consider the Markov process $\{\hat{\rvx}_{k\eta}\}$ as the generation process of diffusion models with underlying distributions $\{\hat{p}_{k\eta}\}$, we require $\hat{p}_0 = \mathcal{N}(\vzero,\mI)$ and $\hat{p}_{K\eta}\approx p_*$, which is similar to the Markov process $\{\rbkwx_{k\eta}\}$.
In order to make the underlying distribution of output particles close to the data distribution, we can generate $\hat{\rvx}_{k\eta}$ with Alg.~\ref{alg:rtk}, which is equivalent to the following steps:
\begin{itemize}[leftmargin=*]
    \item Initialize $\hat{\rvx}_0$ with an easy-to-sample distribution, e.g., $\mathcal{N}(\vzero,\mI)$, which is closed to $p_{K\eta}$.
    \item Update particles by drawing samples from $\hat{p}_{(k+1)\eta|k\eta}(\cdot|\hat{\rvx}_{k\eta})$, which satisfies $$\hat{p}_{(k+1)\eta|k\eta}(\cdot|\hat{\rvx}_{k\eta})\approx p^\gets_{(k+1)\eta|k\eta}(\cdot|\hat{\rvx}_{k\eta}).$$ 
\end{itemize}
Under these conditions, if $\hat{p}_{k\eta}(\vz)\approx p_{(K-k)\eta}(\vz)$ , then we have
\begin{equation*}
    \hat{p}_{(k+1)\eta}(\vz) = \left<\hat{p}_{(k+1)\eta|k\eta}(\vz|\cdot), \hat{p}_{k\eta}(\cdot)\right> \approx \left<p^\gets_{(k+1)\eta|k\eta}(\vz|\cdot), p^\gets_{k\eta}(\cdot)\right> = p_{(k+1)\eta}(\vz)
\end{equation*}
for any $k\in\{0,1,\ldots, K\}$.
This means we can implement the generation of diffusion models by solving a series of sampling subproblems with target distributions $p^\gets_{(k+1)\eta|k\eta}(\cdot|\hat{\rvx}_{k\eta})$.

\noindent \textbf{The closed form of reverse transition kernels.\quad}
To implement Alg.~\ref{alg:rtk}, the most critical problem is determining the analytic form of RTK $p^\gets_{t^\prime|t}(\vx^\prime|\vx)$ for and $t^\prime\ge t\ge 0$ which is shown in the following lemma whose proof is deferred to Appendix~\ref{sec:appendix_rtk_framework}.
\begin{lemma}
    \label{lem:rev_trans_ker_form}
    Suppose a Markov process $\{\rvx_t\}$ with SDE.~\ref{sde:ou}, then for any $t^\prime > t$, we have
    \begin{equation*}
        p^\gets_{T-t|T-t^\prime}(\vx|\vx^\prime) = p_{t|t^\prime}(\vx|\vx^\prime) \propto \exp\left(-f_t(\vx)-\frac{\left\|\vx^\prime - \vx\cdot e^{-(t^\prime-t)}\right\|^2}{2(1-e^{-2(t^\prime-t)})}\right).
    \end{equation*}
\end{lemma}
The first critical property shown in this Lemma is that RTK $p_{t|t^\prime}$ is a perturbation of $p_t$ with a $l_2$ regularization.
This means if the score of $p_t$, i.e., $\grad f_t$, can be well-estimated, the score of RTK, i.e., $\grad \log p_{t|t^\prime}$ can also be approximated with high accuracy.
Moreover, in the diffusion model, $\grad f_t=\nabla \log p_t$ is exactly the score function at time $t$, which is approximated by the score network function $\vs_{\theta,t}(\vx)$, then
\begin{equation*}
    -\grad\log p_{t|t^\prime}(\vx|\vx^\prime) = \grad f_t(\vx) + \frac{e^{-2(t^\prime - t)}\vx - e^{-(t^\prime - t)}\vx^\prime}{1-e^{-2(t^\prime - t)}}\approx \vs_{\theta, t}(\vx) + \frac{e^{-2(t^\prime - t)}\vx - e^{-(t^\prime - t)}\vx^\prime}{1-e^{-2(t^\prime - t)}},
\end{equation*}
which can be directly calculated with a single query of $\vs_{\theta,t}(\vx)$.
The second critical property of RTK is that we can control the spectral information of its score by tuning the gap between $t^\prime$ and $t$.
Specifically, considering the target distribution, i.e., $p_{(K-k-1)\eta|(K-k)\eta}$ for the $k$-th transition, the Hessian matrix of its energy function satisfies
\begin{equation*}
    - \grad^2 \log p_{(K-k-1)\eta|(K-k)\eta} = \grad^2 f_{(K-k-1)\eta}(\vx) +  \frac{e^{-2\eta}}{1-e^{-2\eta}}\cdot \mI.
\end{equation*}
According to Assumption~\ref{ass:lips_score}, the Hessian $\grad^2 f_{(K-k-1)\eta}(\vx)=-\grad^2 \log p_{(K-k-1)\eta}$ can be lower bounded by $-L\mI$, which implies that RTK $p_{(K-k-1)\eta|(K-k)\eta}$ will be $L$-strongly log-concave and $3L$-smooth when the step size is set $\eta= 1/2\cdot \log (1+1/2L)$.
This further implies that the targets of all subsampling problems in Alg. \ref{alg:rtk} will be strongly log-concave, which can be sampled very efficiently by various posterior sampling algorithms.

\noindent\textbf{Sufficient conditions for the convergence.\quad}
According to Pinsker's inequality and Eq. \eqref{eq:chainrule_KL}, the we can obtain the following lemma that establishes the general error decomposition for Alg.\ref{alg:rtk}.
\begin{lemma}
    \label{lem:num_diff_balance_kl}
    For Alg~\ref{alg:rtk}, we have
    \begin{align*}
        \TVD{\hat{p}_{K\eta}}{p_*}\le&  \sqrt{(1+L^2)d + \left\|\grad f_*(\vzero)\right\|^2} \cdot \exp(-K\eta) \\
        &+ \sqrt{\frac{1}{2}\sum_{k=0}^{K-1} \E_{\hat{\rvx}\sim \hat{p}_{k\eta}}\left[\KL{\hat{p}_{(k+1)\eta|k\eta}(\cdot|\hat{\rvx})}{\bkwp_{(k+1)\eta|k\eta}(\cdot|\hat{\rvx})}\right]} 
    \end{align*}
    for any $K\in \mathbb{N}_+$ and $\eta\in\R_+$.
\end{lemma}
It is worth noting that the choice of $\eta$ represents a trade-off between the number of subproblems divided throughout the entire process and the difficulty of solving these subproblems. By considering the choice $\eta= 1/2\cdot \log (1+1/2L)$, we can observe two points: (1) the sampling subproblems in Alg.~\ref{alg:rtk} tend to be simple, as all RTK targets, presented in Lemma~\ref{lem:rev_trans_ker_form}, can be provably strongly log-concave; (2) the total number of subproblems is $K = T/\eta = \tilde{\mathcal O}(1)$, which is not large. Conversely, when considering a larger $\eta$ that satisfies $\eta\gg \log (1+1/L)$, the RTK target will no longer be guaranteed to be log-concave, resulting in high computational complexity, potentially even exponential in $d$, when solving the corresponding sampling subproblems. On the other hand, if a much smaller step size $\eta = o(1)$ is considered, the target distribution of the sampling subproblems can be easily solved, even with a one-step Gaussian process. However, this will increase the total number of sampling subproblems, potentially leading to higher computational complexity.

Therefore, we will consider the setup $\eta = 1/2\cdot \log(1+1/2L)$ in the remaining part of this paper. Now, the remaining task, which will be discussed in the next section, would be designing and analyzing the sampling algorithms for implementing all iterations of Alg.~\ref{alg:rtk}, i.e., solving the subproblems of RTK.



\section{Implementation of RTK inner loops}

In this section, we outline the implementation of Step~\ref{step:inner_sampler} in the RTK algorithm, which aims to solve the sampling subproblems with strong log-concave targets, i.e., $p^\gets_{(k+1)\eta|k\eta}(\cdot|\hat{\rvx}_{k\eta})\propto \exp(-g)$.
Specifically, we employ two MCMC algorithms, i.e., the Metropolis-adjusted Langevin algorithm (MALA) and underdamped Langevin dynamics (ULD).
For each algorithm, we will first introduce the detailed implementation, combined with some explanation about notations and settings to describe the inner sampling process.
After that, we will provide general convergence results and discuss them in several theoretical or practical settings.
Besides, we will also compare our complexity results with the previous ones when achieving the convergence of TV distance to show that the RTK framework indeed obtains a better complexity by balancing the number and complexity of sampling subproblems.

\begin{algorithm}[t]
    \caption{\sc MALA/Projected MALA for RTK Inference}
    \label{alg:inner_mala}
    \begin{algorithmic}[1]
            \State {\bfseries Input:}  Returned particle of the previous iteration $\vx_0$, current iteration number $k$, inner iteration number $S$, inner step size $\tau$, required convergence accuracy $\epsilon$;
            \State Draw the initial particle $\rvz_0$ from 
            \begin{equation*}
                \frac{\mu_0(\der \vz)}{\der \vz} \propto \exp\left(-L\|\vz\|^2 - \frac{\left\|\vx_0 - e^{-\eta}\vz\right\|^2}{2(1-e^{-2\eta})}\right).
            \end{equation*}
            \For{$s=0$ to $S-1$}
                \State \label{step:ula} Draw a sample $\tilde{\vz}_s$ from the Gaussian distribution $\mathcal{N}(\vz_s - \tau\cdot s_{\theta}(\vz_s) ,2\tau\mI)$;
                \If{$\vz_{s+1}\not\in\mathcal{B}(\vz_s, r)\cap\mathcal{B}(\vzero, R)$} \label{step:inball}
                    \State {$\vz_{s+1} = \vz_s$}; \Comment{This condition step is only activated for Projected MALA.}
                    \State \textbf{continue};
                \EndIf
                \State Calculate the accept rate as
                \begin{equation*}
                    \begin{aligned}
                        &\quad \quad a(\tilde{\vz}_s - (\vz_s-\tau\cdot s_\theta(\vz_s)), \vz_s)  \\ =&\min\left\{1, \exp\left(r_g(\vz_s,\tilde{\vz}_s) 
                        +  \frac{\left\|\tilde{\vz}_s - \vz_s +\tau\cdot s_\theta(\vz_s)\right\|^2 - \left\|\vz_s-\tilde{\vz}_s+\tau\cdot s_\theta(\tilde{\vz}_s)\right\|^2}{4\tau}\right)\right\};
                    \end{aligned}
                \end{equation*}
                \State Update the particle $\vz_{s+1} = \tilde{\vz}_s$ with probability $a$, otherwise $\vz_{s+1}=\vz_s$.
            \EndFor
            \State {\bfseries return} $\rvz_{S}$;
    \end{algorithmic}
\end{algorithm}

\noindent\textbf{RTK-MALA.\quad}
Alg.~\ref{alg:inner_mala} presents a solution employing MALA for the inner loop. When it is used to solve the $k$-th sampling subproblem of Alg.~\ref{alg:rtk}, $\vx_0$ is equal to $\hat{\rvx}_{k\eta}$ defined in Section~\ref{sec:framework} and used to initialize particles iterating in Alg.~\ref{alg:inner_mala}.
In Alg.~\ref{alg:inner_mala}, we consider the process $\{\rvz_s\}_{s=0}^S$ whose underlying distribution is denoted as $\{\mu_s\}_{s=0}^{S}$.
Although we expect $\mu_S$ to be close to the target distribution $p^\gets_{(k+1)\eta|k\eta}(\cdot|\vx_0)$, in real practice, the output particles $\rvz_S$ can only approximately follow $p^\gets_{(k+1)\eta|k\eta}(\cdot|\vx_0)$ due to inevitable errors.
Therefore, these errors should be explained in order to conduct a meaningful complexity analysis of the implementable algorithm.
Specifically, Alg.~\ref{alg:inner_mala} introduces two intrinsic errors:
\begin{enumerate}[label=\textbf{[E{\arabic*}]}]
    \item \label{e1}Estimation error of the score function: we assume a score estimator, e.g., a well-trained diffusion model, $s_{\theta}$, which can approximate the score function with an $\epsilon_{\mathrm{score}}$ error, i.e., $\left\|s_{\theta,t}(\vz) - \grad \log p_t(\vz)\right\|\leq \epsilon_{\mathrm{score}}$  for all $\vz\in\RR^d$ and $t\in[0,T]$.
    \item \label{e2}Estimation error of the energy function difference:  we assume an energy difference estimator $r$ which can approximate energy difference with an $\epsilon_{\mathrm{energy}}$ error, i.e.,
    $\left| r_t(\vz^\prime, \vz) + \log p_t(\vz^\prime) - \log p_t(\vz) \right|\leq \epsilon_{\mathrm{energy}}$ for all $\vz,\vz'\in\RR^d$.
\end{enumerate}

Under these settings, we provide a general convergence theorem for Alg.~\ref{alg:inner_mala}.
To clearly convey the convergence properties, we only show an informal version.
\begin{theorem}[Informal version of Theorem~\ref{thm:nn_estimate_complexity_gene}]
    \label{thm:nn_estimate_complexity_gene_informal}
    Under Assumption~\ref{ass:lips_score}--\ref{ass:second_moment}, for Alg.~\ref{alg:rtk}, we choose 
    \begin{equation*}
        \eta= \frac{1}{2}\log \frac{2L+1}{2L} \quad \mathrm{and}\quad K = 4L\cdot \log \frac{(1+L^2)d+\left\|\grad f_*(\vzero)\right\|^2}{\epsilon^2}
    \end{equation*}
    and implement Step 3 of Alg.~\ref{alg:rtk} with Alg.~\ref{alg:inner_mala}.
    Suppose the score~\ref{e1}, energy~\ref{e2} estimation errors and the inner step size $\tau$ satisfy
    \begin{equation*}
        \epsilon_{\mathrm{score}} = \mathcal{O}(\rho d^{-1/2}),\quad  \epsilon_{\mathrm{energy}} = \mathcal{O}(\rho \tau^{1/2}),\quad \mathrm{and}\quad \tau = \tilde{\mathcal{O}}\left(L^{-2}\cdot (d+m_2^2+Z^2)^{-1}\right),
    \end{equation*}
    and the hyperparameters, i.e., $R$ and $r$, are chosen properly. 
    We have
    \begin{equation}
        \label{tv:con_rate}
        \small
        \TVD{\hat{p}_{K\eta}}{p_*}\le \tilde{\mathcal{O}}(\epsilon) + \exp\left(\mathcal{O}(L(d+m_2^2))\right)\cdot \left(1- \frac{\rho^2}{4} \cdot \tau\right)^S  + \tilde{\mathcal{O}}\left(\frac{Ld^{1/2}\epsilon_{\mathrm{score}}}{\rho}\right) + \tilde{\mathcal{O}}\left(\frac{L\epsilon_{\mathrm{energy}}}{\rho\tau^{1/2}}\right)
    \end{equation}
    where $\rho$ is the Cheeger constant of a truncated inner target distribution $\exp(-g(\vz))\vone[\vz\in\mathcal{B}(\vzero,R)]$ and $Z$ denotes the maximal $l_2$ norm of particles appearing in outer loops (Alg.~\ref{alg:rtk}).
\end{theorem}

It should be noted that the choice of $\eta$ choice ensures the $L$ strong log-concavity of target distribution $\exp(-g(\vz))$, which means its Cheeger constant is also $L$.
Although the Cheeger constant $\rho$ in the second term of Eq.~\ref{tv:con_rate} corresponding to truncated $\exp(-g(\vz))$ should also be near $L$ intuitively, current techniques can only provide a loose lower bound at an $\mathcal{O}(\sqrt{L/d})$-level (proven in Corollary~\ref{cor:cheeger_truncation}).
While in both cases above, the Cheeger constant is independent with $\epsilon$.
Combining this fact with an $\epsilon$-independent choice of inner step sizes $\tau$, the second term of Eq.~\ref{tv:con_rate} will converge linearly with respect to $\epsilon$.
As for the diameter $Z$ of particles used to upper bound $\tau$, though it may be unbounded in the standard implementation of Alg.~\ref{alg:inner_mala}, Lemma~\ref{lem:particles_z_bound} can upper bound it with $\tilde{\mathcal{O}}\left(L^{3/2}(d+m_2^2) \rho^{-1}\right)$ under the projected version of Alg.~\ref{alg:inner_mala}.

Additionally, to require the final sampling error to satisfy $\TVD{\hat{p}_{K\eta}}{p_*}\le \tilde{\mathcal{O}}(\epsilon)$, Eq.~\ref{tv:con_rate} shows that the score and energy difference estimation errors should be $\epsilon$-dependent and sufficiently small, where $\epsilon_{\mathrm{score}}$ corresponding to the training loss can be well-controlled.
However, obtaining a highly accurate energy difference estimation (requiring a small $\epsilon_{\mathrm{energy}}$) is hard with only diffusion models.
To solve this problem, we can introduce a neural network energy estimator similar to~\cite{xu2024provably} to construct $r(\vz^\prime, \vz,t)$, which induces the following complexity describing the calls of the score estimation.
\begin{algorithm}[t]
    \caption{\sc ULD for RTK Inference}
    \label{alg:inner_uld}
    \begin{algorithmic}[1]
            \State {\bfseries Input:}  Returned particle of the previous iteration $\vx_0$, current iteration number $k$, inner iteration number $S$, inner step size $\tau$, velocity diffusion coefficient $\gamma$, required convergence accuracy $\epsilon$;
            \State Initialize the particle and velocity pair, i.e., $(\hat{\vz}_0, \hat{\vv}_0)$ with a Gaussian type product measure, i.e.,  $\mathcal{N}(\vzero, e^{2\eta}-1)\otimes \mathcal{N}(\vzero,\mI)$;
            \For{$t=s$ to $S-1$}
            \State Draw noise sample pair $(\xi^z_s,\xi^v_s)$ from a Gaussian type distribution.
            \State $\hat{\vz}_{s+1} = \hat{\vz}_{s} + \gamma^{-1} (1-e^{-\gamma \tau}) \hat{\vv}_{s} - \gamma^{-1} (\tau -\gamma^{-1} (1-e^{-\gamma \tau}) )s_\theta(\hat{\vz}_s) + \xi_s^z$
            \State $\hat{\vv}_{s+1} = e^{-\gamma\tau }\hat{\vv}_s - \gamma^{-1} (1-e^{-\gamma \tau}) 
 s_\theta(\hat{\vz}_s) + \xi^v_t$
            \EndFor
            \State {\bfseries return} $\vz_{S}$;
    \end{algorithmic}
\end{algorithm}

\begin{corollary}[Informal version of Corollary~\ref{cor:complexity_19}]
    \label{cor:complexity_19_informal}
    Suppose the estimation errors of score and energy difference satisfy
    \begin{equation*}
        \epsilon_{\mathrm{score}}\le \frac{\rho\epsilon}{Ld^{1/2}}\quad \mathrm{and}\quad \epsilon_{\mathrm{energy}}\le \frac{\rho\epsilon}{L^2\cdot (d^{1/2}+m_2+Z)},
    \end{equation*}
    If we implement Alg.~\ref{alg:rtk} with the projected version of Alg.~\ref{alg:inner_mala} with the same hyperparameter settings as Theorem~\ref{thm:nn_estimate_complexity_gene_informal}, it has $\TVD{\hat{p}_{K\eta}}{p_*}\le \tilde{\mathcal{O}}(\epsilon)$
    with an $\mathcal{O}(L^4 \rho^{-2}\cdot \left(d+m_2^2\right)^2 Z^2\cdot \log (d/\epsilon) )$ complexity.
\end{corollary}

Considering the loose bound for both $\rho$ and $Z$, the complexity will be at most $\tilde{\mathcal{O}}(L^5(d+m_2^2)^6)$ which is the first linear convergence w.r.t. $\epsilon$ result for the diffusion inference process.

\noindent \textbf{Score-only RTK-MALA.}
However, the parametric energy function may not always exist in real practice. We consider a more practical case where only the score estimation is accessible. In this case, we will make use of estimated score functions to approximate the energy difference, leading to the score-only RTK-MALA algorithm. In particular, recall that the energy difference function takes the following form:
\begin{equation*}
    \small
    g(\vz^\prime) - g(\vz) =  -\log p_{(K-k-1)\eta}(\vz^\prime)+\frac{\left\|\vx_0 - \vz^\prime\cdot e^{-\eta}\right\|^2}{2(1-e^{-2\eta})} +\log p_{(K-k-1)\eta}(\vz) - \frac{\left\|\vx_0 - \vz\cdot e^{-\eta}\right\|^2}{2(1-e^{-2\eta})}.
\end{equation*}
Since the quadratic term can be obtained exactly, we only need to estimate the energy difference. Then let $f(\vz) = -\log p_{(K-k-1)\eta}(\vz)$ and denote $h(t) = f\left( \left(\vz^\prime - \vz\right)\cdot t + \vz\right)$, 
the energy difference $g(\vz^\prime) - g(\vz)$ can be reformulated as
\begin{equation*}
    h(1)-h(0) = \sum_{i=1} \frac{h^{(i)}(0)}{i!}\quad \mathrm{and}\quad h^{(i)}(t) \coloneqq \frac{\der^i h(t)}{(\der t)^i},
\end{equation*}
where we perform the standard Taylor expansion at the point $t=0$. Then, we only need the derives of $h^{i}(0)$, which can be estimated using only the score function. For instance, the $h^{(1)}(t)$ can be estimated with score estimations:
\begin{equation*}
    h^{(1)}(t) = \grad f((\vz^\prime-\vz)\cdot t + \vz)\cdot (\vz^\prime - \vz) \approx \tilde{h}^{(1)}(t) \coloneqq s_{\theta}((\vz^\prime-\vz)\cdot t + \vz) \cdot (\vz^\prime - \vz).
\end{equation*}
Moreover, regarding the high-order derivatives, we can recursively perform the approximation: $\tilde{h}^{(i+1)}(0) = (\tilde{h}^{(i)}(\Delta t) - \tilde{h}^{(i)}(0))/\Delta t $.
Consider performing the approximation up to $u$-order derivatives, we can get the approximation of the energy difference:
\begin{equation*}
    r_{(K-k-1)\eta}(\vz^\prime, \vz) \coloneqq \sum_{i=1}^u\frac{\tilde{h}^{(i)}(0)}{i!}.
\end{equation*}
Then, the following corollary states the complexity of the score-only RTK-MALA algorithm.
\begin{corollary}
    Suppose the estimation errors of the score satisfies $\epsilon_{\mathrm{score}} \ll \rho\epsilon/(Ld^{1/2})$,
    and the log-likelihood function of $p_t$ has a bounded $u$-order derivative, e.g., $\left\|\grad^{(u)} f(\vz)\right\|\le L$,
    we have a non-parametric estimation for log-likelihood to make we have $\TVD{\hat{p}_{K\eta}}{p_*}\le \tilde{\mathcal{O}}(\epsilon)$ with a complexity shown as follows
    \begin{equation*}
        \tilde{\mathcal{O}}\left(L^4 \rho^{-3}\cdot \left(d+m_2^2\right)^2 Z^3\cdot \epsilon^{-2/(u-1)}\cdot 2^{u}\right).
    \end{equation*}
\end{corollary}
This result implies that if the energy function is infinite-order Lipschitz, we can nearly achieve any polynomial order convergence w.r.t. $\epsilon$ with the non-parametric energy difference estimation.

\noindent\textbf{RTK-ULD.}
Alg.~\ref{alg:inner_uld} presents a solution employing ULD for the inner loop, which can accelerate the convergence of the inner loop due to the better discretization of the ULD algorithm. 
When it is used to solve the $k$-th sampling subproblem of Alg.~\ref{alg:rtk}, $\vx_0$ is equal to $\hat{\rvx}_{k\eta}$ defined in Section~\ref{sec:framework} and used to initialize particles iterating in Alg.~\ref{alg:inner_mala}.
Besides, the underlying distribution of noise sample pair is
{\small
\begin{equation*}
    (\xi^z_s,\xi^v_s)\sim\mathcal{N}\left(\vzero,
\left[
\begin{matrix}
\frac{2}{\gamma} \left( \tau - \frac{2}{\gamma} \left(1 - e^{-\gamma\tau} \right)\right) + \frac{1}{2\gamma} \left(1 - e^{-2\gamma\tau}\right) & \frac{1}{\gamma} \left(1 - 2e^{-\gamma\tau} + e^{-2\gamma \tau}\right)  \\
 \frac{1}{\gamma} \left(1 - 2e^{-\gamma\tau} + e^{-2\gamma \tau}\right) & 1-e^{-2\gamma\tau}
\end{matrix}
\right]
\right).
\end{equation*}
}
In Alg.~\ref{alg:inner_uld}, we consider the process $\{(\hat{\rvz}_s, \hat{\rvv}_s)\}_{s=0}^S$ whose underlying distribution is denoted as $\{\hat{\pi}_s\}_{s=0}^{S}$.
We expect the $\vz$-marginal distribution of $\hat{\pi}_S$ to be close to the target distribution presented in Eq.~\ref{def:inner_target_dis}.
Unlike MALA, we only need to consider the error from score estimation in an expectation perspective, which is the same as that shown in~\cite{chen2022sampling}.
\begin{enumerate}[label=\textbf{[E{\arabic*}]}]
\setcounter{enumi}{2}
    \item \label{hat_e1}Estimation error of the score function: we assume a score estimator, e.g., a well-trained diffusion model, $s_{\theta}$, which can approximate the score function with an $\epsilon_{\mathrm{score}}$ error, i.e., $\E_{p_t}\left\|s_{\theta,t}(\vz) - \grad \log p_t(\vz)\right\|^2\leq \epsilon^2_{\mathrm{score}}$  for any $t\in[0,T]$.
\end{enumerate}
Under this condition, the complexity of RTK-ULD to achieve the convergence of TV distance is provided as follows, and the detailed proof is deferred to Theorem~\ref{thm:uld_outer_complexity_gene}.
Besides, we compare our theoretical results with the previous in Table~\ref{tab:comp_old}.
\begin{theorem}
    \label{thm:rtk_uld_complexity}
    Under Assumptions~\ref{ass:lips_score}--\ref{ass:second_moment} and \ref{hat_e1}, for Alg.~\ref{alg:rtk}, we choose 
    \begin{equation*}
        \eta= 1/2\cdot \log[(2L+1)/2L] \quad \mathrm{and}\quad K = 4L\cdot \log [((1+L^2)d+\left\|\grad f_*(\vzero)\right\|^2)^2\cdot\epsilon^{-2}]
    \end{equation*}
    and implement Step 3 of Alg.~\ref{alg:rtk} with projected Alg.~\ref{alg:inner_uld}.
    For the $k$-th run of Alg.~\ref{alg:inner_uld}, we require Gaussian-type initialization and high-accurate score estimation, i.e.,
    \begin{equation*}
        \hat{\pi}_0 = \mathcal{N}(\vzero, e^{2\eta}-1)\otimes \mathcal{N}(\vzero,\mI)\quad \mathrm{and}\quad \epsilon_{\mathrm{score}}=\tilde{\mathcal{O}}(\epsilon/\sqrt{L}).
    \end{equation*}
    If we set the hyperparameters of inner loops as follows. 
    the step size and the iteration number as
    \begin{equation*}
         \begin{aligned}
             \tau & = \tilde{\Theta}\left(\epsilon d^{-1/2} L^{-1/2} \cdot \left(\log \left[\frac{L(d+m_2^2+\|\vx_0\|^2)}{\epsilon^2}\right]\right)^{-1/2}\right)\\
            S &=\tilde{\Theta}\left(\epsilon^{-1}d^{1/2}\cdot \left(\log \left[\frac{L(d+m_2^2+\|\vx_0\|^2)}{\epsilon^2}\right]\right)^{1/2}\right).
         \end{aligned}
     \end{equation*}
    It can achieve $\TVD{\hat{p}_{K\eta}}{p_*}\lesssim \epsilon$
    with an $\tilde{\mathcal{O}}\left(L^2 d^{1/2}\epsilon^{-1} \right)$ gradient complexity.
\end{theorem}

\begin{table*}[t]
    \small
    \centering
    \begin{tabular}{cccc}
    \toprule
     Results & Algorithm & Assumptions &  Complexity \\
     \midrule
     \citet{chen2022sampling} & DDPM (SDE-based) & \ref{ass:lips_score},\ref{ass:second_moment},\ref{hat_e1} &  $\tilde{\mathcal{O}}(L^2d\epsilon^{-2})$\\
     \midrule
     \citet{chen2024probability} & DPOM (ODE-based) & \ref{ass:lips_score},\ref{ass:second_moment},\ref{hat_e1}, and $s_{\theta}$ smoothness &  $\tilde{\mathcal{O}}(L^3d\epsilon^{-2})$\\
     \midrule
     \citet{chen2024probability} & DPUM (ODE-based) & \ref{ass:lips_score},\ref{ass:second_moment},\ref{hat_e1}, and $s_{\theta}$ smoothness  &   $\tilde{\mathcal{\mathcal{O}}}(L^2d^{1/2}\epsilon^{-1})$\\
     \midrule
     \citet{li2023towards} &  ODE-based sampler & \ref{hat_e1} and estimation error of energy Hessian &   $\tilde{\mathcal{O}}(d^3\epsilon^{-1})$\\
     \midrule
     Corollary~\ref{cor:complexity_19_informal} & RTK-MALA & \ref{ass:lips_score},\ref{ass:second_moment},\ref{e1}, and \ref{e2} & \textcolor{red}{$\mathcal{O}(L^4 d^2 \log (d/\epsilon) )$}\\ 
     \midrule
     Theorem~\ref{thm:rtk_uld_complexity} & RTK-ULD (ours) & \ref{ass:lips_score},\ref{ass:second_moment},\ref{hat_e1} & \textcolor{red}{$\tilde{\mathcal{O}}(L^2d^{1/2}\epsilon^{-1})$}\\
     \bottomrule 
    \end{tabular}
    \caption{\small Comparison with prior works for RTK-based methods. The complexity denotes the number of calls for the score estimation to achieve $\TVD{\hat{p}_{K\eta}}{p_*}\le \tilde{\mathcal{O}}(\epsilon)$. $d$ and $\epsilon$ mean the dimension and error tolerance.
    Compared with the state-of-the-art result, RTK-ULD achieves the best dependence for both $d$ and $\epsilon$. Though RTK-MALA requires slightly stricter assumptions and worse dimension dependence, a linear convergence w.r.t. $\epsilon$ makes it suit high-accuracy sampling tasks.} 
    \label{tab:comp_old}
\end{table*}

\section{Conclusion and Limitation}

This paper presents an analysis of a modified version of diffusion models. Instead of focusing on the discretization of the reverse SDE,  we propose a general RTK framework that can produce a large class of algorithms for diffusion inference, which is formulated as solving a sequence of RTK sampling subproblems. Given this framework, we develop two algorithms called RTK-MALA and RTK-ULD, which leverage MALA and ULD to solve the RTK sampling subproblems. We develop theoretical guarantees for these two algorithms under certain conditions on the score estimation, and demonstrate their faster convergence rate than prior works. Numerical experiments support our theory.

We would also like to point out several limitations and future work. One potential limitation of this work is the lack of large-scale experiments. The main focus of this paper is the theoretical understanding and rigorous analysis of the diffusion process. Implementing large-scale experiments requires GPU resources and practitioner support, which can be an interesting direction for future work. Besides, though we provided a score-only RTK-MALA algorithm, the $\tilde{\mathcal{O}}(1/\epsilon)$ convergence rate can only be achieved by the RTK-MALA algorithm (Alg. \ref{alg:inner_mala}). However, this faster algorithm requires a direct approximation of the energy difference, which is not accessible in the existing pretrained diffusion model. Developing practical energy difference approximation algorithms and incorporating them with Alg. \ref{alg:inner_mala} for diffusion inference are also very interesting future directions.

\bibliographystyle{apalike}
\newpage
\bibliography{0_contents/ref}  





\newpage
\appendix

\section{Numerical Experiments}

In this section, we conduct experiments when the target distribution $p_*$ is a Mixture of Gaussian (MoG) and compare RTK-based methods with traditional DDPM.
Specifically, we are considering a forward process from an MoG distribution to a normal distribution in the following
\begin{equation*}        
\mathrm{d} \mathbf{x}_t = -\frac{1}{2}\mathbf{x}_t \mathrm{d} t +  \mathrm{d} \bm{B}_t \quad \text{and} \quad
\mathbf{x}_{0} \sim \frac{1}{K} \sum_{k=1}^K  \mathcal{N}(\vmu_k, \sigma_k^2\cdot \mI),
\end{equation*}

where $K$ is the number of Gaussian components, $\vmu_k$ and $\sigma_k^2$ are the means and variances of the Gaussian components, respectively.
The solution of the SDE follows
\begin{equation*}        
\rvx_t = \mathbf{x}_0e^{-\frac{1}{2} t}+ \sqrt{1-e^{- t}}\cdot \xi \quad \mathrm{where} \quad \xi \sim \mathcal{N}(0, \mI).
\end{equation*}
Since $\mathbf{x}_0$ and $\xi$ are both sampled from Gaussian distributions, their linear combination $\mathbf{x}_t$ also forms a Gaussian distribution, i.e., 
\begin{equation*}
    \mathbf{x}_t \sim \frac{1}{K}\sum_{k=1}^K  \mathcal{N}(\vmu_k e^{-\frac{1}{2}t}, (\sigma_k^2 e^{-t} + 1 -e^{-t})\cdot \mI).
\end{equation*}

Then, we have
\begin{equation*}               
\begin{aligned}            
\nabla  p(\vx_t) & = \frac{1}{K}\sum_{i=1}^{K} \grad_{\vx_t} \left[  \frac{1}{2}            (\frac{1}{\sqrt{2\pi}(\sigma_i^2e^{- t}+ 1- e^{- t}}) \cdot \exp (-\frac{1}{2}(\frac{\vx_t-\vmu_ie^{-\frac{1}{2} t}}{\sigma_i^2e^{- t}+ 1 - e^{- t}})^2)\right]\\        
& = \frac{1}{K}\sum_{i=1}^{K}  p_i(\vx_t) \cdot \grad_{\vx_t} \left[-\frac{1}{2}(\frac{\vx_t-\vmu_ie^{-\frac{1}{2} t}}{\sigma_k^2e^{- t} + 1 - e^{- t}})^2\right]\\            
& = \frac{1}{K}\sum_{i=1}^{K} p_i(\vx_t) \cdot \frac{-(\vx_t-\vmu_ie^{-\frac{1}{2} t})}{\sigma_i^2e^{- t}+1-e^{- t}}.
\end{aligned}    
\end{equation*}

We can also  calculate the score of $\vx_t$, i.e.,
\begin{equation*}               
\begin{aligned} 
\nabla \log p(\vx_t) =\frac{\nabla p(\vx_t)}{p(\vx_t)} = \frac{1/K\cdot \sum_{i=1}^{K}  p_i(\vx_t) \cdot \left(
\frac{-\left(\vx_t-\vmu_ie^{-\frac{1}{2} t} \right)}{\sigma_i^2e^{- t}+1-e^{- t}}\right)}            
{1/K\cdot \sum_{i=1}^{K}  p_i(\vx_t)}.
\end{aligned}    
\end{equation*}

We consider a MoG consisting of 12 Gaussian distributions, each with 10 dimensions, as shown in Fig.~\ref{figure:cluster} (f). The means of the 12 Gaussian distributions are uniformly distributed along the circumference of a circle with a radius of one in the first and second dimensions, while the remaining dimensions are centered at the origin. Each component of the mixture has an equal probability and a variance of 0.007 across all dimensions.

We evaluate Alg.~\ref{alg:rtk} with unadjust Langevin algorithm (ULA), which leads to RTK-ULA, Alg~\ref{alg:inner_mala},~\ref{alg:inner_uld} implementations, and DDPM under the same Number of Function Evaluations (NFE). Specifically, while DDPM models \(\mathbf{x}_{\eta}\) across a sequence of \(\eta\) timesteps spanning from 0 to $T$ in increments of $0.001 \times T$ (i.e., \([0, 0.001T, 0.002T, \ldots, T]\)), we execute Alg.~\ref{alg:rtk}, \ref{alg:inner_mala}, and \ref{alg:inner_uld} at fewer timesteps within \(\mathbf{x}_{[0, 0.2T, 0.4T, 0.6T, 0.8T]}\), and we distribute the NFE uniformly to these timesteps for MCMC. The experiments are taken on a single NVIDIA GeForce RTX 4090 GPU. We evaluate the sampling quality using marginal accuracy, i.e.,
\begin{equation*}               
  \text{Marginal Accuracy}(\hat{p}, p) = 1 - 0.5 \times \frac{1}{d}\sum_{i=1}^{d} TV(\hat{p}_i,p_i),
\end{equation*}
where \( \hat{p}_i(x) \) is the empirical marginal distribution of the \(i\)-th dimension obtained from the sampled data, \( p_i(x) \) is the true marginal distribution of the \(i\)-th dimension, and \( d \) is the total number of dimensions.

\begin{figure}[htb]
\centering{\includegraphics[width=1\textwidth]{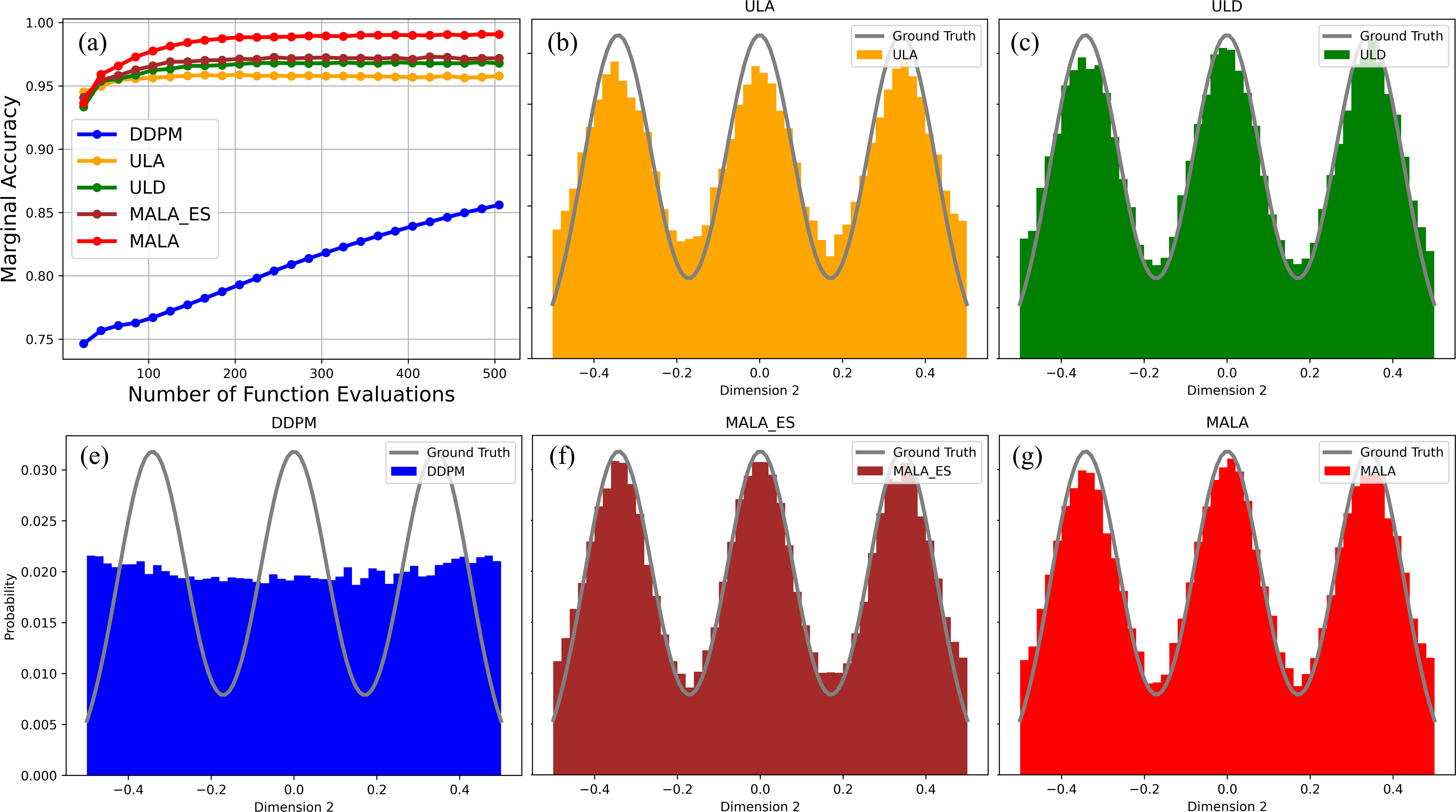}}
\caption{(a) Mariginal accuracy of the sampled MoG by different algorithms along NFE. (b-f) The histograms along a certain direction of sampled MoG by different algorithms. The plots labeled by `ULA', `ULD', `MALA', `MALA\_ES' correspond to RTK-ULA, RTK-ULD, RTK-MALA, score-only RTK-MALA, respectively. The histogram is oriented along the second dimension when the first dimension is constrained within (0.75, 1.25).}
\label{figure:nfe_ma}
\end{figure}

\begin{figure}[htb]
\centering{\includegraphics[width=1\textwidth]{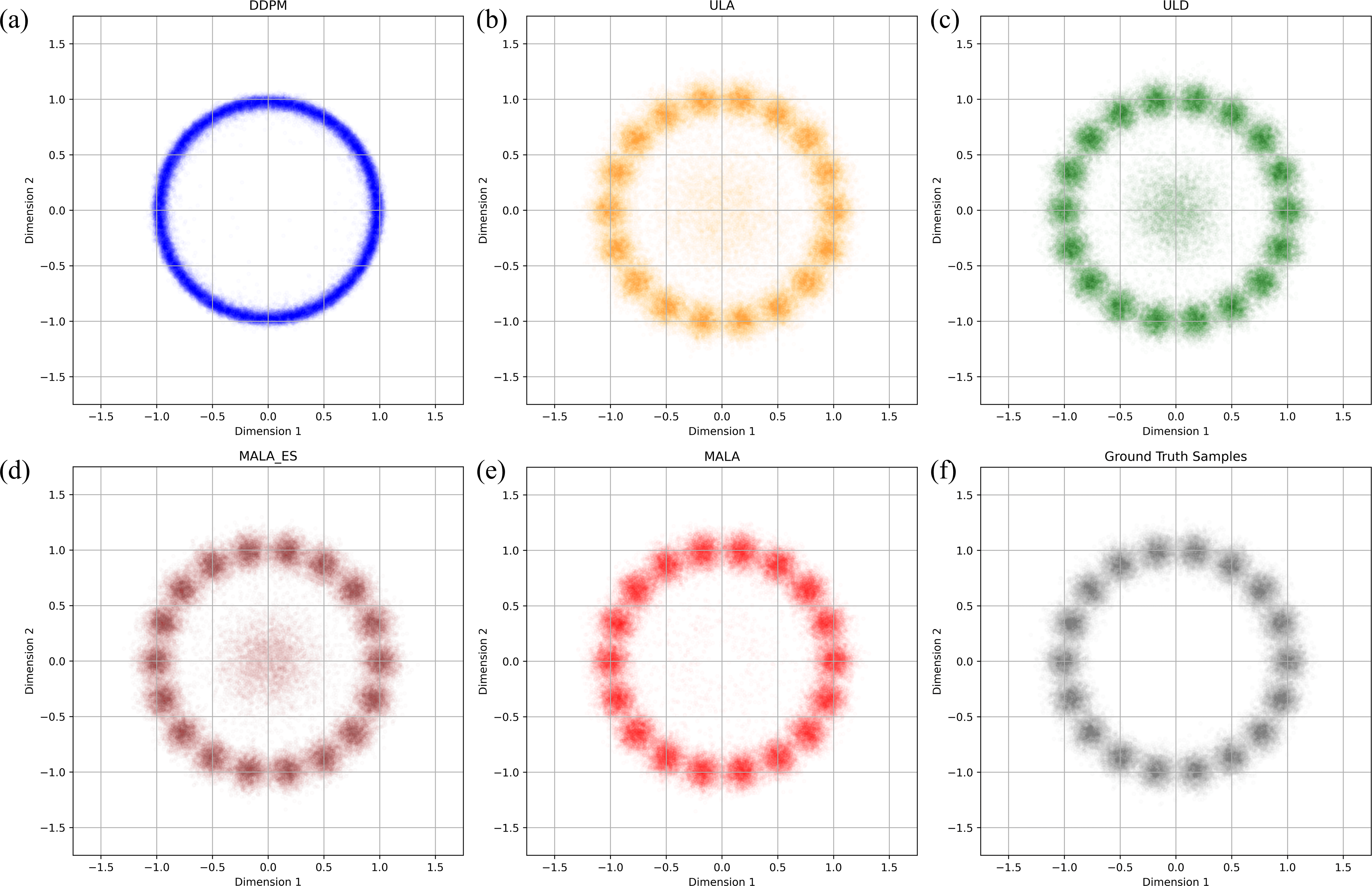}}
\caption{(a-e) Clusters sampled by DDPM, RTK-ULA, RTK-ULD, score-only RTK-MALA, and RTK-MALA, respectively. (f) Clusters sampled by the ground truth distribution. These \( 2D \) clusters represent the projection of the original \( 10D \) data onto the first two dimensions. }
\label{figure:cluster}
\end{figure}

Fig.~\ref{figure:nfe_ma} (a) shows the marginal accuracies of our RTK sampling algorithms and DDPM along NFE. We observe that all algorithms using RTK converge quickly. Among all RTK algorithms, RTK-MALA achieves the highest marginal accuracy. Score-only RTK-MALA is worse than RTK-MALA since the estimated energy contains errors, yet it is still slightly better than RTK-ULD. Along all RTK algorithms, RTK-ULA demonstrates the lowest performance in terms of marginal accuracy, but it still outperforms DDPM with a large margin especially when NFE is small. 

Fig.~\ref{figure:nfe_ma} (b-f) shows the histograms of sampled MoG by DDPM and RTK-based methods. We observe that DDPM cannot reconstruct the local structure of MoG. ULA can roughly reconstruct the MoG structure, but it is still weak in complex regions, specifically around the peaks and valleys. In contrast, RTK-ULD, score-only RTK-MALA, and RTK-MALA can reconstruct more fine-grained structures in complex regions.

Fig.~\ref{figure:cluster} (a-e) shows the clusters sampled by DDPM and RTK-based methods. We observe that DDPM fails to accurately reconstruct the ground truth distribution. In contrast, all methods based on RTK can generate distributions that closely approximate the ground truth. Additionally, RTK-MALA shows superior performance in accurately reconstructing the distribution in regions of low probability.

Overall, these numerical experiments demonstrate the benefit of the RTK framework for developing faster algorithms than DDPM in diffusion inference. Besides, experimental results also well support our theory, showing that RTK-MALA achieves faster convergence than RTK-ULA and RTK-ULD, even with estimated energy difference via score functions.

\section{Inference process with reverse transition kernel framework}
\label{sec:appendix_rtk_framework}

\begin{proof}[Proof of Lemma~\ref{lem:rev_trans_ker_form}]
\label{proof:rtk}
According to Bayes theorem, the following equation should be validated for any $\vx\in \R^d$ and $t^\prime>t$,
\begin{equation}
    \label{tbp:bias}
    p_t(\vx) = \int p_{t|t^\prime}(\vx|\vx^\prime) \cdot p_{t^\prime}(\vx^\prime)\der \vx^\prime.
\end{equation}
To simplify the notation, we suppose the normalizing constant of $p_t$, i.e.,
\begin{equation*}
    Z_t \coloneqq \int \exp(-f_t(\vx)) \der\vx.
\end{equation*}
Besides, the forward OU process, i.e., SDE.~\ref{sde:ou}, has a closed transition kernel, i.e.,
\begin{equation*}
    p_{t^\prime|t}(\vx^\prime|\vx)= \left(2\pi \left(1-e^{-2(t^\prime-t)}\right)\right)^{-d/2}
     \cdot \exp \left[\frac{-\left\|\vx^\prime -e^{-(t^\prime-t)}\vx \right\|^2}{2\left(1-e^{-2(t^\prime-t)}\right)}\right]
\end{equation*}
Then, we have
\begin{equation*}
    \begin{aligned}
        p_{t^\prime}(\vx^\prime) = &\int p_t(\vy) p_{t^\prime|t}(\vx^\prime|\vy)\der \vy\\
        = &\int Z_t^{-1}\cdot \exp(-f_t(\vy))\cdot \left(2\pi \left(1-e^{-2(t^\prime-t)}\right)\right)^{-d/2}
     \cdot \exp \left[\frac{-\left\|\vx^\prime -e^{-(t^\prime-t)}\vy \right\|^2}{2\left(1-e^{-2(t^\prime-t)}\right)}\right] \der\vy.
    \end{aligned}
\end{equation*}
Plugging this equation into Eq.~\ref{tbp:bias}, and we have
\begin{equation*}
    \small
    \begin{aligned}
        \text{RHS of Eq.~\ref{tbp:bias}} = &\int p_{t|t^\prime}(\vx|\vx^\prime)\cdot p_{t^\prime}(\vx^\prime) \der \vx^\prime\\
        = &\int p_{t|t^\prime}(\vx|\vx^\prime) \cdot \int Z_t^{-1}\cdot \exp(-f_t(\vy))\cdot \left(2\pi \left(1-e^{-2(t^\prime-t)}\right)\right)^{-d/2} \cdot \exp \left[\frac{-\left\|\vx^\prime -e^{-(t^\prime-t)}\vy \right\|^2}{2\left(1-e^{-2(t^\prime-t)}\right)}\right] \der\vy \der\vx^\prime.
    \end{aligned}
\end{equation*}
Moreover, when we plug the reverse transition kernel 
\begin{equation*}
    p_{t|t^\prime}(\vx|\vx^\prime) \propto \exp\left(-f_t(\vx)-\frac{\left\|\vx^\prime - \vx\cdot e^{-(t^\prime-t)}\right\|^2}{2(1-e^{-2(t^\prime-t)})}\right)
\end{equation*}
into the previous equation and have
\begin{equation*}
    \begin{aligned}
        \text{RHS of Eq.~\ref{tbp:bias}} 
        = & \int \frac{\exp\left( \textcolor{red}{-f_t(\vx)} \textcolor{blue}{-\frac{\left\|\vx^\prime - \vx\cdot 
        e^{-(t^\prime-t)}\right\|^2}{2(1-e^{-2(t^\prime-t)})}}\right)}
        {\textcolor{green}{\int \exp\left(-f_t(\vx)-\frac{\left\|\vx^\prime - \vx\cdot 
 e^{-(t^\prime-t)}\right\|^2}{2(1-e^{-2(t^\prime-t)})}\right)\der \vx}} \cdot\\
        & \int \textcolor{red}{Z_t^{-1}}\cdot \textcolor{green}{\exp(-f_t(\vy))}\cdot \textcolor{blue}{\left(2\pi \left(1-e^{-2(t^\prime-t)}\right)\right)^{-d/2}} \cdot \textcolor{green}{\exp \left[\frac{-\left\|\vx^\prime -e^{-(t^\prime-t)}\vy \right\|^2}{2\left(1-e^{-2(t^\prime-t)}\right)}\right]} \der\vy \der\vx^\prime\\
        & = \textcolor{red}{Z_t^{-1}\cdot \exp(-f_t(\vx))}\cdot \int \textcolor{blue}{\exp\left(-\frac{\left\|\vx^\prime - \vx\cdot 
        e^{-(t^\prime-t)}\right\|^2}{2(1-e^{-2(t^\prime-t)})}\right) \cdot \left(2\pi \left(1-e^{-2(t^\prime-t)}\right)\right)^{-d/2}} \cdot\\
        &\quad \left[\int \textcolor{green}{\frac{\exp\left(-f_t(\vy) - \frac{\left\|\vx^\prime - e^{-(t^\prime-t)}\cdot \vy\right\|^2}{2(1-e^{-2(t^\prime-t)})}\right)}{\int \exp\left(-f_t(\vx)-\frac{\left\|\vx^\prime - \vx\cdot 
 e^{-(t^\prime-t)}\right\|}{2(1-e^{-2(t^\prime-t)})}\right) \der \vx}}\der \vy\right] \der \vx^\prime\\
 & = \textcolor{red}{p_t(\vx)}=\text{LHS of Eq.~\ref{tbp:bias}}.
    \end{aligned}
\end{equation*}
Hence, the proof is completed.
\end{proof}

\begin{lemma}[Chain rule of TV]
    \label{lem:tv_chain_rule}
    Consider four random variables, $\rvx, \rvz, \tilde{\rvx}, \tilde{\rvz}$, whose underlying distributions are denoted as $p_x, p_z, q_x, q_z$.
    Suppose $p_{x,z}$ and $q_{x,z}$ denotes the densities of joint distributions of $(\rvx,\rvz)$ and $(\tilde{\rvx},\tilde{\rvz})$, which we write in terms of the conditionals and marginals as
    \begin{equation*}
        \begin{aligned}
        &p_{x,z}(\vx,\vz) = p_{x|z}(\vx|\vz)\cdot p_z(\vz)=p_{z|x}(\vz|\vx)\cdot p_{x}(\vx)\\
        &q_{x,z}(\vx,\vz)=q_{x|z}(\vx|\vz)\cdot q_z(\vz) = q_{z|x}(\vz|\vx)\cdot q_x(\vx).
        \end{aligned}
    \end{equation*}
    then we have
    \begin{equation*}
        \begin{aligned}
            \TVD{p_{x,z}}{q_{x,z}} \le  \min & \left\{ \TVD{p_z}{q_z} + \E_{\rvz\sim p_z}\left[\TVD{p_{x|z}(\cdot|\rvz)}{q_{x|z}(\cdot|\rvz)}\right],\right.\\
            &\quad  \left.\TVD{p_x}{q_x}+\E_{\rvx \sim p_x}\left[\TVD{p_{z|x}(\cdot|\rvx)}{q_{z|x}(\cdot|\rvx)}\right]\right\}.
        \end{aligned}
    \end{equation*}
    Besides, we have
    \begin{equation*}
        \TVD{p_x}{q_x}\le \TVD{p_{x,z}}{q_{x,z}}.
    \end{equation*}
\end{lemma}
\begin{proof}
    According to the definition of the total variation distance, we have
    \begin{equation*}
        \begin{aligned}
            \TVD{p_{x,z}}{q_{x,z}} = & \frac{1}{2}\int \int \left|p_{x,z}(\vx,\vz) - q_{x,z}(\vx,\vz)\right| \der\vz\der\vx\\
            = &  \frac{1}{2}\int\int \left|p_z(\vz)p_{x|z}(\vx|\vz) - p_z(\vz)q_{x|z}(\vx|\vz)+p_z(\vz)q_{x|z}(\vx|\vz) - q_z(\vz)q_{x|z}(\vx|\vz)\right| \der\vz\der\vx\\
            \le & \frac{1}{2}\int p_z(\vz) \int \left|p_{x|z}(\vx|\vz) - q_{x|z}(\vx|\vz)\right| \der \vx \der\vz + \frac{1}{2}\int \left|p_z(\vz) - q_z(\vz)\right|\int q_{x|z}(\vx|\vz) \der\vx\der \vz\\
            = & \E_{\rvz\sim p_z}\left[\TVD{p_{x|z}(\cdot|\rvz)}{q_{x|z}(\cdot|\rvz)}\right] + \TVD{p_z}{q_z}.
        \end{aligned}
    \end{equation*}
    With a similar technique, we have
    \begin{equation*}
        \TVD{p_{x,z}}{q_{x,z}}\le \TVD{p_x}{q_x}+\E_{\rvx \sim p_x}\left[\TVD{p_{z|x}(\cdot|\rvx)}{q_{z|x}(\cdot|\rvx)}\right].
    \end{equation*}
    Hence, the first inequality of this Lemma is proved.
    Then, for the second inequality, we have
    \begin{equation*}
        \begin{aligned}
            \TVD{p_x}{q_x} = &\frac{1}{2}\int \left|p_x(\vx) - q_x(\vx)\right| \der \vx \\
            = & \frac{1}{2}\int \left|\int p_{x,z}(\vx,\vz)\der \vz - \int q_{x,z}(\vx,\vz) \der \vz\right| \der\vx\\
            \le & \frac{1}{2}\int \int \left|p_{x,z}(\vx,\vz) - q_{x,z}(\vx,\vz)\right| \der\vz\der\vx\ = \TVD{p_{x,z}}{q_{x,z}}.
        \end{aligned}
    \end{equation*}
    Hence, the proof is completed.
\end{proof}

\begin{lemma}
    \label{lem:num_diff_balance_tv}
    For Alg~\ref{alg:rtk}, we have
    \begin{align*}
        \TVD{\hat{p}_{K\eta}}{p_*}\le & \sqrt{(1+L^2)d + \left\|\grad f_*(\vzero)\right\|^2} \cdot \exp(-K\eta) \\
            & + \sum_{k=0}^{K-1} \E_{\hat{\rvx}\sim \hat{p}_{k\eta}}\left[\TVD{\hat{p}_{(k+1)\eta|k\eta}(\cdot|\hat{\rvx})}{\bkwp_{(k+1)\eta|k\eta}(\cdot|\hat{\rvx})}\right]
    \end{align*}
    for any $K\in \mathbb{N}_+$ and $\eta\in\R_+$.
\end{lemma}

\begin{proof}
    For any $k\in\{0,1,\ldots, K-1\}$, let $\hat{p}_{(k+1)\eta, k\eta}$ and $\bkwp_{(k+1)\eta,k\eta}$ denote the joint distribution of $(\hat{\rvx}_{(k+1)\eta},\hat{\rvx}_{k\eta})$ and $(\rbkwx_{(k+1)\eta}, \rbkwx_{k\eta})$, which we write in term of the conditionals and marginals as
    \begin{equation*}
        \begin{aligned}
            &\hat{p}_{(k+1)\eta,k\eta}(\vx^\prime,\vx) = \hat{p}_{(k+1)\eta|k\eta}(\vx^\prime|\vx) \cdot \hat{p}_{k\eta}(\vx) =  \hat{p}_{k\eta|(k+1)\eta}(\vx|\vx^\prime)\cdot \hat{p}_{(k+1)\eta}(\vx^\prime)\\
            &\bkwp_{(k+1)\eta,k\eta}(\vx^\prime, \vx) = \bkwp_{(k+1)\eta|k\eta}(\vx^\prime | \vx) \cdot \bkwp_{k\eta}(\vx) = \bkwp_{k\eta|(k+1)\eta}(\vx|\vx^\prime)\cdot \bkwp_{(k+1)\eta}(\vx^\prime).
        \end{aligned}
    \end{equation*}
    Under this condition, we have
    \begin{equation*}
        \begin{aligned}
            \TVD{\hat{p}_{K\eta}}{p_*} = \TVD{\hat{p}_{K\eta}}{\bkwp_{K\eta}} & \le \TVD{\hat{p}_{K\eta, (K-1)\eta}}{\bkwp_{K\eta,(K-1)\eta}}\\
            & \le \TVD{\hat{p}_{(K-1)\eta}}{\bkwp_{(K-1)\eta}} + \E_{\hat{\rvx}\sim \hat{p}_{(K-1)\eta}}\left[\TVD{\hat{p}_{K\eta|(K-1)\eta}(\cdot|\hat{\rvx})}{\bkwp_{K\eta|(K-1)\eta}(\cdot|\hat{\rvx})}\right]
        \end{aligned}
    \end{equation*}
    where the inequalities follow from Lemma~\ref{lem:tv_chain_rule}.
    By using the inequality recursively, we have
    \begin{equation}
        \label{ineq:cumulative_error}
        \begin{aligned}
            \TVD{\hat{p}_{K\eta}}{p_*} \le & \TVD{\hat{p}_0}{\bkwp_{0}} + \sum_{k=0}^{K-1} \E_{\hat{\rvx}\sim \hat{p}_{k\eta}}\left[\TVD{\hat{p}_{(k+1)\eta|k\eta}(\cdot|\hat{\rvx})}{\bkwp_{(k+1)\eta|k\eta}(\cdot|\hat{\rvx})}\right]\\
            = & \underbrace{\TVD{p_{\infty}}{p_{K\eta}}}_{\text{Term\ 1}} + \sum_{k=0}^{K-1} \E_{\hat{\rvx}\sim \hat{p}_{k\eta}}\left[\TVD{\hat{p}_{(k+1)\eta|k\eta}(\cdot|\hat{\rvx})}{\bkwp_{(k+1)\eta|k\eta}(\cdot|\hat{\rvx})}\right]
        \end{aligned}
    \end{equation}
    where $p_\infty$ denotes the stationary distribution of the forward process.
    In this analysis, $p_\infty$ is the standard since the forward SDE.~\ref{sde:ou}, whose negative log density is $1$-strongly convex and also satisfies LSI with constant $1$ due to Lemma~\ref{lem:strongly_lsi}.
    \paragraph{For Term 1.} we have
    \begin{equation*}
        \begin{aligned}
            \TVD{p_\infty}{p_{K\eta}}\le & \sqrt{\frac{1}{2}\KL{p_{K\eta}}{p_\infty}}\le \sqrt{\frac{1}{2}\cdot \exp\left(-2K\eta\right)\cdot \KL{p_0}{p_\infty}}\\
            \le & \sqrt{(1+L^2)d + \left\|\grad f_*(\vzero)\right\|^2} \cdot \exp(-K\eta)
        \end{aligned}
    \end{equation*}
    where the first inequality follows from Pinsker's inequality, the second one follows from Lemma~\ref{thm4_vempala2019rapid}, and the last one follows from Lemma~\ref{lem:init_error_bound}.
    It should be noted that the smoothness of $p_0$ required in Lemma~\ref{lem:init_error_bound} is given by~\ref{ass:lips_score}.
    
    Plugging this inequality into Eq.~\ref{ineq:cumulative_error}, we have
    \begin{equation*}
        \begin{aligned}
            \TVD{\hat{p}_{K\eta}}{p_*}\le & \sqrt{(1+L^2)d + \left\|\grad f_*(\vzero)\right\|^2} \cdot \exp(-K\eta) \\
            & + \sum_{k=0}^{K-1} \E_{\hat{\rvx}\sim \hat{p}_{k\eta}}\left[\TVD{\hat{p}_{(k+1)\eta|k\eta}(\cdot|\hat{\rvx})}{\bkwp_{(k+1)\eta|k\eta}(\cdot|\hat{\rvx})}\right]
        \end{aligned} 
    \end{equation*}
    Hence, the proof is completed.
\end{proof}

\begin{corollary}
    \label{lem:tv_inner_conv}
    For Alg~\ref{alg:rtk}, if we set 
    \begin{equation*}
        \eta= \frac{1}{2}\cdot \log \frac{2L+1}{2L},\quad \quad K= 4L\cdot \log \frac{(1+L^2)d+\left\|\grad f_*(\vzero)\right\|^2}{\epsilon^2}
    \end{equation*}
    and
    \begin{equation*}
        \TVD{\hat{p}_{(k+1)\eta|k\eta}(\cdot|\hat{\vx})}{\bkwp_{(k+1)\eta|k\eta}(\cdot|\hat{\vx})}\le \frac{\epsilon}{K} = \frac{\epsilon}{4L}\cdot \left[\log \frac{(1+L^2)d+\left\|\grad f_*(\vzero)\right\|^2}{\epsilon^2}\right]^{-1},
    \end{equation*}
    we have the total variation distance between the underlying distribution of Alg~\ref{alg:rtk} output and the data distribution $p_*$ will satisfy $\TVD{\hat{p}_{K\eta}}{p_*}\le 2\epsilon$.    
\end{corollary}
\begin{proof}
    According to Lemma~\ref{lem:num_diff_balance_tv}, we have
    \begin{equation*}
        \begin{aligned}
            \TVD{\hat{p}_{K\eta}}{p_*}\le & \sqrt{(1+L^2)d + \left\|\grad f_*(\vzero)\right\|^2} \cdot \exp(-K\eta) \\
            & \underbrace{+ \sum_{k=0}^{K-1} \E_{\hat{\rvx}\sim \hat{p}_{k\eta}}\left[\TVD{\hat{p}_{(k+1)\eta|k\eta}(\cdot|\hat{\rvx})}{\bkwp_{(k+1)\eta|k\eta}(\cdot|\hat{\rvx})}\right]}_{\mathrm{Term\ 2}}
        \end{aligned} 
    \end{equation*}
    for any $K\in \mathbb{N}_+$ and $\eta\in\R_+$.
    To achieve the upper bound $\TVD{p_\infty}{p_{K\eta}}\le \epsilon$, we only require
    \begin{equation}
        \label{ineq:mix_time}
        T= K\eta \ge \frac{1}{2}\log \frac{(1+L^2)d+\left\|\grad f_*(\vzero)\right\|^2}{\epsilon^2}.
    \end{equation}
    \paragraph{For Term 2.} For any $\vx\in\R^d$, the formulation of $\bkwp_{k+1|k}(\cdot|\hat{\vx})$ is 
    \begin{equation*}
        \bkwp_{(k+1)\eta|k\eta}(\vx|\hat{\vx}) = p_{(K-k-1)\eta | (K-1)\eta}(\vx|\hat{\vx}) \propto \exp\left(-f_{(K-k-1)\eta}(\vx)-\frac{\left\|\hat{\vx} - \vx\cdot 
 e^{-\eta}\right\|^2}{2(1-e^{-2\eta})}\right),
    \end{equation*}
    whose negative log Hessian satisfies
    \begin{equation*}
        -\grad^2_{\vx}\log \bkwp_{(k+1)\eta|k\eta}(\vx|\hat{\vx}) = \grad^2 f_{(K-k-1)\eta}(\vx) + \frac{e^{-2\eta}}{1-e^{-2\eta}}\cdot \mI \succeq \left(\frac{e^{-2\eta}}{1-e^{-2\eta}}  - L\right)\cdot \mI.
    \end{equation*}
    Note that the last inequality follows from~\ref{ass:lips_score}.
    In this condition, if we require 
    \begin{equation*}
        \left(\frac{e^{-2\eta}}{1-e^{-2\eta}}  - L\right) \ge L\quad \Leftrightarrow \quad \eta\le \frac{1}{2}\log \frac{2L+1}{2L},
    \end{equation*}
    then we have
    \begin{equation*}
        \frac{e^{-2\eta}}{2(1-e^{-2\eta})} \cdot \mI \preceq -\grad^2_{\vx}\log \bkwp_{(k+1)\eta|k\eta}(\vx|\hat{\vx}) \preceq \frac{3e^{-2\eta}}{2(1-e^{-2\eta})}\cdot \mI.
    \end{equation*}
    To simplify the following analysis, we choose $\eta$ to its upper bound, and we know for all $k\in\{0,1,\ldots, K-1\}$, the conditional density $\bkwp_{k+1|k}(\vx|\hat{\vx})$ is strongly-log concave, and its score is $3L$-Lipschitz.
    Besides, combining Eq.~\ref{ineq:mix_time} and the choice of $\eta$, we require
    \begin{equation*}
        K = T/\eta \ge \log \frac{(1+L^2)d+\left\|\grad f_*(\vzero)\right\|^2}{\epsilon^2} \Big/ \log \frac{2L+1}{2L}
    \end{equation*}
    which can be achieved by 
    \begin{equation*}
        K\coloneqq 4L\cdot \log \frac{(1+L^2)d+\left\|\grad f_*(\vzero)\right\|^2}{\epsilon^2}
    \end{equation*}
    when we suppose $L\ge 1$ without loss of generality.
    In this condition, if there is a uniform upper bound for all conditional probability approximation, i.e., 
    \begin{equation*}
        \TVD{\hat{p}_{(k+1)\eta|k\eta}(\cdot|\hat{\vx})}{\bkwp_{(k+1)\eta|k\eta}(\cdot|\hat{\vx})}\le \frac{\epsilon}{K} = \frac{\epsilon}{4L}\cdot \left[\log \frac{(1+L^2)d+\left\|\grad f_*(\vzero)\right\|^2}{\epsilon^2}\right]^{-1},
    \end{equation*}
    then we can find Term 2 in Eq.~\ref{ineq:cumulative_error} will be upper bounded by $\epsilon$.
    Hence, the proof is completed.
\end{proof}

\begin{lemma}[Chain rule of KL]
    \label{lem:kl_chain_rule}
    Consider four random variables, $\rvx, \rvz, \tilde{\rvx}, \tilde{\rvz}$, whose underlying distributions are denoted as $p_x, p_z, q_x, q_z$.
    Suppose $p_{x,z}$ and $q_{x,z}$ denotes the densities of joint distributions of $(\rvx,\rvz)$ and $(\tilde{\rvx},\tilde{\rvz})$, which we write in terms of the conditionals and marginals as
    \begin{equation*}
        \begin{aligned}
        &p_{x,z}(\vx,\vz) = p_{x|z}(\vx|\vz)\cdot p_z(\vz)=p_{z|x}(\vz|\vx)\cdot p_{x}(\vx)\\
        &q_{x,z}(\vx,\vz)=q_{x|z}(\vx|\vz)\cdot q_z(\vz) = q_{z|x}(\vz|\vx)\cdot q_x(\vx).
        \end{aligned}
    \end{equation*}
    then we have
    \begin{equation*}
        \begin{aligned}
            \KL{p_{x,z}}{q_{x,z}} = & \KL{p_z}{q_z} + \E_{\rvz\sim p_z}\left[\KL{p_{x|z}(\cdot|\rvz)}{q_{x|z}(\cdot|\rvz)}\right]\\
            = & \KL{p_x}{q_x}+\E_{\rvx \sim p_x}\left[\KL{p_{z|x}(\cdot|\rvx)}{q_{z|x}(\cdot|\rvx)}\right]
        \end{aligned}
    \end{equation*}
    where the latter equation implies
    \begin{equation*}
        \KL{p_x}{q_x}\le \KL{p_{x,z}}{q_{x,z}}.
    \end{equation*}
\end{lemma}
\begin{proof}
    According to the formulation of KL divergence, we have
    \begin{equation*}
        \begin{aligned}
            \KL{p_{x,z}}{q_{x,z}} = &\int p_{x,z}(\vx, \vz) \log \frac{p_{x,z}(\vx, \vz)}{q_{x,z}(\vx, \vz)}\der (\vx, \vz)\\
            = & \int p_{x,z}(\vx,\vz)\left(\log \frac{p_x(\vx)}{q_x(\vx)} + \log \frac{p_{z|x}(\vz|\vx)}{q_{z|x}(\vz|\vx)}\right) \der(\vx,\vz)\\
            = & \int p_{x,z}(\vx, \vz)\log \frac{p_x(\vx)}{q_x(\vx)}\der(\vx, \vz)+\int p_x(\vx) \int p_{z|x}(\vz|\vx)\log \frac{p_{z|x}(\vz|\vx)}{q_{z|x}(\vz|\vx)}\der \vz \der\vx\\
            = & \KL{p_x}{q_x}+\E_{\rvx \sim p_x}\left[\KL{p_{z|x}(\cdot|\rvx)}{q_{z|x}(\cdot|\rvx)}\right] \ge \KL{p_x}{q_x},
        \end{aligned}
    \end{equation*}
    where the last inequality follows from the fact
    \begin{equation*}
        \KL{p_{z|x}(\cdot|\vx)}{\tilde{p}_{z|x}(\cdot|\vx)}\ge 0\quad \forall\ \vx.
    \end{equation*}
    With a similar technique, it can be obtained that
    \begin{equation*}
        \begin{aligned}
            \KL{p_{x,z}}{q_{x,z}} = &\int p_{x,z}(\vx, \vz) \log \frac{p_{x,z}(\vx, \vz)}{q_{x,z}(\vx, \vz)}\der (\vx, \vz)\\
            = & \int p_{x,z}(\vx,\vz)\left(\log \frac{p_z(\vz)}{q_z(\vz)} + \log \frac{p_{x|z}(\vx|\vz)}{q_{x|z}(\vx|\vz)}\right) \der(\vx,\vz)\\
            = & \int p_{x,z}(\vx, \vz)\log \frac{p_z(\vz)}{q_z(\vz)}\der(\vx, \vz)+\int p_z(\vz) \int p_{x|z}(\vx|\vz)\log \frac{p_{x|z}(\vx|\vz)}{q_{x|z}(\vx|\vz)}\der \vz \der\vx\\
            = & \KL{p_z}{q_z}+\E_{\rvz \sim p_z}\left[\KL{p_{x|z}(\cdot|\rvz)}{\tilde{p}_{x|z}(\cdot|\rvz)}\right].
        \end{aligned}
    \end{equation*}
    Hence, the proof is completed.
\end{proof}

\begin{proof}[Proof of Lemma~\ref{lem:num_diff_balance_kl}]
    This Lemma uses nearly the same techniques as those in Lemma~\ref{lem:num_diff_balance_tv}, while it may have a better smoothness dependency in convergence since the chain rule of KL divergence.
    Hence, we will omit several steps overlapped in Lemma~\ref{lem:num_diff_balance_tv}.
    
    For any $k\in\{0,1,\ldots, K-1\}$, let $\hat{p}_{(k+1)\eta, k\eta}$ and $\bkwp_{(k+1)\eta,k\eta}$ denote the joint distribution of $(\hat{\rvx}_{(k+1)\eta},\hat{\rvx}_{k\eta})$ and $(\rbkwx_{(k+1)\eta}, \rbkwx_{k\eta})$, which we write in term of the conditionals and marginals as
    \begin{equation*}
        \begin{aligned}
            &\hat{p}_{(k+1)\eta,k\eta}(\vx^\prime,\vx) = \hat{p}_{(k+1)\eta|k\eta}(\vx^\prime|\vx) \cdot \hat{p}_{k\eta}(\vx) =  \hat{p}_{k\eta|(k+1)\eta}(\vx|\vx^\prime)\cdot \hat{p}_{(k+1)\eta}(\vx^\prime)\\
            &\bkwp_{(k+1)\eta,k\eta}(\vx^\prime, \vx) = \bkwp_{(k+1)\eta|k\eta}(\vx^\prime | \vx) \cdot \bkwp_{k\eta}(\vx) = \bkwp_{k\eta|(k+1)\eta}(\vx|\vx^\prime)\cdot \bkwp_{(k+1)\eta}(\vx^\prime).
        \end{aligned}
    \end{equation*}
    Under this condition, we have
    \begin{equation*}
        \begin{aligned}
            \TVD{\hat{p}_{K\eta}}{p_*} = & \TVD{\hat{p}_{K\eta}}{\bkwp_{K\eta}}  \le \sqrt{\frac{1}{2}\KL{\hat{p}_{K\eta}}{\bkwp_{K\eta}}} \le \sqrt{\frac{1}{2}\KL{\hat{p}_{K\eta, (K-1)\eta}}{\bkwp_{K\eta,(K-1)\eta}}}\\
            & \le \sqrt{\frac{1}{2}\KL{\hat{p}_{(K-1)\eta}}{\bkwp_{(K-1)\eta}} + \frac{1}{2}\E_{\hat{\rvx}\sim \hat{p}_{(K-1)\eta}}\left[\KL{\hat{p}_{K\eta|(K-1)\eta}(\cdot|\hat{\rvx})}{\bkwp_{K\eta|(K-1)\eta}(\cdot|\hat{\rvx})}\right]}
        \end{aligned}
    \end{equation*}
    where the first inequality follows from Pinsker's inequality, the second and the third inequalities follow from Lemma~\ref{lem:kl_chain_rule}.
    By using the inequality recursively, we have
    \begin{equation}
        \label{ineq:cumulative_error_kl}
        \begin{aligned}
            & \TVD{\hat{p}_{K\eta}}{p_*} \le  \sqrt{\frac{1}{2}\KL{\hat{p}_0}{\bkwp_{0}} + \frac{1}{2}\sum_{k=0}^{K-1} \E_{\hat{\rvx}\sim \hat{p}_{k\eta}}\left[\KL{\hat{p}_{(k+1)\eta|k\eta}(\cdot|\hat{\rvx})}{\bkwp_{(k+1)\eta|k\eta}(\cdot|\hat{\rvx})}\right]}\\
            & = \sqrt{\frac{1}{2}\KL{p_{\infty}}{p_{K\eta}}} + \sqrt{\frac{1}{2}\sum_{k=0}^{K-1} \E_{\hat{\rvx}\sim \hat{p}_{k\eta}}\left[\KL{\hat{p}_{(k+1)\eta|k\eta}(\cdot|\hat{\rvx})}{\bkwp_{(k+1)\eta|k\eta}(\cdot|\hat{\rvx})}\right]}\\
            & \le \sqrt{(1+L^2)d + \left\|\grad f_*(\vzero)\right\|^2} \cdot \exp(-K\eta) + \sqrt{\frac{1}{2}\sum_{k=0}^{K-1} \E_{\hat{\rvx}\sim \hat{p}_{k\eta}}\left[\KL{\hat{p}_{(k+1)\eta|k\eta}(\cdot|\hat{\rvx})}{\bkwp_{(k+1)\eta|k\eta}(\cdot|\hat{\rvx})}\right]}
        \end{aligned}
    \end{equation}
    where the last inequality follows from Lemma~\ref{lem:init_error_bound}.
    Hence, the proof is completed.
\end{proof}

\begin{corollary}
\label{lem:kl_inner_conv}
    For Alg~\ref{alg:rtk}, if we set 
    \begin{equation*}
        \eta= \frac{1}{2}\cdot \log \frac{2L+1}{2L},\quad \quad K= 4L\cdot \log \frac{(1+L^2)d+\left\|\grad f_*(\vzero)\right\|^2}{\epsilon^2}
    \end{equation*}
    and
    \begin{equation*}
        \KL{\hat{p}_{(k+1)\eta|k\eta}(\cdot|\hat{\vx})}{p_{(K-k-1)\eta|(K-k)\eta}(\cdot|\hat{\vx})} \le \frac{\epsilon^2}{4L}\cdot \left[\log \frac{(1+L^2)d+\left\|\grad f_*(\vzero)\right\|^2}{\epsilon^2}\right]^{-1},
    \end{equation*}
    we have the total variation distance between the underlying distribution of Alg~\ref{alg:rtk} output and the data distribution $p_*$ will satisfy $\TVD{\hat{p}_{K\eta}}{p_*}\le 2\epsilon$.   
\end{corollary}
\begin{proof}
    According to Lemma~\ref{lem:num_diff_balance_kl}, we have
    \begin{equation}
        \begin{aligned}
            \TVD{\hat{p}_{K\eta}}{p_*} \le  & \underbrace{\sqrt{(1+L^2)d + \left\|\grad f_*(\vzero)\right\|^2} \cdot \exp(-K\eta)}_{\text{Term\ 1}}\\
            &  + \underbrace{\sqrt{\frac{1}{2}\sum_{k=0}^{K-1} \E_{\hat{\rvx}\sim \hat{p}_{k\eta}}\left[\KL{\hat{p}_{(k+1)\eta|k\eta}(\cdot|\hat{\rvx})}{\bkwp_{(k+1)\eta|k\eta}(\cdot|\hat{\rvx})}\right]}}_{\text{Term\ 2}}
        \end{aligned}
    \end{equation}
    To achieve the upper bound $\text{Term 1}\le \epsilon$, we only require
    \begin{equation}
        \label{ineq:mix_time_kl}
        T= K\eta \ge \frac{1}{2}\log \frac{(1+L^2)d+\left\|\grad f_*(\vzero)\right\|^2}{\epsilon^2}.
    \end{equation}
    For $\text{Term 2}$, by choosing 
    \begin{equation*}
        \eta= \frac{1}{2}\log \frac{2L+1}{2L},
    \end{equation*}
    we know for all $k\in\{0,1,\ldots, K-1\}$, the conditional density $\bkwp_{k+1|k}(\vx|\hat{\vx})$ is strongly-log concave, and its score is $3L$-Lipschitz.
    In this condition, we require 
    \begin{equation*}
        K\coloneqq 4L\cdot \log \frac{(1+L^2)d+\left\|\grad f_*(\vzero)\right\|^2}{\epsilon^2}
    \end{equation*}
    when we suppose $L\ge 1$ without loss of generality.
    Then, to achieve $\text{Term 2}\le \epsilon$, the sufficient condition is to require a uniform upper bound for all conditional probability approximation, i.e., 
    \begin{equation*}
        \KL{\hat{p}_{(k+1)\eta|k\eta}(\cdot|\hat{\vx})}{\bkwp_{k+1|k}(\cdot|\hat{\vx})}\le \frac{\epsilon^2}{K} = \frac{\epsilon^2}{4L}\cdot \left[\log \frac{(1+L^2)d+\left\|\grad f_*(\vzero)\right\|^2}{\epsilon^2}\right]^{-1}.
    \end{equation*}
    Hence, the proof is completed.
\end{proof}
\begin{remark}
    To achieve the TV error tolerance shown in Corollary~\ref{lem:tv_inner_conv}, .i.e.,
    \begin{equation*}
        \TVD{\hat{p}_{(k+1)\eta|k\eta}(\cdot|\hat{\vx})}{\bkwp_{(k+1)\eta|k\eta}(\cdot|\hat{\vx})}\le \frac{\epsilon}{4L}\cdot \left[\log \frac{(1+L^2)d+\left\|\grad f_*(\vzero)\right\|^2}{\epsilon^2}\right]^{-1},
    \end{equation*}
    it requires the KL divergence error to satisfy
    \begin{equation*}
        \begin{aligned}
            \TVD{\hat{p}_{(k+1)\eta|k\eta}(\cdot|\hat{\vx})}{\bkwp_{(k+1)\eta|k\eta}(\cdot|\hat{\vx})}&\le \sqrt{\frac{1}{2}\KL{\hat{p}_{(k+1)\eta|k\eta}(\cdot|\hat{\vx})}{\bkwp_{(k+1)\eta|k\eta}(\cdot|\hat{\vx})}}\\
            & \le \frac{\epsilon^2}{16L^2}\cdot \left[\log \frac{(1+L^2)d+\left\|\grad f_*(\vzero)\right\|^2}{\epsilon^2}\right]^{-2}.
        \end{aligned}
    \end{equation*}
    Compared with the results shown in Corollary~\ref{lem:kl_inner_conv}, this result requires a higher accuracy with an $\mathcal{O}(L)$ factor, which is not acceptable sometimes.
\end{remark}

\begin{lemma}
    \label{lem:second_moment_bound}
    Suppose Assumption~\ref{ass:lips_score}-\ref{ass:second_moment} hold, the choice of $\eta$ keeps the same as that in Corollary~\ref{lem:kl_inner_conv}, and the second moment of the underlying distribution of $\hat{\rvx}_{k\eta}$ is $M_k$, then we have 
    \begin{equation*}
        M_{k+1} \le \frac{2\delta_k}{L} + 16(d+m_2^2) + 24M_k.
    \end{equation*}
\end{lemma}
\begin{proof}
    Considering the second moment of $\hat{\rvx}_{(k+1)\eta}$, we have
    \begin{equation}
        \label{eq:2ndmoment_ite}
        \begin{aligned}
            \E_{\hat{p}_{(k+1)\eta}}\left[\left\|\hat{\rvx}_{(k+1)\eta}\right\|^2\right] = &\int \hat{p}_{(k+1)\eta}(\vx) \cdot \|\vx\|^2 \der\vx\\
            = & \int \left(\int \hat{p}_{k\eta}(\vy)\cdot \hat{p}_{(k+1)\eta|k\eta}(\vx|\vy) \der\vy\right)\cdot \|\vx\|^2 \der\vx\\
            = & \int \hat{p}_{k\eta}(\vy) \cdot \int \hat{p}_{(k+1)\eta|k\eta}(\vx|\vy)\cdot \|\vx\|^2 \der\vx \der\vy.
        \end{aligned}
    \end{equation}
    Then, we focus on the innermost integration, suppose $\hat{\gamma}_{\vy}(\cdot,\cdot)$ as the optimal coupling between $\hat{p}_{(k+1)\eta|k\eta}(\cdot|\vy)$ and ${p}^{\gets}_{(k+1)\eta|k\eta}(\cdot | \vy)$. 
    Then, we have
    \begin{equation}
        \label{ineq:pati_to_closed_2ndmoment}
        \begin{aligned}
            &\int \hat{p}_{(k+1)\eta|k\eta}(\vx|\vy)\left\|\vx\right\|^2 \der \vx - 2\int p^\gets_{(k+1)\eta|k\eta}(\vx|\vy)\left\|\vx\right\|^2 \der\vx \\
            &\le \int \hat{\gamma}_{\vy}(\hat{\vx}, \vx) \left(\left\|\hat{\vx}\right\|^2 - 2\left\|\vx\right\|^2 \right) \der (\hat{\vx},\vx) \le \int \hat{\gamma}_{\vy}(\hat{\vx}, \vx) \left\|\hat{\vx}-\vx\right\|^2 \der (\hat{\vx},\vx)\\
            & = W_2^2\left(\hat{p}_{(k+1)\eta|k\eta}, {p}^\gets_{(k+1)\eta|k\eta}\right).
        \end{aligned}
    \end{equation}
    Since ${p}^\gets_{(k+1)\eta|k\eta}$ is strongly log-concave, i.e.,
    \begin{equation*}
        -\grad^2_{\vx}\log \bkwp_{(k+1)\eta|k\eta}(\vx|\hat{\vx}) = \grad^2 f_{(K-k-1)\eta}(\vx) + \frac{e^{-2\eta}}{1-e^{-2\eta}}\cdot \mI \succeq L\mI,
    \end{equation*}
    the distribution ${p}^\gets_{(k+1)\eta|k\eta}$ also satisfies $1/L$ log-Sobolev inequality due to Lemma~\ref{lem:strongly_lsi}.
    By Talagrand's inequality, we have
    \begin{equation}
        \label{ineq:talagrand_closed_inner}
        W_2^2\left(\hat{p}_{(k+1)\eta|k\eta}, {p}^\gets_{(k+1)\eta|k\eta}\right) \le \frac{2}{L}\cdot  \KL{\hat{p}_{k+1|k+\frac{1}{2},b}}{{p}_{k+1|k+\frac{1}{2},b}}\coloneqq \frac{2\delta_k}{L}.
    \end{equation}
    Plugging Eq~\ref{ineq:pati_to_closed_2ndmoment} and Eq~\ref{ineq:talagrand_closed_inner} into Eq~\ref{eq:2ndmoment_ite}, we have
    \begin{equation}
        \label{eq:2ndmoment_ite_mid}
        \E\left[\left\|\hat{\rvx}_{(k+1)\eta}\right\|^2\right] \le \int \hat{p}_{k\eta}(\vy)\cdot \left(\frac{2\delta_k}{L} + 2\int p_{(k+1)\eta|k\eta}(\vx|\vy)\left\|\vx\right\|^2 \der\vx \right) \der \vy.
    \end{equation}
    To upper bound the innermost integration, we suppose the optimal coupling between $p_{(K-k-1)\eta}$ and $p^\gets_{(k+1)\eta|k\eta}(\cdot |\vy)$ is $\gamma_{\vy}(\cdot,\cdot)$. Then it has
    \begin{equation}
        \label{ineq:closed_to_target_2ndmoment}
        \begin{aligned}
            & \int p^{\gets}_{(k+1)\eta|k\eta}(\vx|\vy)\left\|\vx\right\|^2 \der\vx - 2\int p_{(K-k-1)\eta}(\vx)\left\|\vx\right\|^2 \der \vx\\
            & \le \int \gamma_{\vy}(\vx^\prime, \vx)\left(\left\|\vx^\prime\right\|^2 - 2\left\|\vx\right\|^2\right)\der (\vx^\prime, \vx) \le \int \gamma_{\vy}(\vx^\prime, \vx)\left\|\vx^\prime-\vx\right\|^2\der (\vx^\prime, \vx)\\
            & = W_2^2(p_{(K-k-1)\eta}, p^\gets_{(k+1)\eta|k\eta})
        \end{aligned}
    \end{equation}
    Since ${p}^\gets_{(k+1)\eta|k\eta}$ satisfies LSI with constant $1/L$. 
    By Talagrand's inequality and LSI, we have
    \begin{equation*}
        \begin{aligned}
            & W_2^2(p_{(K-k-1)\eta}, p^\gets_{(k+1)\eta|k\eta}) \le  \frac{2}{L}\cdot \KL{p_{(K-k-1)\eta}}{p_{(k+1)\eta|k\eta}} \\
            & \le \frac{4}{L^2}\cdot \int p_{(K-k-1)\eta}(\vx)\cdot \left\|\grad \log \frac{p_{(K-k-1)\eta}(\vx)}{p^\gets_{(k+1)\eta|k\eta}(\vx|\vy)}\right\|^2 \der \vx \\
            & =\frac{4}{L^2}\cdot  \int p_{(K-k-1)\eta}(\vx)\cdot \left\|\frac{e^{-\eta}\vy - e^{-2\eta}\vx}{1-e^{-2\eta}}\right\|^2 \der \vx\\
            & \le 12\left\|\vy\right\|^2 + 8\int p_{(K-k-1)\eta}(\vx)\|\vx\|^2 \der\vx\\
            & \le 12\|\vy\|^2 + 8(d+m_2^2).
        \end{aligned}
    \end{equation*}
    where the last inequality follows from the choice of $\eta = 1/2\cdot \log (2L+1)/2L$ and the fact $E_{p_{(K-k-1)\eta}}[\|\rvx\|^2]\le (d+m_2^2)$ obtained by Lemma~\ref{lem:lem10_chen2022sampling}.
    Plugging this results into Eq.~\ref{eq:2ndmoment_ite_mid}, we have
    \begin{equation*}
        \E\left[\left\|\hat{\rvx}_{(k+1)\eta}\right\|^2\right] \le \frac{2\delta_k}{L} + 16(d+m_2^2) + 24\cdot \E\left[\left\|\hat{\rvx}_{k\eta}\right\|^2\right].
    \end{equation*}
\end{proof}

\section{Implement RTK inference with MALA}
In this section, we consider introducing a MALA variant to sample from $\bkwp_{k+1|k}(\vz|\vx_0)$. 
To simplify the notation, we set 
\begin{equation}
    \label{def:energy_inner_mala}
    g(\vz) \coloneqq f_{(K-k-1)\eta}(\vz)+\frac{\left\|\vx_0 - \vz\cdot e^{-\eta}\right\|^2}{2(1-e^{-2\eta})}
\end{equation}
and consider $k$ and $\vx_0$ to be fixed.
Besides, we set 
\begin{equation*}
    \bkwp(\vz|\vx_0) \coloneqq \bkwp_{k+1|k}(\vz|\vx_0) \propto \exp(-g(\vz))
\end{equation*}
According to Corollary~\ref{lem:kl_inner_conv} and Corollary~\ref{lem:tv_inner_conv}, when we choose
\begin{equation*}
    \eta= \frac{1}{2}\log \frac{2L+1}{2L},
\end{equation*}
the log density $g$ will be $L$-strongly log-concave and $3L$-smooth.
With the following two approximations,
\begin{equation}
    \label{def:energy_score_estimation}
    \vs_{\vtheta}(\vz)\approx \grad g(\vz)\quad \mathrm{and}\quad r_{\vtheta^\prime}(\vz,\vz^\prime) \approx  g(\vz)-g(\vz^\prime),
\end{equation}
We left the approximation level here and determined when we needed the detailed analysis.
we can use the following Algorithm to replace Line~\ref{step:inner_sampler} of Alg.~\ref{alg:rtk}.

In this section, we introduce several notations about three transition kernels presenting the standard, the projected, and the ideally projected implementation of Alg.~\ref{alg:inner_mala}.

\paragraph{Standard implementation of Alg.~\ref{alg:inner_mala}.}
According to Step~\ref{step:ula}, the transition distribution satisfies
\begin{equation}
    \label{eq:std_ula_dis}
    Q_{\vz_s} = \mathcal{N}\left(\vz_s - \tau\cdot s_\theta(\vz_s), 2\tau\right)
\end{equation}
with a density function
\begin{equation}
    \label{eq:std_ula_pdf}
    q(\tilde{\vz}_s|\vz_s) = \varphi_{2\tau}\left(\tilde{\vz}_s - \left(\vz_s - \tau\cdot s_\theta(\vz_s)\right)\right).
\end{equation}
Considering a $1/2$-lazy version of the update, we set
\begin{equation}
    \label{eq:std_lazy}
    \gT^\prime_{\vz_s}(\der\vz^\prime) = \frac{1}{2}\cdot \delta_{\vz_s}(\der \vz^\prime) + \frac{1}{2}\cdot Q_{\vz_s}(\der\vz^\prime).
\end{equation}
Then, with the following Metropolis-Hastings filter,
\begin{equation}
    \label{eq:std_acc_rate}
    a_{\vz_s}(\vz^\prime) = \min\left\{1, \frac{q(\vz_s|\vz^\prime)}{q(\vz^\prime|\vz_s)} \cdot \exp\left(-r_\theta( \vz^\prime,\vz_s)\right)\right\}\quad \mathrm{where}\quad a_{\vz_s}(\vz^\prime) = a(\vz^\prime - (\vz_s - \tau\cdot s_{\theta}(\vz_s)), \vz_s), 
\end{equation}
the transition kernel for the standard implementation of Alg,~\ref{alg:inner_mala} will be 
\begin{equation}
    \label{eq:std_transker}
    \gT_{\vz_s}(\der \vz_{s+1}) =  \gT^\prime_{\vz_s}(\der \vz_{s+1})\cdot a_{\vz_s}(\vz_{s+1})+\left(1- \int a_{\vz_s}(\vz^\prime) \gT^\prime_{\vz_s}(\der \vz^\prime)\right)\cdot \delta_{\vz_s}(\der\vz_{s+1}).
\end{equation}

\paragraph{Projected implementation of Alg.~\ref{alg:inner_mala}.}
According to Step~\ref{step:ula}, the transition distribution satisfies
\begin{equation*}
    \tilde{Q}_{\vz_s} = \mathcal{N}\left(\vz_s - \tau\cdot s_\theta(\vz_s), 2\tau\right)
\end{equation*}
with a density function
\begin{equation*}
    \tilde{q}(\tilde{\vz}_s|\vz_s) = \varphi_{2\tau}\left(\tilde{\vz}_s - \left(\vz_s - \tau\cdot \grad s_\theta(\vz_s)\right)\right).
\end{equation*}
Considering the projection operation, i.e., Step~\ref{step:inball} in Alg~\ref{alg:rtk}, if we suppose the feasible set 
\begin{equation*}
    \Omega = \mathcal{B}(\vzero,R)\quad \mathrm{and}\quad \Omega_{\vz} = \mathcal{B}(\vz,r)
    \cap \mathcal{B}(\vzero,R)
\end{equation*}
the transition distribution becomes
\begin{equation*}
    \tilde{Q}^\prime_{\vz_s}(\gA) = \int_{\gA\cap \Omega_{\vz_s}} \tilde{Q}_{\vz_s}(\der\vz^\prime) + \int_{\gA - \Omega_{\vz_s}}\tilde{Q}_{\vz_s}(\der\vz^\prime) \cdot \delta_{\vz_s}(\gA).
\end{equation*}
Hence, a $1/2$-lazy version of the transition distribution becomes
\begin{equation*}
    \tilde{\gT}^\prime_{\vz_s}(\der\vz^\prime) = \frac{1}{2}\cdot \delta_{\vz_s}(\der\vz^\prime) + \frac{1}{2}\cdot \tilde{Q}^\prime_{\vz_s}(\der \vz^\prime).
\end{equation*}
Then, with the following Metropolis-Hastings filter,
\begin{equation*}
    \tilde{a}_{\vz_s}(\vz^\prime) = \min\left\{1, \frac{\tilde{q}(\vz_s|\vz^\prime)}{\tilde{q}(\vz^\prime|\vz_s)} \cdot \exp\left(-r_\theta( \vz^\prime,\vz_s)\right)\right\}\quad \mathrm{where}\quad \tilde{a}_{\vz_s}(\vz^\prime) = a(\vz^\prime - (\vz_s - \tau\cdot s_{\theta}(\vz_s)), \vz_s), 
\end{equation*}
the transition kernel for the projected implementation of Alg,~\ref{alg:inner_mala} will be 
\begin{equation*}
    \tilde{\gT}_{\vz_s}(\der \vz_{s+1}) =  \tilde{\gT}^\prime_{\vz_s}(\der \vz_{s+1})\cdot \tilde{a}_{\vz_s}(\vz_{s+1})+\left(1- \int_\Omega \tilde{a}_{\vz_s}(\vz^\prime) \tilde{\gT}^\prime_{\vz_s}(\der \vz^\prime)\right)\cdot \delta_{\vz_s}(\der\vz_{s+1}).
\end{equation*}

\paragraph{Ideally projected implementation of Alg.~\ref{alg:inner_mala}.} In this condition, we know the  accurate $g(\vz)-g(\vz^\prime)$ and $\grad g(\vz)$. 
In this condition, the ULA step will provide
\begin{equation}
    \label{eq:idea_ula}
    \tilde{Q}_{*, \vz_s} = \mathcal{N}\left(\vz_s - \tau\cdot \grad g(\vz_s), 2\tau\right)
\end{equation}
with a density function
\begin{equation*}
    \tilde{q}_*(\tilde{\vz}_s|\vz_s) = \varphi_{2\tau}\left(\tilde{\vz}_s - \left(\tau\cdot \grad g(\vz_s)\right)\right).
\end{equation*}
Considering the projection operation, i.e., Step~\ref{step:inball} in Alg~\ref{alg:rtk}, the transition distribution becomes
\begin{equation}
    \label{eq:idea_proj}
    \tilde{Q}^\prime_{*, \vz_s}(\gA) = \int_{\gA\cap \Omega_{\vz_s}} \tilde{Q}_{*,\vz_s}(\der\vz^\prime) + \int_{\gA - \Omega_{\vz_s}}\tilde{Q}_{*,\vz_s}(\der\vz^\prime) \cdot \delta_{\vz_s}(\gA).
\end{equation}
Hence, a $1/2$-lazy version of the transition distribution becomes
\begin{equation}
    \label{eq:idea_lazy}
    \tilde{\gT}^\prime_{*, \vz_s}(\der\vz^\prime) = \frac{1}{2}\cdot \delta_{\vz_s}(\der\vz^\prime) + \frac{1}{2}\cdot \tilde{Q}^\prime_{*,\vz_s}(\der \vz^\prime).
\end{equation}
Then, with the following Metropolis-Hastings filter,
\begin{equation}
    \label{eq:idea_acc_rate}
    \tilde{a}_{*, \vz_s}(\vz^\prime) = \min\left\{1, \frac{\tilde{q}_*(\vz_s|\vz^\prime)}{\tilde{q}_*(\vz^\prime|\vz_s)} \cdot \exp\left(-\left(g(\vz^\prime)-g(\vz_s)\right)\right)\right\},
\end{equation}
the transition kernel for the accurate projected update will be 
\begin{equation}
    \label{def:ideal_proj_rtk}
    \tilde{\gT}_{*, \vz_s}(\der \vz_{s+1}) =  \tilde{\gT}^\prime_{*, \vz_s}(\der \vz_{s+1})\cdot \tilde{a}_{*, \vz_s}(\vz_{s+1})+\left(1- \int_\Omega \tilde{a}_{*, \vz_s}(\vz^\prime) \tilde{\gT}^\prime_{*, \vz_s}(\der \vz^\prime)\right)\cdot \delta_{\vz_s}(\der\vz_{s+1}).
\end{equation}

\begin{lemma}
    \label{lem:sc_sm_of_g}
    Suppose we have
    \begin{equation*}
        \eta= \frac{1}{2}\log \frac{2L+1}{2L},
    \end{equation*}
    then the target distribution of the Inner MALA, i.e., $p^\gets(\vz|\vx_0)$ will be $L$-strongly log-concave and $3L$-smooth for any given $\vx_0$.
\end{lemma}
\begin{proof}
    Consider the energy function $g(\vz)$ of $p^\gets(\vz|\vx_0)$, we have
    \begin{equation*}
        g(\vz) = f_{(K-k-1)\eta}(\vz)+\frac{\left\|\vx_0 - \vz\cdot e^{-\eta}\right\|^2}{2(1-e^{-2\eta})}
    \end{equation*}
    whose Hessian matrix satisfies
    \begin{equation*}
        \begin{aligned}
            \left(\frac{e^{-2\eta}}{(1-e^{-2\eta})} + L\right)\cdot \mI \succeq \grad^2 g(\vz) = \grad^2 f_{(K-k-1)}(\vz) + \frac{e^{-2\eta}}{(1-e^{-2\eta})}\cdot \mI \succeq \left(\frac{e^{-2\eta}}{(1-e^{-2\eta})} - L\right)\cdot \mI.
        \end{aligned}
    \end{equation*}
    Under these conditions, if we have
    \begin{equation*}
        \eta \le \frac{1}{2}\log \frac{2L+1}{2L}\quad \Leftrightarrow \quad \frac{e^{-2\eta}}{1-e^{-2\eta}}\ge 2L,
    \end{equation*}
    which means
    \begin{equation*}
        \frac{3e^{-2\eta}}{2(1-e^{-2\eta})} \succeq \grad^2 g(\vz) \succeq \frac{e^{-2\eta}}{2(1-e^{-2\eta})}.
    \end{equation*}
    For the analysis convenience, we set 
    \begin{equation*}
        \eta = \frac{1}{2}\log \frac{2L+1}{2L},
    \end{equation*}
    that is to say $g(\vz)$ is $L$-strongly convex and $3L$-smooth.
\end{proof}

\subsection{Control the error from the projected transition kernel}
Here, we consider the marginal distribution of $\{\rvz_s\}$ and $\{\tilde{\rvz}_s\}$ to be the random process when Alg.~\ref{alg:inner_mala} is implemented by the standard and projected version, respectively.
The underlying distributions of these two processes are denoted as $\rvz_s\sim \mu_s$ and $\tilde{\rvz}_s \sim \tilde{\mu}_s$, and we would like to upper bound $\TVD{\mu_S}{\tilde{\mu}_S}$ for any given $\vx_0$.

Rewrite the formulation of $\rvz_S$, we have
\begin{equation*}
    \rvz_S = \tilde{\rvz}_S\cdot \vone\left(\rvz_S = \tilde{\rvz}_S\right)+\rvz_S\cdot \vone\left(\rvz_S\not=\tilde{\rvz}_S\right)
\end{equation*}
where $\vone(\cdot)$ is the indicator function.
In this condition, for any set $\gA$, we have
\begin{equation*}
    \begin{aligned}
        \vone\left(\rvz_S \in \gA \right) = & \vone\left(\tilde{\rvz}_S \in \gA \right) \cdot \vone\left(\rvz_S = \tilde{\rvz}_S \right) + \vone\left(\rvz_S\in \gA \right) \cdot \vone\left(\rvz_S\not=\tilde{\rvz}_S \right)\\
        = & \vone\left(\tilde{\rvz}_S\in \gA \right) - \vone\left(\tilde{\rvz}_S \in \gA \right) \cdot \vone\left(\rvz_S \not= \tilde{\rvz}_S\right) + \vone\left(\rvz_S\in \gA\right) \cdot \vone\left(\rvz_S\not=\tilde{\rvz}_S\right),
        \end{aligned}
\end{equation*}
which means
\begin{equation*}
     - \vone\left(\tilde{\rvz}_S \in \gA \right) \cdot \vone\left(\rvz_S \not= \tilde{\rvz}_S \right) \le \vone\left(\rvz_S \in \gA \right) - \vone(\tilde{\rvz}_S\in \gA ) \le \vone\left(\rvz_S\in \gA\right) \cdot \vone\left(\rvz_S\not=\tilde{\rvz}_S\right).
    \end{equation*}
Therefore, the total variation distance between $\mu_{S}$ and $\hat{\mu}_{S}$ can be upper bounded with
\begin{equation*}
    \TVD{\mu_S}{\tilde{\mu}_S} \le \sup_{\gA\subseteq \R^d}\left|\mu_S(\gA) - \tilde{\mu}_S(\gA)\right|\le \vone\left(\rvz_S\not=\tilde{\rvz}_S\right).
\end{equation*}
Hence, to require $\TVD{\mu_S}{\tilde{\mu}_S} \le \epsilon/4$ a sufficient condition is to consider $\mathrm{Pr}[\rvz_S\not=\tilde{\rvz}_S]$.
The next step is to show that, in Alg.~\ref{alg:inner_mala}, the projected version generates the same outputs as that of the standard version with probability at least $1-\epsilon/4$.
It suffices to show that with probability at least $1-\epsilon/4$, projected MALA will accept all $S$ iterates.
In this condition, let $\{\vz_1, \vz_2, \ldots, \vz_S\}$ be the iterates generated by the standard MALA (without the projection step), our goal is to prove that with probability at least $1-\epsilon/4$ all $\rvz_s$ stay inside the region $\mathcal{B}(\vzero, R)$ and $\left\|\vz_{s}-\vz_{s-1}\right\|\le r$ for all $s\le S$.
That means we need to prove the following two facts
\begin{enumerate}
    \item With probability at least $1-\epsilon/8$, all iterates stay inside the region $\mathcal{B}(\vzero, R)$.
    \item With probability at least $1-\epsilon/8$, $\left\|\rvx_s-\rvx_{s-1}\right\|\le r$ for all $s\le S$.
\end{enumerate}

\begin{lemma}
    \label{lem:proj_gap_lem}
    Let $\mu_S$ and $\tilde{\mu}_S$ be distributions of the outputs of standard and projected implementation of Alg.~\ref{alg:inner_mala}. 
    For any $\epsilon\in(0,1)$, we set
    \begin{equation*}
        R\ge \max\left\{8\cdot \sqrt{\frac{\|\grad g(\vzero)\|^2}{L^2}+\frac{d}{L}}, 63\cdot \sqrt{\frac{d}{L}\log\frac{16S}{\epsilon}}\right\},\quad r \ge (\sqrt{2}+1)\cdot \sqrt{\tau d} + 2\sqrt{\tau \log \frac{8S}{\epsilon}} 
    \end{equation*}
    where $\vz_*$ is denoted as the global optimum of the energy function, i.e., $g$, defined in Eq.~\ref{def:energy_inner_mala}.
    Suppose $\mathrm{P}(\|\rvz_0\|\ge R/2)\le \epsilon/4$ and set
    \begin{equation*}
        \tau \le \min\left\{\frac{d}{(3LR+\left\|\grad g(\vzero)\right\|+\epsilon_{\mathrm{score}})^2} ,\frac{16 d}{L^2R^2}\right\} = \frac{d}{(3LR+\left\|\grad g(\vzero)\right\|+\epsilon_{\mathrm{score}})^2},
    \end{equation*}
    then we have
    \begin{equation*}
        \TVD{\mu_S}{\tilde{\mu}_S}\le \frac{\epsilon}{4}.
    \end{equation*}
\end{lemma}
\begin{proof}
    We borrow the proof techniques provided in Lemma 6.1 of~\cite{zou2021faster} to control the TVD gap between the standard and the projected implementation of Alg.~\ref{alg:inner_mala}.

    \paragraph{Particles stay inside $\gB(\vzero, R)$.} We first consider the expectation of $\left\|\rvz_{s+1}\right\|^2$ when $\rvz_{s}$ is given, and have
    \begin{equation}
        \label{ineq:update_norm_upb}
        \begin{aligned}
            & \E\left[\left\|\rvz_{s+1}\right\|^2 \Big| \vz_{s}\right] =  \int \left\|\vz^\prime\right\|^2 \gT_{\vz_s}(\der \vz^\prime)\\
            & = \int \left\|\vz^\prime\right\|^2 \cdot \left[\gT^\prime_{\vz_s}(\der\vz^\prime) \cdot a_{\vz_s}(\vz^\prime) + \left(1- \int a_{\vz_s}(\tilde{\vz})\gT^\prime_{\vz_s}(\der\tilde{\vz})\right) \delta_{\vz_s}(\der \vz^\prime)\right] \\
            & = \left\|\vz_s\right\|^2 +\int \left(\left\|\vz^\prime\right\|^2 - \left\|\vz_s\right\|^2\right) \cdot a_{\vz_s}(\vz^\prime)\gT_{\vz_s}^\prime(\der\vz^\prime)\cdot \\
            & = \left\|\vz_s\right\|^2 + \int \left(\left\|\vz^\prime\right\|^2 - \left\|\vz_s\right\|^2\right)\cdot a_{\vz_s}(\vz^\prime) \cdot\left(\frac{1}{2}\cdot\delta_{\vz_s}(\der\vz^\prime)+\frac{1}{2}\cdot Q_{\vz_s}(\der\vz^\prime)\right) \\
            & =  \left\|\vz_s\right\|^2  +  \frac{1}{2}\int \left(\left\|\vz^\prime\right\|^2 - \left\|\vz_s\right\|^2\right) \cdot  \min\left\{q(\vz^\prime|\vz_s), q(\vz_s|\vz^\prime)\cdot \exp\left(-r_\theta( \vz^\prime,\vz_s)\right) \right\} \der\vz^\prime\\
            & \le \frac{1}{2}\left\|\vz_s\right\|^2 + \frac{1}{2}\int \left\|\vz^\prime\right\|^2 \cdot q(\vz^\prime|\vz_s)\der\vz^\prime,
        \end{aligned}
    \end{equation}
    where the second equation follows from Eq.~\ref{eq:std_transker}, the forth equation follows from Eq.~\ref{eq:std_lazy} and the fifth equation follows from Eq.~\ref{eq:std_acc_rate} and Eq.~\ref{eq:std_ula_pdf}.
    Note that $q(\vz^\prime|\vz_s)$ is a Gaussian-type distribution whose mean and variance are $\vz_s - \tau\cdot  s_{\theta}(\vz_s)$ and $2\tau$ respectively.
    It means
    \begin{equation}
        \label{ineq:update_norm_upb_2}
        \int \left\|\vz^\prime\right\|^2 \cdot q(\vz^\prime|\vz_s)\der\vz^\prime = \left\|\vz_s - \tau \cdot s_{\theta}(\vz_s) \right\|^2 + 2\tau d.
    \end{equation}
    Suppose $\vz_*$ is the global optimum of the function $g$ due to Lemma~\ref{lem:sc_sm_of_g}, we have
    \begin{equation}
        \label{ineq:update_norm_upb_3}
        \begin{aligned}
            &\left\|\vz_s - \tau \cdot s_{\theta}(\vz_s) \right\|^2 =  \left\|\vz_s\right\|^2 - 2\tau\cdot \vz_s^\top s_\theta(\vz_s) + \tau^2 \cdot \left\|s_\theta(\vz_s)\right\|^2\\
            & =  \left\|\vz_s\right\|^2  - 2\tau\cdot \vz_s^\top \grad g(\vz_s) + 2\tau\cdot \vz_s^\top \left(s_\theta(\vz_s) - \grad g(\vz_s)\right) + \tau^2\cdot \left\|s_\theta(\vz_s) - \grad g(\vz_s) + \grad g(\vz_s)\right\|^2\\
            & \le \left\|\vz_s\right\|^2 - 2\tau \cdot \left(\frac{L\left\|\vz_s\right\|^2}{2} - \frac{\left\|\grad g(\vzero)\right\|^2}{2L}\right) + \tau^2\cdot \left\|\vz_s\right\|^2 + \left\|s_\theta(\vz_s) - \grad g(\vz_s)\right\|^2\\
            &\quad + 2\tau^2 \cdot \left\|\grad g(\vz_s)\right\|^2 + 2\tau^2 \cdot \left\|s_\theta(\vz_s) - \grad g(\vz_s)\right\|^2\\
            & = \left(1-L\tau + \tau^2\right)\cdot \left\|\vz_s\right\|^2  + \tau\cdot \left\|\grad g(\vzero)\right\|^2/L + (1+2\tau^2)\epsilon^2_{\mathrm{score}} + 2\tau^2\cdot \left\|\grad g(\vz_s)\right\|^2\\
            & \le \left(1-L\tau + (1+36L^2)\cdot \tau^2\right)\cdot \left\|\vz_s\right\|^2  + \tau\cdot \left\|\grad g(\vzero)\right\|^2/L + 4\tau^2\cdot \left\|\grad g(\vzero)\right\|^2 + (1+2\tau^2)\epsilon^2_{\mathrm{score}},
        \end{aligned}
    \end{equation}
    where the first inequality follows from the combination of $L$-strong convexity of $g$ and Lemma~\ref{lem:sc_to_dissp} , the second inequality follows from the $3L$-smoothness of $g$
    The strong convexity and the smoothness of $g$ follow from Lemma~\ref{lem:sc_sm_of_g}.
    
    Combining Eq.~\ref{ineq:update_norm_upb}, Eq.~\ref{ineq:update_norm_upb_2} and Eq.~\ref{ineq:update_norm_upb_3}, we have
    \begin{equation*}
        \begin{aligned}
            \E\left[\left\|\rvz_{s+1}\right\|^2 \Big| \vz_{s}\right] \le &\left(1 - \frac{L\tau}{2} + \frac{1+36L^2}{2}\cdot \tau^2\right)\cdot \left\|\vz_s\right\|^2\\
            & + \left(\frac{\tau}{2L} + 2\tau^2\right)\cdot \left\|\grad g(\vzero)\right\|^2+ \frac{(1+2\tau^2)\epsilon^2_{\mathrm{score}}}{2} + \tau d.
        \end{aligned}
    \end{equation*}
    By requiring $\epsilon_{\mathrm{score}}\le \tau\le L/(2+72L^2) <1$, we have
    \begin{equation*}
        \begin{aligned}
            \E\left[\left\|\rvz_{s+1}\right\|^2 \Big| \vz_{s}\right] \le &\left(1 - \frac{L\tau}{4}\right)\cdot \left\|\vz_s\right\|^2 + \frac{\tau}{L}\cdot \left\|\grad g(\vzero)\right\|^2 + (2+d)\tau.
        \end{aligned}
    \end{equation*}
    Suppose a radio $R$ satisfies
    \begin{equation}
        \label{ineq:choice_R_1}
        R\ge 8\cdot \sqrt{\frac{\|\grad g(\vzero)\|^2}{L^2}+\frac{d}{L}}.
    \end{equation}
    Then, if $\|\vz_s\|\ge R/2 \ge  4\sqrt{\|\grad g(\vzero)\|^2/L^2 + d/L}$, it has
    \begin{equation*}
        \begin{aligned}
            & \left\|\vz_s\right\|^2\ge  16\cdot \left(\frac{\|\grad g(\vzero)\|^2}{L^2} + \frac{d}{L}\right) \ge  \frac{8\|\grad g(\vzero)\|^2}{L^2}+\frac{8\cdot (2+d)}{L}\\
            & \Leftrightarrow\quad  \frac{L\tau \left\|\vz_s\right\|^2}{8}\ge \frac{\tau}{L}\cdot \left\|\grad g(\vzero)\right\|^2 + (2+d)\tau\\
            & \Leftrightarrow\quad  \E\left[\left\|\rvz_{s+1}\right\|^2 \Big| \rvz_{s}\right] \le \left(1 - \frac{L\tau}{8}\right)\cdot \left\|\vz_s\right\|^2.
        \end{aligned}
    \end{equation*}
    To prove $\|\vz_s\|\le R$ for all $s\le S$, we only need to consider $\vz_s$ satisfying $\|\vz_s\|\ge 4\sqrt{\|\grad g(\vzero)\|^2/L^2 + d/L}$, otherwise $\|\vz_s\|\le R/2\le R$ naturally holds. 
    Then, by the concavity of the function $\log(\cdot)$, for any $\|\vz_s\|\ge R/2$, we have
    \begin{equation}
        \label{ineq:super_mg_0}
        \E\left[\log (\|\rvz_{s+1}\|^2)|\rvz_s\right] \le \log \E\left[\|\rvz_{s+1}\|^2 | \rvz_s\right]\le \log(1-\frac{L\tau}{4}) + \log(\|\vz_s\|^2) \le \log(\|\vz_s\|^2) - \frac{L\tau}{4}.
    \end{equation}
    Consider the random variable
    \begin{equation*}
        \tilde{\rvz}_s \coloneqq \vz_s - \tau\cdot s_\theta(\vz_s) + \sqrt{2\tau}\cdot \xi\quad\mathrm{where}\quad \xi\sim \mathcal{N}(\vzero, \mI)
    \end{equation*}
    obtained by the transition kernel Eq.~\ref{eq:std_ula_dis},
    Note that $\|\xi\|$ is the square root of a $\chi(d)$ random variable, which is subgaussian and satisfies
    \begin{equation*}
        \mathbb{P}\left[\|\xi\|\ge \sqrt{d}+ \sqrt{2}t\right]\le e^{-t^2}
    \end{equation*}
    for any $t\ge 0$.
    Under these conditions, requiring 
    \begin{equation}
        \label{ineq:tau_choice_0}
        \tau\le (3LR+G+\epsilon_{\mathrm{score}})^{-2} \cdot d \quad \mathrm{where}\quad G\coloneqq \left\|\grad g(\vzero)\right\|,
    \end{equation}
    we have
    \begin{equation}
        \label{ineq:prob_upb_vardiff}
        \begin{aligned}
            &\mathbb{P}\left[\left\|\rvz_{s+1}\right\| - \left\|\vz_s\right\| \ge 3\sqrt{\tau d} + 2\sqrt{\tau}t\right]\le \mathbb{P}\left[\left\|\tilde{\rvz}_s\right\| - \left\|\vz_s\right\| \ge 3\sqrt{\tau d} + 2\sqrt{\tau}t\right]\\
            & \le \mathbb{P}\left[\tau \left\|s_\theta(\vz_s)\right\| + \sqrt{2\tau}\left\|\xi\right\|\ge 3\sqrt{\tau d} + 2\sqrt{\tau}t\right] \le  \mathbb{P}\left[\sqrt{2\tau}\|\xi\|\ge \sqrt{2\tau d}+ 2\sqrt{\tau}t\right]\le e^{-t^2}.
        \end{aligned}
    \end{equation}
    In Eq.~\ref{ineq:prob_upb_vardiff}, the first inequality follows from the definition of transition kernel $\gT_{\vz_s}$ shown in Eq.~\ref{eq:std_transker} and the second inequality follows from 
    \begin{equation*}
        \left\|\tilde{\rvz}_s\right\| - \left\|\vz_s\right\| \le \tau\left\|s_\theta(\vz_s)\right\| + \sqrt{2\tau}\left\|\xi\right\|.
    \end{equation*}
    According to the fact
    \begin{equation}
        \label{ineq:s_theta_upb}
        \begin{aligned}
            \tau\left\|s_\theta(\vz_s)\right\|\le & \tau \left\|\grad g(\vz_s)\right\| + \tau \epsilon_{\mathrm{score}} \le \tau\cdot\left(\left\|\grad g(\vz_s) - \grad g(\vzero)\right\| + \left\|\grad g(\vzero)\right\| +\epsilon_{\mathrm{score}}\right)\\
            \le & \tau\cdot \left(3L\cdot\left\|\vz_s\right\| + \left\|\grad g(\vzero)\right\| + \epsilon_{\mathrm{score}}\right) \le \sqrt{\tau d}
        \end{aligned}
    \end{equation}
    where the second inequality follows from the smoothness of $g$, and the last inequality follows from Eq.~\ref{ineq:tau_choice_0} and $\|\vz_s\|\le R$, 
    we have
    \begin{equation*}
        \begin{aligned}
            3\sqrt{\tau d} + 2\sqrt{\tau}t - \tau\|s_{\theta}(\vz_s)\| \ge \sqrt{2\tau d} + 2\sqrt{\tau} t,
        \end{aligned}
    \end{equation*}
    which implies the last inequality of Eq.~\ref{ineq:prob_upb_vardiff} for all $t\ge 0$.
    Furthermore, suppose $\|\vz_s\|\ge R/2$, it follows that
    \begin{equation*}
        \begin{aligned}
            \log(\|\vz_{s+1}\|^2) - \log (\|\vz_s\|^2) = 2\log(\|\vz_{s+1}\|/\|\vz_s\|)\le \|\vz_{s+1}\|/\|\vz_s\| - 1\le \frac{2\|\vz_{s+1}\|- 2\|\vz_s\|}{R}.
        \end{aligned}
    \end{equation*}
    Therefore, we have $\log(\|\vz_{s+1}\|^2) - \log (\|\vz_s\|^2)$ is also a sub-Gaussian random variable and satisfies
    \begin{equation}
        \label{ineq:log_var_subGau}
        \begin{aligned}
            \mathbb{P}\left[\log(\|\rvz_{s+1}\|^2) - \log (\|\vz_s\|^2) \ge 6R^{-1}\sqrt{\tau d} + 4R^{-1}t\sqrt{\tau}\right] \le \exp(-t^2).
        \end{aligned}
    \end{equation}
    We consider any subsequence among $\{\rvz_k\}_{k=1}^S$, with all iterates, except the first one, staying outside the region $\mathcal{B}(\vzero, R/2)$.
    Denote such subsequence by $\{\rvy_s\}_{s=0}^{S^\prime}$ where $\|\rvy_0\|\le R/2$ and $S^\prime\le S$.
    Then, we know $\rvy_s$ and $\rvy_{s+1}$ satisfy Eq.~\ref{ineq:super_mg_0} and Eq.~\ref{ineq:log_var_subGau} for all $s\ge 1$.
    Under these conditions, by requiring $\|\rvz_0\|\le R/2$ with a probability at least $1-\epsilon/16$, we only need to prove all points in $\{\rvy_s\}_{s=0}^{S^\prime}$ will stay inside the region $\mathcal{B}(\vzero,R)$ with probability at least $1-\epsilon/16$.

    Then, set $\mathcal{E}_s$ to be the event that
    \begin{equation*}
        \mathcal{E}_s = \left\{ \|\rvy_{s^\prime}\|\le R,\forall s^\prime\le s   \right\},
    \end{equation*}
    which satisfies $\mathcal{E}_{s-1}\subseteq \mathcal{E}_s$.
    Besides, suppose the filtration $\mathcal{F}_s = \{\rvy_0, \rvy_1, \ldots, \rvy_s\}$, the sequence 
    \begin{equation*}
        \begin{aligned}
            \left\{\vone(\mathcal{E}_{s-1})\cdot \left(\log (\|\rvy_s\|^2+ Ls \tau /4)\right)\right\}_{s=1,2,\ldots, S}
        \end{aligned}
    \end{equation*}
    is a super-martingale, and the martingale difference has a subgaussian tail, i.e., for any $t\ge 0$,
    \begin{equation*}
        \begin{aligned}
            &\mathbb{P}\left[\log(\|\rvy_{s+1}\|^2) + \frac{L(s+1)\tau}{4} - \log (\|\rvy_s\|^2) - \frac{Ls\tau}{4} \ge  7R^{-1}\sqrt{\tau d} + 4R^{-1}t\sqrt{\tau d} \right]\\
            & \le \mathbb{P}\left[\log(\|\rvy_{s+1}\|^2) + \frac{L(s+1)\tau}{4} - \log (\|\rvy_s\|^2) - \frac{Ls\tau}{4} \ge  6R^{-1}\sqrt{\tau d} + 4R^{-1}t\sqrt{\tau} + \frac{L\tau}{4} \right]\\
            & = \mathbb{P}\left[\log(\|\rvz_{s+1}\|^2) - \log (\|\vz_s\|^2) \ge 6R^{-1}\sqrt{\tau d} + 4R^{-1}t\sqrt{\tau}\right]\le \exp(-t^2),
        \end{aligned}
    \end{equation*}
    where the first inequality is established when 
    \begin{equation}
        \label{ineq:tau_choice_1}
        \frac{L\tau}{4}\le  \frac{\sqrt{\tau d}}{R} \quad\Leftrightarrow \quad \tau\le \frac{16 d}{L^2R^2}\quad \mathrm{and}\quad d\ge 1.
    \end{equation}
    Under these conditions, suppose 
    \begin{equation*}
        u = \frac{6\sqrt{\tau d}}{R} + \frac{4t\sqrt{\tau d}}{R}\quad\Leftrightarrow \quad t = \frac{uR}{4 \sqrt{\tau d}} - \frac{3}{2},
    \end{equation*}
    it implies
    \begin{equation*}
        t^2 \ge \frac{R^2u^2}{64 \tau d} - 1
    \end{equation*}
    which follows from the fact $(a-b)^2\ge a^2/4 - b^2/3$ for all $a,b\in \R$.
    Then, for any $u\ge 0$, we have
    \begin{equation*}
        \begin{aligned}
            &\mathbb{P}\left[\log(\|\rvy_{s+1}\|^2) + \frac{L(s+1)\tau}{4} - \log (\|\rvy_s\|^2) - \frac{Ls\tau}{4} \ge  u \right]\le \exp\left(-\frac{R^2u^2}{64\tau d} + 1\right)\le 3\exp\left(-\frac{R^2u^2}{64\tau d}\right),
        \end{aligned}
    \end{equation*}
    which implies that the martingale difference is subgaussian. Then by Theorem 2 in~\cite{shamir2011variant}, for any $s$, we have
    \begin{equation*}
        \log (\|\rvy_s\|^2) + \frac{Ls\tau}{4} \le \log (\|\rvy_0\|^2) + \frac{74}{R}\cdot \sqrt{s\tau d \log (1/\epsilon^\prime)}
    \end{equation*}
    with the probability at least $1-\epsilon^\prime$ conditioned on $\mathcal{E}_{s-1}$.
    Taking the union bound over all $s=1,2,\ldots, S^\prime\  (S^\prime\le S)$ and set $\epsilon = 16\epsilon^\prime S^\prime $, we have with probability at least $1-\epsilon/16$, for all $s=1,2,\ldots, S^\prime$, it holds 
    \begin{equation*}
        \begin{aligned}
            \log (\|\rvy_s\|^2) \le & 2\log (R/2) + \frac{74}{R}\cdot \sqrt{s\tau d \log (16S/\epsilon)} - \frac{Ls\tau}{4}\\
            \le & 2\log (R/2) + \frac{74^2\cdot d \log (16S/\epsilon)}{R^2 L}.
        \end{aligned}
    \end{equation*}
    By requiring
    \begin{equation}
        \label{ineq:choice_R_2}
        R\ge 63\cdot \sqrt{\frac{d}{L}\log\frac{16S}{\epsilon}}\quad \Rightarrow \quad \frac{74^2\cdot d \log (16S/\epsilon)}{R^2 L} \le 2\log 2,
    \end{equation}
    we have $\log (\|\rvy_s\|^2)\le \log (R^2)$, which is equivalent to $\|\rvy_s\|\le R$.
    Combining with the fact that with probability at least $1-\epsilon/16$ the initial point $\rvy_0$ stays inside $\mathcal{B}(\vzero, R/2)$, we can conclude that with probability at least $1-\epsilon/8$ all iterates stay inside the region $\mathcal{B}(\vzero, R)$.

    \paragraph{The difference between $\rvz_{s+1}$ and $\rvz_{s}$ is smaller than $r$.}
    In this paragraph, we aim to prove $\|\rvz_{s+1}-\rvz_{s}\|\le r$ for all $s\le S$.
    Similar to the previous techniques, we consider
    \begin{equation*}
        \tilde{\rvz}_s \coloneqq \rvz_s - \tau\cdot s_\theta(\rvz_s) + \sqrt{2\tau}\cdot \xi\quad\mathrm{where}\quad \xi\sim \mathcal{N}(\vzero, \mI).
    \end{equation*}
    According to the transition kernel Eq.~\ref{eq:std_transker}, it has
    \begin{equation}
        \label{ineq:byd_local_upb}
        \begin{aligned}
            \mathbb{P}\left[\|\rvz_{s+1}- \rvz_s\| \ge r\right] \le &  \mathbb{P}\left[\|\tilde{\rvz}_s -  \rvz_s\|  \ge r\right] \le \mathbb{P}\left[\tau\|s_\theta(\rvz_s)\| + \sqrt{2\tau}\|\xi\|\ge r\right]\\
            = & \mathbb{P}\left[\|\xi\|\ge \frac{r - \tau\|s_\theta(\rvz_s)\|}{\sqrt{2\tau}}\right] \le \mathbb{P}\left[\|\xi\|\ge \frac{r - \sqrt{\tau d}\|}{\sqrt{2\tau}}\right] 
        \end{aligned}
    \end{equation}
    where the second inequality follows from the triangle inequality, and the last inequality follows from Eq.~\ref{ineq:s_theta_upb} when the choice of $\tau$ satisfies Eq.~\ref{ineq:tau_choice_0}.
    Under these conditions, by choosing 
    \begin{equation*}
        r \ge (\sqrt{2}+1)\cdot \sqrt{\tau d} + 2\sqrt{\tau \log (8S/\epsilon)}\quad \Leftrightarrow \quad  \frac{r-\sqrt{\tau d}}{\sqrt{2\tau}} \ge \sqrt{d}+ \sqrt{2}\cdot \sqrt{\log(8S/\epsilon)},
    \end{equation*}
    Eq.~\ref{ineq:byd_local_upb} becomes
    \begin{equation*}
        \begin{aligned}
            \mathbb{P}\left[\|\rvz_{s+1}- \rvz_s\| \ge r\right] & \le \mathbb{P}\left[\|\xi\|\ge \frac{r - \sqrt{\tau d}\|}{\sqrt{2\tau}}\right] \le \mathbb{P}\left[\|\xi\|\ge \sqrt{d}+ \sqrt{2}\cdot \sqrt{\log(8S/\epsilon)}\right]\le \frac{\epsilon}{8S},
        \end{aligned}
    \end{equation*}
    which means 
    \begin{equation*}
        \mathbb{P}\left[\|\rvz_{s+1} - \rvz_s\|\le r\right]\ge 1-\frac{\epsilon}{8S}.
    \end{equation*}
    Taking union bound over all iterates, we know all particles satisfy the local condition, i.e., $\|\rvz_{s+1}-\rvz_s\|\le r$ with the probability at least $1-\epsilon/8$.
    Hence, the proof is completed.
\end{proof}

\subsection{Control the error from the approximation of score and energy}

\begin{lemma}
    \label{lem:lem62_zou2021faster_gene}
    Under Assumption~\ref{ass:lips_score}--\ref{ass:second_moment}, we set 
    \begin{equation*}
    \eta= \frac{1}{2}\log \frac{2L+1}{2L}\quad \mathrm{and}\quad G\coloneqq \left\|\grad g(\vzero)\right\|.
    \end{equation*}
    For any $\epsilon\in(0,1)$, we set
    \begin{equation*}
        R\ge \max\left\{8\cdot \sqrt{\frac{\|\grad g(\vzero)\|^2}{L^2}+\frac{d}{L}}, 63\cdot \sqrt{\frac{d}{L}\log\frac{16S}{\epsilon}}\right\},\quad r=3\cdot \sqrt{\tau d\log\frac{8S}{\epsilon}}.
    \end{equation*}
    Suppose it has
    \begin{equation*}
        \frac{\delta}{16}\coloneqq \frac{3\epsilon_{\mathrm{score}}}{2}\cdot \sqrt{\tau d\log\frac{8S}{\epsilon}} + \frac{\tau\epsilon^2_{\mathrm{score}}}{4}+\frac{\tau(3LR+G)\epsilon_{\mathrm{score}}}{2}\le \frac{1}{32}\quad \mathrm{and}\quad \epsilon_{\mathrm{energy}}\le \frac{1}{10},
    \end{equation*}
    we have
    \begin{equation*}
        \left(1-\delta-5\epsilon_{\mathrm{energy}}\right)\cdot \tilde{\gT}_{*, \vz}(\Omega^\prime_{\vz})\le \tilde{\gT}_{\vz}(\Omega^\prime_{\vz})\le \left(1+\delta+5\epsilon_{\mathrm{energy}}\right)\cdot \tilde{\gT}_{*,\vz}(\Omega^\prime_{\vz}).
    \end{equation*}
    for any set $\gA\subseteq \mathcal{B}(0,R)$ and point $\vz\in \mathcal{B}(0,R)$.
\end{lemma}
\begin{proof}
    Note that the Markov process defined by $\tilde{\gT}_{\vz}(\cdot)$ and $\tilde{\gT}_{*,\vz}(\cdot)$ are $1/2$-lazy.
    We prove the lemma by considering two cases: $\vz\not\in\gA$ and $\vz\in\gA$.
    \paragraph{When $\vz\not\in \gA$,} we have
    \begin{equation*}
        \begin{aligned}
            \tilde{\gT}_{\vz}(\gA) & = \int_{\gA} \tilde{a}_{\vz}(\vz^\prime) \tilde{\gT}^\prime_{\vz}(\der\vz^\prime) = \frac{1}{2}\int_{\gA} \tilde{a}_{\vz}(\vz^\prime) \tilde{Q}^\prime_{\vz}(\der \vz^\prime)\\
            & = \frac{1}{2}\int_{\gA\cap \Omega_{\vz}} \tilde{a}_{\vz}(\vz^\prime) \tilde{Q}_{\vz}(\der \vz^\prime) = \frac{1}{2}\int_{\gA\cap \Omega_{\vz}} \tilde{a}_{\vz}(\vz^\prime) \tilde{q}(\vz^\prime|\vz)\der \vz^\prime.
        \end{aligned}
    \end{equation*}
    Similarly, we have
    \begin{equation*}
        \tilde{\gT}_{*, \vz}(\gA) = \frac{1}{2}\int_{\gA\cap \Omega_{\vz}} \tilde{a}_{*, \vz}(\vz^\prime) \tilde{q}_*(\vz^\prime|\vz)\der \vz^\prime.
    \end{equation*}
    In this condition, we consider
    \begin{equation*}
        \begin{aligned}
            2\tilde{\gT}_{\vz}(\gA)- 2\tilde{\gT}_{*,\vz}(\gA)  = &  - \int_{\gA\cap \Omega_{\vz}} \tilde{a}_{*,\vz}(\vz^\prime)\tilde{q}_*(\vz^\prime|\vz)\der \vz^\prime + \int_{\gA\cap \Omega_{\vz}} \tilde{a}_{*,\vz}(\vz^\prime)\tilde{q}(\vz^\prime|\vz)\der \vz^\prime\\
            & - \int_{\gA\cap \Omega_{\vz}} \tilde{a}_{*,\vz}(\vz^\prime)\tilde{q}(\vz^\prime|\vz)\der \vz^\prime + \int_{\gA\cap \Omega_{\vz}} \tilde{a}_{\vz}(\vz^\prime)\tilde{q}(\vz^\prime|\vz)\der\vz^\prime,
        \end{aligned}
    \end{equation*}
    which means
    \begin{equation}
        \label{eq:case_1_tbc}
        \begin{aligned}
            \frac{\tilde{\gT}_{\vz}(\gA) - \tilde{\gT}_{*,\vz}(\gA)}{\tilde{\gT}_{*,\vz}(\gA)} = & \underbrace{\frac{\int_{\gA\cap \Omega_{\vz}} \tilde{a}_{*,\vz}(\vz^\prime)\cdot\left( \tilde{q}(\vz^\prime|\vz)-\tilde{q}_{*}(\vz^\prime|\vz)\right)\der\vz^\prime}{\int_{\gA\cap \Omega_{\vz}} \tilde{a}_{*,\vz}(\vz^\prime)\tilde{q}_*(\vz^\prime|\vz)\der\vz^\prime}}_{\mathrm{Term\ 1}}\\
            & + \underbrace{\frac{\int_{\gA\cap \Omega_{\vz}} \left( \tilde{a}_{\vz}(\vz^\prime)-\tilde{a}_{*,\vz}(\vz^\prime)\right) \tilde{q}(\vz^\prime|\vz) \der\vz^\prime}{\int_{\gA\cap \Omega_{\vz}} \tilde{a}_{*,\vz}(\vz^\prime)\tilde{q}_{*}(\vz^\prime|\vz)\der\vz^\prime}}_{\mathrm{Term\ 2}}.
        \end{aligned}
    \end{equation}
    First, we try to control Term 1, which can be achieved by investigating $\tilde{q}(\vz^\prime|\vz)/{\tilde{q}_*(\vz^\prime|\vz)}$ as follows.
    \begin{equation*}
        \begin{aligned}
            \frac{\tilde{q}(\vz^\prime|\vz)}{\tilde{q}_*(\vz^\prime|\vz)}=  \exp\left(-\frac{\left\|\vz^\prime - (\vz - \tau\cdot s_\theta(\vz))\right\|^2}{4\tau} + \frac{\left\|\vz^\prime - (\vz - \tau\cdot \grad g(\vz))\right\|^2}{4\tau}\right),
        \end{aligned}
    \end{equation*}
    In this condition, we have
    \begin{equation}
        \label{eq:ula_trans_ker_exp}
        \begin{aligned}
            \frac{\tilde{q}(\vz^\prime|\vz)}{\tilde{q}_*(\vz^\prime|\vz)} = &\exp\left( (4\tau)^{-1}\cdot \left( - \left\|\vz^\prime - \vz\right\|^2  - 2\tau\cdot \left(\vz^\prime - \vz\right)^\top s_{\theta}(\vz) -  \tau^2\cdot \left\|s_{\theta}(\vz)\right\|^2\right.\right.\\
            &\quad\left.\left. +\left\|\vz^\prime - \vz\right\|^2 + 2\tau\cdot \left(\vz^\prime - \vz\right)^\top \grad g(\vz) + \tau^2\cdot\left\|\grad g(\vz)\right\|^2\right)\right)\\
            = & \exp\left(\frac{1}{2}(\vz^\prime - \vz)^\top \left( - s_\theta(\vz) + \grad g(\vz)\right) + \frac{\tau}{4}\left(- \left\|s_\theta(\vz)\right\|^2 + \left\|\grad g(\vz)\right\|^2\right)\right).
        \end{aligned}
    \end{equation}
    It means
    \begin{equation*}
        \begin{aligned}
            \left|\ln \frac{\tilde{q}(\vz^\prime|\vz)}{\tilde{q}_*(\vz^\prime|\vz)}\right| = & \left|\frac{1}{2}(\vz^\prime - \vz)^\top \left( - s_\theta(\vz) + \grad g(\vz)\right) + \frac{\tau}{4}\left( - \left\|s_\theta(\vz)\right\|^2 + \left\|\grad g(\vz)\right\|^2\right)\right|\\
            \le & \frac{1}{2}\left\|\vz^\prime - \vz\right\| \cdot \left\|s_\theta(\vz) - \grad g(\vz)\right\| + \frac{\tau}{4}\cdot \left[\left\|s_\theta(\vz) + \grad g(\vz)\right\| \cdot \left\|s_\theta(\vz) - \grad g(\vz)\right\|\right]\\
            \le &  \frac{1}{2}\left\|\vz^\prime - \vz\right\| \cdot \left\|s_\theta(\vz) - \grad g(\vz)\right\| + \frac{\tau}{4}\cdot \left\|s_\theta(\vz) - \grad g(\vz)\right\|^2 + \frac{\tau}{2}\cdot\left\|\grad g(\vz)\right\|\cdot \left\|s_\theta(\vz) - \grad g(\vz)\right\|\\
            \le & \frac{r\epsilon_{\mathrm{score}}}{2} + \frac{\tau\epsilon^2_{\mathrm{scroe}}}{4} + \frac{\tau(3LR+G)\epsilon_{\mathrm{score}}}{2}
        \end{aligned}
    \end{equation*}
    where the last inequality follows from the fact $\vz^\prime\in  \mathcal{B}(\vz,r)
    \cap \mathcal{B}(\vzero,R)/\{\vz\}$, $\vz\in \gB(\vzero,R)$ and 
    \begin{equation*}
        \left\|\grad g(\vz)\right\|=\left\|\grad g(\vz) - \grad g(\vzero) + \grad g(\vzero)\right\|\le 3L\cdot\left\|\vz\right\| + G\le 3LR+G.
    \end{equation*}
    According to the definition of $R$ and r shown in Lemma~\ref{lem:proj_gap_lem}, we choose
    \begin{equation*}
        \begin{aligned}
            r\coloneqq 3\cdot \sqrt{\tau} \cdot \sqrt{d\log\frac{8S}{\epsilon}} \ge (\sqrt{2}+1)\cdot \sqrt{\tau d} + 2\sqrt{\tau \log \frac{8S}{\epsilon}}  
        \end{aligned}
    \end{equation*}
    Under this condition, we require
    \begin{equation}
        \label{ineq:delta_choice_2}
        \frac{\delta}{16}\coloneqq \frac{3\epsilon_{\mathrm{score}}}{2}\cdot \sqrt{\tau d\log\frac{8S}{\epsilon}} + \frac{\tau\epsilon^2_{\mathrm{score}}}{4}+\frac{\tau(3LR+G)\epsilon_{\mathrm{score}}}{2}\le \frac{1}{32},
    \end{equation}
    then we have
    \begin{equation*}
        \ln\left(1-\frac{\delta}{8}\right)\le \ln \frac{\tilde{q}(\vz^\prime|\vz)}{\tilde{q}_*(\vz^\prime|\vz)}\le \ln\left(1+\frac{\delta}{8}\right) \quad\Leftrightarrow\quad 1-\frac{\delta}{8}<\frac{\tilde{q}(\vz^\prime|\vz)}{\tilde{q}_*(\vz^\prime|\vz)}\le 1+\frac{\delta}{8},
    \end{equation*}
    and 
    \begin{equation}
        \label{ineq:case_1_term_1}
        -\frac{\delta}{8}\le \min_{\vz^\prime\in\gA\cap \Omega_{\vz}}\frac{\tilde{q}(\vz^\prime|\vz)}{\tilde{q}_*(\vz^\prime|\vz)} - 1 \le  \mathrm{Term\ 1} \le  \max_{\vz^\prime\in\gA\cap\Omega_{\vz}}\frac{\tilde{q}(\vz^\prime|\vz)}{\tilde{q}_*(\vz^\prime|\vz)} - 1 \le \frac{\delta}{8}. 
    \end{equation}
    with the definition of $\mathrm{Term\ 1}$ shown in Eq.~\ref{eq:case_1_tbc}.
    
    Then, we try to control $\mathrm{Term\ 2}$ of Eq.~\ref{eq:case_1_tbc} and have
    \begin{equation}
        \label{eq:case_1_term21_refor}
        \frac{\int_{\gA\cap \Omega_{\vz}} \left( \tilde{a}_{\vz}(\vz^\prime)-\tilde{a}_{*,\vz}(\vz^\prime)\right) \tilde{q}(\vz^\prime|\vz) \der\vz^\prime}{\int_{\gA\cap \Omega_{\vz}} \tilde{a}_{*,\vz}(\vz^\prime)\tilde{q}_{*}(\vz^\prime|\vz)\der\vz^\prime} = \frac{\int_{\gA\cap \Omega_{\vz}} \left( \tilde{a}_{\vz}(\vz^\prime)- \tilde{a}_{*,\vz}(\vz^\prime)\right) \tilde{q}(\vz^\prime|\vz) \der\vz^\prime}{\int_{\gA\cap \Omega_{\vz}} \tilde{a}_{*,\vz}(\vz^\prime)\tilde{q}(\vz^\prime|\vz)\der\vz^\prime}\cdot \frac{\int_{\gA\cap \Omega_{\vz}} \tilde{a}_{*,\vz}(\vz^\prime)\tilde{q}(\vz^\prime|\vz)\der\vz^\prime}{\int_{\gA\cap \Omega_{\vz}} \tilde{a}_{*,\vz}(\vz^\prime)\tilde{q}_{*}(\vz^\prime|\vz)\der\vz^\prime}.
    \end{equation}
    According to Eq.~\ref{ineq:case_1_term_1}, it has 
    \begin{equation}
        \label{ineq:case_1_final_t_bound}
        1-\frac{\delta}{8}\le \min_{\vz^\prime\in\gA\cap \Omega_{\vz}}\frac{\tilde{q}(\vz^\prime|\vz)}{\tilde{q}_*(\vz^\prime|\vz)} \le \frac{\int_{\gA\cap \Omega_{\vz}} \tilde{a}_{*,\vz}(\vz^\prime)\tilde{q}(\vz^\prime|\vz)\der\vz^\prime}{\int_{\gA\cap \Omega_{\vz}} \tilde{a}_{*,\vz}(\vz^\prime)\tilde{q}_{*}(\vz^\prime|\vz)\der\vz^\prime} \le \max_{\vz^\prime\in\gA\cap \Omega_{\vz}}\frac{\tilde{q}(\vz^\prime|\vz)}{\tilde{q}_*(\vz^\prime|\vz)}\le 1+\frac{\delta}{8},
    \end{equation}
    then we can upper and lower bounding $\mathrm{Term\ 2}$ by investigating $\tilde{a}_{\vz}(\vz^\prime)/\tilde{a}_{*,\vz}(\vz^\prime)$ as follows
    \begin{equation*}
        \begin{aligned}
            \tilde{a}_{*,\vz}(\vz^\prime)=&\min\left\{1, \exp\left(-(g(\vz^\prime)-g(\vz))\right)\cdot \frac{\tilde{q}_*(\vz|\vz^\prime)}{\tilde{q}_*(\vz^\prime|\vz)} \right\},\\
            \tilde{a}_{\vz}(\vz^\prime)=&\min\left\{1, \exp\left(-r_\theta(\vz^\prime,\vz)\right)\cdot \frac{\tilde{q}(\vz|\vz^\prime)}{\tilde{q}(\vz^\prime|\vz)}  \right\}.
        \end{aligned}
    \end{equation*}
    In this condition, for any $0<\delta\le 1$, we first consider two cases.
    When 
    \begin{equation*}
        \tilde{a}_{*,\vz}(\vz^\prime)= 1\le \exp\left(-(g(\vz^\prime)-g(\vz))\right)\cdot \frac{\tilde{q}_*(\vz|\vz^\prime)}{\tilde{q}_*(\vz^\prime|\vz)}  \quad \mathrm{and}\quad \tilde{a}_{\vz}(\vz^\prime)= \exp\left(-r_\theta(\vz^\prime,\vz)\right)\cdot \frac{\tilde{q}(\vz|\vz^\prime)}{\tilde{q}(\vz^\prime|\vz)}\le 1,
    \end{equation*}
    we have 
    \begin{equation}
        \label{ineq:case_1_a_div_upb}
        \underbrace{\frac{\exp\left(-r_\theta(\vz^\prime,\vz)\right)\cdot \frac{\tilde{q}(\vz|\vz^\prime)}{\tilde{q}(\vz^\prime|\vz)}}{\exp\left(-(g(\vz^\prime)-g(\vz))\right)\cdot \frac{\tilde{q}_*(\vz|\vz^\prime)}{\tilde{q}_*(\vz^\prime|\vz)}}}_{\mathrm{Term\ 2.1}} \le \frac{\tilde{a}_{\vz}(\vz^\prime)}{\tilde{a}_{*,\vz}(\vz^\prime)} = \exp\left(-r_\theta(\vz^\prime,\vz)\right)\cdot \frac{\tilde{q}(\vz|\vz^\prime)}{\tilde{q}(\vz^\prime|\vz)} \le 1.
    \end{equation}
    Besides, when 
    \begin{equation*}
        \tilde{a}_{*,\vz}(\vz^\prime)= \exp\left(-(g(\vz^\prime)-g(\vz))\right)\cdot \frac{\tilde{q}_*(\vz|\vz^\prime)}{\tilde{q}_*(\vz^\prime|\vz)} \le 1  \quad \mathrm{and}\quad \tilde{a}_{\vz}(\vz^\prime)=1\le \exp\left(-r_\theta(\vz^\prime,\vz)\right)\cdot \frac{\tilde{q}(\vz|\vz^\prime)}{\tilde{q}(\vz^\prime|\vz)},
    \end{equation*}
    we have
    \begin{equation}
        \label{ineq:case_1_a_div_lpb}
        1\le \frac{\tilde{a}_{\vz}(\vz^\prime)}{\tilde{a}_{*,\vz}(\vz^\prime)} =  \frac{1}{\exp\left(-(g(\vz^\prime)-g(\vz))\right)\cdot \frac{\tilde{q}_*(\vz|\vz^\prime)}{\tilde{q}_*(\vz^\prime|\vz)}}\le \underbrace{\frac{\exp\left(-r_\theta(\vz^\prime,\vz)\right)\cdot \frac{\tilde{q}(\vz|\vz^\prime)}{\tilde{q}(\vz^\prime|\vz)}}{\exp\left(-(g(\vz^\prime)-g(\vz))\right)\cdot \frac{\tilde{q}_*(\vz|\vz^\prime)}{\tilde{q}_*(\vz^\prime|\vz)}}}_{\mathrm{Term\ 2.1}}.
    \end{equation}
    Then, we start to consider finding the range of $\ln(\mathrm{Term\ 2.1})$ as follows
    \begin{equation}
        \label{ineq:case_1_upb_term2.1}
        \begin{aligned}
            \left|\ln\left(\mathrm{Term\ 2.1}\right)\right| = & \left|\left(- r_{\theta}(\vz^\prime,\vz) + (g(\vz^\prime)-g(\vz))\right) + \ln \frac{\tilde{q}_{*}(\vz^\prime|\vz)}{\tilde{q}(\vz^\prime|\vz)} + \ln \frac{\tilde{q}(\vz|\vz^\prime)}{\tilde{q}_{*}(\vz|\vz^\prime)}\right|\\
            \le & \epsilon_{\mathrm{energy}} + \left|\ln \frac{\tilde{q}_{*}(\vz^\prime|\vz)}{\tilde{q}(\vz^\prime|\vz)}\right| + \left| \ln \frac{\tilde{q}(\vz|\vz^\prime)}{\tilde{q}_{*}(\vz|\vz^\prime)}\right|\le \epsilon_{\mathrm{energy}}+\frac{\delta}{16} +\left|\ln \frac{\tilde{q}(\vz|\vz^\prime)}{\tilde{q}_{*}(\vz|\vz^\prime)}\right|,
        \end{aligned}
    \end{equation}
    where the last inequality follows from Eq.~\ref{ineq:delta_choice_2}.
    Besides, similar to Eq.~\ref{eq:ula_trans_ker_exp}, we have
    \begin{equation*}
        \begin{aligned}
            \frac{\tilde{q}(\vz|\vz^\prime)}{\tilde{q}_{*}(\vz|\vz^\prime)} = & \exp\left((4\tau)^{-1}\cdot \left(-\left\|\vz-\vz^\prime\right\|^2 - 2\tau\cdot (\vz-\vz^\prime)^\top s_\theta(\vz^\prime) - \tau^2\cdot \left\|s_\theta(\vz^\prime)\right\|^2\right.\right.\\
            &\quad \left.\left. +\left\|\vz-\vz^\prime\right\|^2 + 2\tau\cdot (\vz-\vz^\prime)^\top\grad g(\vz^\prime) + \tau^2\cdot \left\|\grad g(\vz^\prime)\right\|^2 \right)\right),
        \end{aligned}
    \end{equation*}
    which means
    \begin{equation*}
        \begin{aligned}
            \left|\ln \frac{\tilde{q}(\vz|\vz^\prime)}{\tilde{q}_{*}(\vz|\vz^\prime)}\right| = & \left|\frac{1}{2}(\vz-\vz^\prime)^\top( - s_{\theta}(\vz^\prime) + \grad g(\vz^\prime)) + \frac{\tau}{4}\left(-\left\|s_\theta(\vz^\prime)\right\|^2 + \left\|\grad g(\vz^\prime)\right\|^2\right)\right|\\
            \le &  \frac{1}{2}\left\|\vz-\vz^\prime\right\|\cdot \left\|s_\theta(\vz^\prime) - \grad g(\vz^\prime)\right\|+\frac{\tau}{4}\cdot \left\|s_\theta(\vz^\prime)+\grad g(\vz^\prime)\right\|\cdot \left\|s_\theta(\vz^\prime) - \grad g(\vz^\prime)\right\|\\
            \le & \frac{1}{2}\left\|\vz-\vz^\prime\right\|\cdot \left\|s_\theta(\vz^\prime) - \grad g(\vz^\prime)\right\|+ \frac{\tau}{4}\cdot\left\|s_\theta(\vz^\prime)-\grad g(\vz^\prime)\right\|^2 + \frac{\tau}{2}\left\|\grad g(\vz^\prime)\right\|\cdot \left\|s_\theta(\vz^\prime) - \grad g(\vz^\prime)\right\|\\
            \le & \frac{r\epsilon_{\mathrm{score}}}{2} + \frac{\tau\epsilon^2_{\mathrm{score}}}{4}+ \frac{\tau(3LR+G)\epsilon_{\mathrm{score}}}{2},
        \end{aligned}
    \end{equation*}
    where the last inequality follows from the fact $\vz^\prime\in  \mathcal{B}(\vz,r)
    \cap \mathcal{B}(\vzero,R)/\{\vz\}$ and 
    \begin{equation*}
        \left\|\grad g(\vz^\prime)\right\|=\left\|\grad g(\vz^\prime) - \grad g(\vzero) + \grad g(\vzero)\right\|\le 3L\cdot\left\|\vz^\prime\right\| + G\le 3LR+G.
    \end{equation*}
    Combining this result with Eq.~\ref{ineq:delta_choice_2}, we have
    \begin{equation*}
        \begin{aligned}
              \left|\ln \frac{\tilde{q}(\vz|\vz^\prime)}{\tilde{q}_{*}(\vz|\vz^\prime)}\right| \le \frac{\delta}{16}\quad\Leftrightarrow\quad \frac{\delta}{16}= \frac{3\epsilon_{\mathrm{score}}}{2}\cdot \sqrt{\tau d\log\frac{8S}{\epsilon}} + \frac{\tau\epsilon^2_{\mathrm{score}}}{4}+\frac{\tau(3LR+G)\epsilon_{\mathrm{score}}}{2}.
        \end{aligned}
    \end{equation*}
    Plugging this result into Eq.~\ref{ineq:case_1_upb_term2.1}, it has
    \begin{equation*}
        \left|\ln\left(\mathrm{Term\ 2.1}\right)\right|\le \frac{\delta}{8}+\epsilon_{\mathrm{energy}}.
    \end{equation*}
    By requiring $\epsilon_{\mathrm{energy}}\le 0.1$, we have
    \begin{equation*}
        \begin{aligned}
            & \ln\left(1-\frac{\delta}{4}-2\epsilon_{\mathrm{energy}}\right)\le \ln\left(\mathrm{Term\ 2.1}\right)\le \ln\left(1+\frac{\delta}{4}+2\epsilon_{\mathrm{energy}}\right)\\
            & \Leftrightarrow\quad 1-\frac{\delta}{4}-2\epsilon_{\mathrm{energy}}\le \mathrm{Term\ 2.1}\le 1+\frac{\delta}{4}+2\epsilon_{\mathrm{energy}}.
        \end{aligned}
    \end{equation*}
    Combining this result with Eq.~\ref{ineq:case_1_a_div_upb} and Eq.~\ref{ineq:case_1_a_div_lpb}, we have
    \begin{equation*}
       1-\frac{\delta}{4}-2\epsilon_{\mathrm{energy}}\le \frac{a_{\vz}(\vz^\prime)}{a_{*,\vz}(\vz^\prime)}\le 1+\frac{\delta}{4}+2\epsilon_{\mathrm{energy}},
    \end{equation*}
    which implies
    \begin{equation}
        \label{ineq:case_1_final_a_bound}
        \begin{aligned}
            & \frac{\int_{\gA\cap \Omega_{\vz}} \left( \tilde{a}_{\vz}(\vz^\prime)-\tilde{a}_{*,\vz}(\vz^\prime)\right) \tilde{q}(\vz^\prime|\vz) \der\vz^\prime}{\int_{\gA\cap \Omega_{\vz}} \tilde{a}_{*,\vz}(\vz^\prime)\tilde{q}(\vz^\prime|\vz)\der\vz^\prime}\ge \min_{z^\prime\in \gA}\frac{\tilde{a}_{\vz}(\vz^\prime)}{\tilde{a}_{*,\vz}(\vz^\prime)}-1 \ge -\frac{\delta}{4}-2\epsilon_{\mathrm{energy}}\\
            & \frac{\int_{\gA\cap \Omega_{\vz}} \left( \tilde{a}_{\vz}(\vz^\prime)-\tilde{a}_{*,\vz}(\vz^\prime)\right) \tilde{q}(\vz^\prime|\vz) \der\vz^\prime}{\int_{\gA\cap \Omega_{\vz}} \tilde{a}_{*,\vz}(\vz^\prime)\tilde{q}(\vz^\prime|\vz)\der\vz^\prime} \le \max_{z^\prime\in \gA}\frac{\tilde{a}_{\vz}(\vz^\prime)}{\tilde{a}_{*,\vz}(\vz^\prime)}-1 \le \frac{\delta}{4}+2\epsilon_{\mathrm{energy}}.
        \end{aligned}
    \end{equation}
    Plugging Eq.~\ref{ineq:case_1_final_a_bound} and Eq.~\ref{ineq:case_1_final_t_bound} into Eq.~\ref{eq:case_1_term21_refor}, we have
    \begin{equation}
        \label{ineq:case_1_term_2}
        -\frac{\delta}{3}-\frac{5\epsilon_{\mathrm{energy}}}{2} \le \left(-\frac{\delta}{4}-2\epsilon_{\mathrm{energy}}\right)\cdot \left(1+\frac{\delta}{8}\right)\le \mathrm{Term\ 2}\le \left(\frac{\delta}{4}+2\epsilon_{\mathrm{energy}}\right)\cdot \left(1+\frac{\delta}{8}\right) \le \frac{\delta}{3}+\frac{5\epsilon_{\mathrm{energy}}}{2}.
    \end{equation}
    In this condition, combining Eq.~\ref{ineq:case_1_term_2}, Eq.~\ref{ineq:case_1_term_1} with Eq.~\ref{eq:case_1_tbc}, we have
    \begin{equation}
        \label{ineq:case_1_ulb}
        \begin{aligned}
            & -\frac{\delta+5\epsilon_{\mathrm{energy}}}{2} \le \frac{\tilde{\gT}_{\vz}(\gA) - \tilde{\gT}_{*,\vz}(\gA)}{\tilde{\gT}_{*,\vz}(\gA)} \le \frac{\delta+5\epsilon_{\mathrm{energy}}}{2}\\ & \Leftrightarrow\quad \left(1-\frac{\delta+5\epsilon_{\mathrm{energy}}}{2}\right)\cdot \tilde{\gT}_{*, \vz}(\gA)\le \tilde{\gT}_{\vz}(\gA)\le \left(1+\frac{\delta+5\epsilon_{\mathrm{energy}}}{2}\right)\cdot \tilde{\gT}_{*,\vz}(\gA).
        \end{aligned}
    \end{equation}
    Hence, we complete the proof for $\vz\not\in \gA$.
    \paragraph{When $\vz\in \gA$,} suppose there exist some $r^\prime$ satisfying
    \begin{equation*}
        \Omega_{\vz}^\prime \coloneqq \mathcal{B}(\vz, r^\prime)\subseteq \gA.
    \end{equation*}
    We can split $\gA$ into $\gA-\Omega^\prime_{\vz}$ and $\Omega^\prime_{\vz}$ .
    Note that by our results in the first case, we have
    \begin{equation*}
        \left(1-\frac{\delta+5\epsilon_{\mathrm{energy}}}{2}\right)\cdot \tilde{\gT}_{*, \vz}(\gA-\Omega^\prime_{\vz})\le \tilde{\gT}_{\vz}(\gA-\Omega^\prime_{\vz})\le \left(1+\frac{\delta+5\epsilon_{\mathrm{energy}}}{2}\right)\cdot \tilde{\gT}_{*,\vz}(\gA-\Omega^\prime_{\vz}).
    \end{equation*}
    Then for the set $\Omega^\prime_{\vz}$, we have
    \begin{equation*}
        \begin{aligned}
            & \left|\frac{\tilde{\gT}_{\vz}(\Omega^\prime_{\vz})- \tilde{\gT}_{*, \vz}(\Omega^\prime_{\vz})}{\tilde{\gT}_{*, \vz}(\Omega^\prime_{\vz})}\right| = \left|\frac{\left(1 - \tilde{\gT}_{\vz}(\Omega - \Omega^\prime_{\vz})\right) - \left(1- \tilde{\gT}_{*,\vz}(\Omega - \Omega^\prime_{\vz})\right)}{\tilde{\gT}_{*, \vz}(\Omega^\prime_{\vz})}\right|\\
            & = \left|\frac{\tilde{\gT}_{*,\vz}(\Omega - \Omega^\prime_{\vz}) - \tilde{\gT}_{\vz}(\Omega - \Omega^\prime_{\vz})}{\tilde{\gT}_{*, \vz}(\Omega^\prime_{\vz})}\right|
            \le \left|\frac{\tilde{\gT}_{*,\vz}(\Omega - \Omega^\prime_{\vz}) - \tilde{\gT}_{\vz}(\Omega - \Omega^\prime_{\vz})}{\tilde{\gT}_{*, \vz}(\Omega - \Omega^\prime_{\vz})}\right| \cdot \left|\frac{\tilde{\gT}_{*, \vz}(\Omega - \Omega^\prime_{\vz})}{\tilde{\gT}_{*, \vz}(\Omega^\prime_{\vz})}\right|\\
            & \le \frac{\delta+5\epsilon_{\mathrm{energy}}}{2}\cdot 2 = \delta+5\epsilon_{\mathrm{energy}},
        \end{aligned}
    \end{equation*}
    where the last inequality follows from Eq.~\ref{ineq:case_1_ulb} and the property of $1/2$ lazy, i.e.,
    \begin{equation*}
        \tilde{\gT}_{*, \vz}(\Omega - \Omega^\prime_{\vz}) \le 1\quad \mathrm{and}\quad \tilde{\gT}_{*, \vz}(\Omega^\prime_{\vz})\ge \frac{1}{2}.
    \end{equation*}
    In this condition, we have
    \begin{equation*}
        \left(1-\delta-5\epsilon_{\mathrm{energy}}\right)\cdot \tilde{\gT}_{*, \vz}(\Omega^\prime_{\vz})\le \tilde{\gT}_{\vz}(\Omega^\prime_{\vz})\le \left(1+\delta+5\epsilon_{\mathrm{energy}}\right)\cdot \tilde{\gT}_{*,\vz}(\Omega^\prime_{\vz}).
    \end{equation*}
    Hence, we complete the proof for $\vz\in \gA$.
\end{proof}
\begin{corollary}
    \label{lem:lem62_zou2021faster}
    Under the same conditions as shown in Lemma~\ref{lem:lem62_zou2021faster_gene}, if we require
    \begin{equation*}
        \epsilon_{\mathrm{energy}}\le \delta/5,
    \end{equation*}
    then we have 
    \begin{equation*}
    \left(1-2\delta\right)\cdot \tilde{\gT}_{*, \vz}(\gA)\le \tilde{\gT}_{\vz}(\gA)\le \left(1+2\delta\right)\cdot \tilde{\gT}_{*,\vz}(\gA),
    \end{equation*}
    for any set $\gA\subseteq \mathcal{B}(0,R)$ and point $\vz\in \mathcal{B}(0,R)$.
\end{corollary}

\subsection{Control the error from Inner MALA to its stationary}

In this section, we denote the ideally projected implementation of Alg.~\ref{alg:inner_mala} whose Markov process, transition kernel, and particles' underlying distributions are denoted as $\{\tilde{\rvz}_{*,s}\}_{s=0}^S$, Eq.~\ref{def:ideal_proj_rtk}, and $\tilde{\mu}_{*,s}$ respectively.
According to~\cite{zou2021faster}, we know the stationary distribution of the time-reversible process $\{\tilde{\rvz}_{*,s}\}_{s=0}^S$ is 
\begin{equation}
    \label{def:sta_tilde_mu_*}
    \tilde{\mu}_* (\der\vz) = \left\{
        \begin{aligned}
            & \frac{e^{-g(\vz)}}{\int_\Omega e^{-g(\vz^\prime)}\der\vz^\prime}\der\vz && \quad \vx\in\Omega;\\
            & 0 &&\quad \text{otherwise}.
        \end{aligned}
    \right.
\end{equation}
Here, we denote $\Omega = \mathcal{B}(\vzero,R)$ and 
\begin{equation*}
    \Omega_{\vz} = \mathcal{B}(\vzero, R)\cap \mathcal{B}(\vz, r).
\end{equation*}
In the following analysis, we default 
\begin{equation*}
    \eta= \frac{1}{2}\log \frac{2L+1}{2L}.
\end{equation*}
Under this condition, the smoothness of $g$ is $3L$ and the strong convexity constant is $L$.

we aim to build the connection between the underlying distribution of the output particles obtained by projected Alg~\ref{alg:inner_mala}, i.e., $\tilde{\mu}_S$, and the stationary distribution $\tilde{\mu}_*$ though the process $\{\tilde{\rvz}_{*,s}\}_{s=0}^S$.
Since the ideally projected implementation of Alg.~\ref{alg:inner_mala} is similar to standard MALA except for the projection, we prove its convergence through its conductance properties, which can be deduced by the Cheeger isoperimetric inequality of $\tilde{\mu}_*$.

Under these conditions, we organize this subsection in the following three steps:
\begin{enumerate}
    \item Find the Cheeger isoperimetric inequality of $\tilde{\mu}_*$.
    \item Find the conductance properties of $\tilde{\gT}_*$.
    \item Build the connection between $\tilde{\mu}_S$ and $\tilde{\mu}_*$ through the process $\{\tilde{\rvz}_{*,s}\}_{s=0}^S$.
\end{enumerate}

\subsubsection{The Cheeger isoperimetric inequality of $\tilde{\mu}_*$}

\begin{definition}[Definition 2.5.9 in~\cite{sinho2024logconcave}]
    A probability measure $\mu$ defined on a Polish space $(\gX,\mathrm{dis})$ satisfies a Cheeger isoperimetric inequality with constant $\rho>0$ if for all Borel set 
    $A\subseteq \gX$, it has
    \begin{equation*}
        \lim\inf_{\epsilon \rightarrow 0} \frac{\mu(A^{\epsilon})-\mu(A)}{\epsilon} \ge \frac{1}{\rho} \mu(A)\mu(A^c).
    \end{equation*}
\end{definition}

\begin{lemma}[Theorem 2.5.14 in~\cite{sinho2024logconcave}]
    Let $\mu\in\gP_1(\gX)$ and let $\mathrm{Ch}> 0$. The following are equivalent.
    \begin{enumerate}
        \item $\mu$ satisfies a Cheeger isoperimetric inequality with constant $\mathrm{Ch}$.
        \item For all Lipschitz $f\colon \gX \rightarrow \R$, it holds that
        \begin{equation}
            \label{ineq:cheeger_inequality}
            \E_{\mu}\left|f- \E_\mu f\right|\le 2\rho\cdot \E_{\mu}\left\|\grad f\right\|
        \end{equation}
    \end{enumerate}
\end{lemma}
\begin{remark}
    For a general non-log-concave distribution, a tight bound on the Cheeger constant can hardly be provided. 
    However, considering the Cheeger isoperimetric inequality is stronger than the Poincar\'e inequality,~\cite{buser1982note} lower bound the Cheeger constant $\rho$ with $\Omega(d^{1/2}c_{P})$ where $c_{P}$ is the Poincar\'e constant of $\tilde{\mu}_*$.
    The lower bound of $c_{P}$ can be generally obtained by the Bakry-Emery criterion and achieve $\exp(-\tilde{\mathcal{O}}(d))$.
    While for target distributions with better properties, $\rho$ can usually be much better.
    When the target distribution is a mixture of strongly log-concave distributions, the lower bound of $\rho$ can achieve $1/\mathrm{poly}(d)$ by~\cite{lee2018beyond}.
    For log-concave distributions,~\cite{lee2017eldan} proved that $\rho=\Omega(1/(\Tr(\Sigma^2))^{1/4})$, where $\Sigma$ is the covariance matrix of the distribution $\tilde{\mu}_*$.
    When the target distribution is $m$-strongly log-concave, based on~\cite{dwivedi2019log}, $\rho$ can even achieve $\Omega(\sqrt{L})$.
    In the following, we will prove that the Cheeger constant can be independent of $x_0$.
\end{remark}

\begin{lemma}
    \label{lem:bound_nor_con}
    Suppose $\mu_*$ and $\tilde{\mu_*}$ are defined as Eq.~\ref{def:energy_inner_mala} and Eq.~\ref{def:sta_tilde_mu_*}, respectively, where $R$ in $\tilde{\mu_*}$ is chosen as that in Lemma~\ref{lem:proj_gap_lem}.
    For any $\epsilon\in (0,1)$, we have
    \begin{equation*}
        \frac{1}{2}\le \frac{\int_{\Omega} \tilde{\mu}_*(\der\vz)}{\int_{\R^d} \mu_*(\der\vz)}\le 1.
    \end{equation*}
\end{lemma}
\begin{proof}
    Suppose $\mu_* \propto \exp(-g)$ and $\tilde{\mu}_*$ are the original and truncated target distributions of the inner loops.
    Following from Lemma~\ref{lem:lem66_zou2021faster}, it has
    \begin{equation*}
        \TVD{\mu_*}{\tilde{\mu}_*}\le \frac{\epsilon}{4}
    \end{equation*}
    when $\tilde{\mu}_*$ is deduced by the $R$ shown in Lemma~\ref{lem:proj_gap_lem}.
    Under these conditions, supposing $\Omega = \mathcal{B}(\vzero,R)$, then we have
    \begin{equation}
        \label{ineq:tv_formal}
        \begin{aligned}
            \TVD{\tilde{\mu}_*}{\mu} & = \int_{\R^d} \left|\mu_*(\der \vz) - \tilde{\mu}_*(\der\vz)\right| = \int_{\Omega} \left|\mu_*(\der \vz) - \tilde{\mu}_*(\der \vz)\right| + \int_{\R^d - \Omega} \mu_*(\der\vz)\\
            & =  \int_{\Omega}\left|\frac{\exp\left(-g(\vz)\right)}{\int_{\R^d} \exp\left(-g(\vz^\prime)\right) \der \vz^\prime} - \frac{\exp\left(-g(\vz)\right)}{\int_{\Omega}\exp\left(-g(\vz^\prime)\right) \der \vz^\prime}\right| \der \vz + \int_{\R^d - \Omega} \frac{\exp\left(-g(\vz)\right)}{\int_{\R^d} \exp\left(-g(\vz^\prime)\right)\der\vz^\prime}\der \vz.
        \end{aligned}
    \end{equation}
    Suppose 
    \begin{equation*}
        Z = \int_{\R^d} \exp\left(-g(\vz)\right) \der \vz\quad \mathrm{and}\quad Z_{\Omega} = \int_{\Omega} \exp\left(-g(\vz)\right) \der\vz, 
    \end{equation*}
    then the first term of RHS of Eq.~\ref{ineq:tv_formal} satisfies
    \begin{equation*}
        \begin{aligned}
            & \int_{\Omega}\left|\frac{\exp\left(-g(\vz)\right)}{\int_{\R^d} \exp\left(-g(\vz^\prime)\right) \der \vz^\prime} - \frac{\exp\left(-g(\vz)\right)}{\int_{\Omega}\exp\left(-g(\vz^\prime)\right) \der \vz^\prime}\right| \der \vz \\
            & = \left(\frac{1}{\int_{\Omega} \exp\left(-g(\vz^\prime)\right) \der \vz^\prime} - \frac{1}{\int_{\R^d} \exp\left(-g(\vz^\prime)\right) \der \vz^\prime}\right)\cdot \int_{\Omega} \exp\left(-g(\vz^\prime)\right) \der \vz^\prime = 1-\frac{Z_\Omega}{Z}
        \end{aligned}
    \end{equation*}
    and the second term satisfies
    \begin{equation*}
        \int_{\R^d - \Omega} \frac{\exp\left(-g(\vz)\right)}{\int_{\R^d} \exp\left(-g(\vz^\prime)\right)\der\vz^\prime}\der \vz =  \frac{\int_{\R^d} \exp\left(-g(\vz^\prime)\right)\der\vz^\prime - \int_{\Omega} \exp\left(-g(\vz^\prime)\right)\der\vz^\prime}{\int_{\R^d} \exp\left(-g(\vz^\prime)\right)\der\vz^\prime} = 1-\frac{Z_\Omega}{Z}.
    \end{equation*}
    Combining all these things, we have
    \begin{equation*}
        2\cdot \left(1-\frac{Z_\Omega}{Z}\right)\le \frac{\epsilon}{4}\quad \Rightarrow \quad \frac{1}{2} \le \frac{Z_{\Omega}}{Z}\le 1
    \end{equation*}
    where we suppose $\epsilon\le 1$ without loss of generality. 
    Hence, the proof is completed.
\end{proof}

\begin{lemma}
    \label{lem:truncated_var_bound}
    Suppose $\mu_*$, $\tilde{\mu}_*$ and $\epsilon$ are under the same settings as those in Lemma~\ref{lem:bound_nor_con}, the variance of $\tilde{\mu}_*$ can be upper bounded by $2d/L$.
\end{lemma}
\begin{proof}
    According to the fact that $\mu_*$ is a $L$-strongly log-concave distribution defined on $\R^d$ with the mean $\vv_m$, which satisfies
    \begin{equation*}
        \int_{\R^d} \mu(\vz)\left\|\vz - \vv_m \right\|^2 \der \vz \le \frac{d}{L}
    \end{equation*}
    following from Lemma~\ref{lem:lsi_var_bound}.
    Suppose
    \begin{equation*}
        \Omega = \mathcal{B}(\vzero,R),\quad Z=\int_{\R^d} \exp(-g(\vz))\der\vz,\quad Z_{\Omega} = \int_{\Omega}\exp(-g(\vz))\der \vz
    \end{equation*}
    where $R$ shown in Lemma~\ref{lem:proj_gap_lem}, then the variance bound can be reformulated as
    \begin{equation*}
        \int_{\Omega} \frac{\exp(-g(\vz))}{Z}\left\|\vz - \vv_m\right\|^2\der \vz + \int_{\R^d - \Omega} \frac{\exp(-g(\vz))}{Z}\left\|\vz - \vv_m\right\|^2\der \vz \le \frac{d}{L},
    \end{equation*}
    which implies
    \begin{equation}
        \label{ineq:var_truncated_1}
        \int_{\Omega}\frac{\exp(-g(\vz))}{Z_{\Omega}}\left\|\vz - \vv_m\right\|^2\der \vz \le \frac{Z}{Z_\Omega}\cdot \frac{d}{L}\le \frac{2d}{L}.
    \end{equation}
    Note that the last inequality follows from Lemma~\ref{lem:bound_nor_con}.
    Besides, suppose the mean of $\tilde{\mu}_*$ is $\vv_{\tilde{m}}$, then we have
    \begin{equation}
        \label{ineq:var_truncated_2}
        \begin{aligned}
            & \int_{\Omega} \frac{\exp(-g(\vz))}{Z_\Omega} \cdot \left\|\vz - \vv_m\right\|^2\der \vz = \int_{\Omega} \frac{\exp(-g(\vz))}{Z_\Omega}\cdot \left\|\vz - \vv_{\tilde{m}} + \vv_{\tilde{m}}- \vv_m\right\|^2\der \vz\\
            & = \int_{\Omega} \frac{\exp(-g(\vz))}{Z_\Omega}\cdot \left\|\vz - \vv_{\tilde{m}}\right\|^2 \der\vz + 2\cdot \int_{\Omega} \frac{\exp(-g(\vz))}{Z_\Omega}\cdot \left<\vz-\vv_{\tilde{m}}, \vv_{\tilde{m}}-\vv_m\right>\der \vz\\
            &\quad + \int_{\Omega} \frac{\exp(-g(\vz))}{Z_\Omega}\cdot \left\|\vv_m - \vv_{\tilde{m}}\right\|^2 \der\vz\\
            & = \int_{\Omega} \frac{\exp(-g(\vz))}{Z_\Omega}\cdot \left\|\vz - \vv_{\tilde{m}}\right\|^2 \der\vz +  \int_{\Omega} \frac{\exp(-g(\vz))}{Z_\Omega}\cdot \left\|\vv_m - \vv_{\tilde{m}}\right\|^2 \der\vz
        \end{aligned}
    \end{equation}
    Combining Eq.~\ref{ineq:var_truncated_1} and Eq.~\ref{ineq:var_truncated_2}, the variance of $\tilde{\mu}_*$ satisfies
    \begin{equation*}
        \int_{\Omega} \frac{\exp(-g(\vz))}{Z_\Omega}\cdot \left\|\vz - \vv_{\tilde{m}}\right\|^2 \der\vz \le \frac{2d}{L}.
    \end{equation*}
    Hence, the proof is completed.
\end{proof}

\begin{corollary}
    \label{cor:cheeger_truncation}
    For each truncated target distribution defined as Eq.~\ref{def:sta_tilde_mu_*}, their Cheeger constant can be lower bounded by $\rho = \Omega(\sqrt{L/d})$.
\end{corollary}
\begin{proof}
    It can be easily found that $\tilde{\mu}_*$ is log-concave distribution, which means their Cheeger constant can be upper bounded by $\rho=\Omega(1/(\Tr(\Sigma))^{1/2})$, where $\Sigma$ is the covariance matrix of the distribution $\tilde{\mu}_*$.
    Under these conditions, we have
    \begin{equation*}
        \Tr\left(\Sigma\right) = \int_{\Omega} \frac{\exp(-g(\vz))}{Z_\Omega}\cdot \left\|\vz - \vv_{\tilde{m}}\right\|^2 \der\vz \le \frac{2d}{L},
    \end{equation*}
    where the last inequality follows from Lemma~\ref{lem:truncated_var_bound}.
    Hence, $\rho = \Omega(\sqrt{L/d})$ and the proof is completed.
\end{proof}

\subsubsection{The conductance properties of $\tilde{\gT}_{*}$}

We prove the conductance properties of $\tilde{\gT}_{*,\vz}$ with the following lemma.
\begin{lemma}[Lemma 13 in~\cite{lee2018convergence}]
    \label{lem:lem13_lee2018convergence}
    Let $\tilde{\gT}_{*,\vz}$ be a be a time-reversible Markov chain on $\Omega$ with stationary distribution $\tilde{\mu}_*$.
    Fix any $\Delta>0$, suppose for any $\vz,\vz^\prime \in\Omega$ with $\|\vz-\vz^\prime\|\le \Delta$ we have $\TVD{\tilde{\gT}_{*,\vz}}{\tilde{\gT}_{*,\vz^\prime}}\le 0.99$, then the conductance of $\tilde{\gT}_{*,\vz}$ satisfies $\phi\ge C\rho\Delta$ for some absolute constant $C$, where $\rho$ is the Cheeger constant of $\tilde{\mu}_*$.
\end{lemma}

In order to apply Lemma~\ref{lem:lem13_lee2018convergence}, we have known the Cheeger constant of $\tilde{\mu}_*$ is $\rho$.
We only need to verify the corresponding condition, i.e., proving that as long as $\|\vz-\vz^\prime\|\le \Delta$, we have $\TVD{\tilde{\gT}_{*,\vz}}{\tilde{\gT}_{*,\vz^\prime}}\le 0.99$ for some $\Delta$.
Recalling Eq.~\ref{def:ideal_proj_rtk}, we have
\begin{equation}
    \label{eq:rtk_ideal_proj}
    \begin{aligned}
        \tilde{\gT}_{*, \vz}(\der \hat{\vz}) =  & \tilde{\gT}^\prime_{*, \vz}(\der \hat{\vz})\cdot \tilde{a}_{*, \vz}(\hat{\vz})+\left(1- \int_\Omega \tilde{a}_{*, \vz}(\tilde{\vz}) \tilde{\gT}^\prime_{*, \vz}(\der\tilde{\vz})\right)\cdot \delta_{\vz}(\der\hat{\vz})\\
        = & \left(\frac{1}{2}\delta_{\vz}(\der \hat{\vz}) + \frac{1}{2}\cdot \tilde{Q}^\prime_{*,\vz}(\der \hat{\vz})\right)\cdot  \tilde{a}_{*, \vz}(\hat{\vz})+ \left[1- \int \tilde{a}_{*,\vz}(\tilde{\vz})\cdot \left(\frac{1}{2}\delta_{\vz}(\der \tilde{\vz})+ \frac{1}{2}\tilde{Q}^\prime_{*,\vz}(\der \tilde{\vz})\right) \right]\cdot \delta_{\vz}(\der\hat{\vz})\\
        = & \left(\frac{1}{2}\delta_{\vz}(\der \hat{\vz}) + \frac{1}{2}\cdot \tilde{Q}^\prime_{*,\vz}(\der \hat{\vz})\right)\cdot  \tilde{a}_{*, \vz}(\hat{\vz})+ \left(1- \frac{1}{2}\tilde{a}_{*,\vz}(\vz) - \frac{1}{2}\int \tilde{a}_{*,\vz}(\tilde{\vz})\cdot \tilde{Q}^\prime_{*,\vz}(\der\tilde{\vz})\right)\cdot \delta_{\vz}(\der\hat{\vz})\\
        = & \left(1 - \frac{1}{2}\int  \tilde{a}_{*,\vz}(\tilde{\vz})\cdot \tilde{Q}^\prime_{*,\vz}(\der\tilde{\vz})\right) \cdot \delta_{\vz}(\der\hat{\vz})  + \frac{1}{2}\cdot \tilde{Q}^\prime_{*,\vz}(\der\hat{\vz})\cdot  \tilde{a}_{*,\vz}(\hat{\vz})\\
        = & \left(1-\frac{1}{2}\int_{\Omega_{\vz}} \tilde{a}_{*,\vz}(\tilde{\vz}) \tilde{Q}_{*,\vz}(\der\tilde{\vz})\right) + \frac{1}{2}\cdot \tilde{Q}_{*,\vz}(\der\hat{\vz})\cdot \tilde{a}_{*,\vz}(\hat{\vz})\cdot\vone\left[\hat{\vz}\in \Omega_{\vz}\right],
    \end{aligned}
\end{equation}
where the second inequality follows from Eq.~\ref{eq:idea_lazy} and the last inequality follows from Eq.~\ref{eq:idea_proj}.
Then the rest will be proving the upper bound of $\TVD{\tilde{\gT}_{*,\vz}}{\tilde{\gT}_{*,\vz^\prime}}$, and we state another two useful lemmas as follows.


\begin{lemma}[Lemma B.6 in~\cite{zou2021faster}]
    \label{lem:lemB6_zou2021faster}
    For any two points $\vz,\vz^\prime\in \R^d$, it holds that
    \begin{equation*}
        \TVD{\tilde{Q}_{*,\vz}(\cdot)}{\tilde{Q}_{*,\vz^\prime}(\cdot)}\le \frac{(1+3L\tau)\left\|\vz-\vz^\prime\right\|}{\sqrt{2\tau}}
    \end{equation*}
\end{lemma}
\begin{proof}
    This lemma can be easily obtained by plugging the smoothness of $g$, i.e., $3L$, into Lemma B.6 in~\cite{zou2021faster}.
\end{proof}

\begin{corollary}[Variant of Lemma 6.5 in~\cite{zou2021faster}]
    \label{cor:lem65_zou2021faster}
    Under Assumption~\ref{ass:lips_score}--\ref{ass:second_moment}, we set 
    \begin{equation*}
    \eta= \frac{1}{2}\log \frac{2L+1}{2L}\quad \mathrm{and}\quad G\coloneqq \left\|\grad g(\vzero)\right\|.
    \end{equation*}
    If we set 
    \begin{equation*}
        \tau\le \frac{1}{16\cdot  (3LR+G+\epsilon_{\mathrm{score}})^{2}} \quad \mathrm{and}\quad r = 3\cdot \sqrt{\tau d \log \frac{8S}{\epsilon}}
    \end{equation*} 
    there exist absolute constants $c_0$, such that
    $\phi \ge c_0\rho\sqrt{\tau}$ where $\rho$ is the Cheeger constant of the distribution $\tilde{\mu}_*$.
\end{corollary}
\begin{proof}
    By the definition of total variation distance, there exists a set $\gA\subseteq \Omega$ satisfying
    \begin{equation*}
        \begin{aligned}
            \TVD{\tilde{\gT}_{*,\vz}(\cdot)}{\tilde{\gT}_{*,\vz^\prime}(\cdot)} = & \left|\tilde{\gT}_{*,\vz}(\gA) - \tilde{\gT}_{*,\vz^\prime}(\gA) \right|.
        \end{aligned}
    \end{equation*}
    Due to the closed form of $\tilde{\gT}_{*,\vz} $ shown in Eq.~\ref{eq:rtk_ideal_proj}, we have 
    \begin{equation*}
        \tilde{\gT}_{*,\vz}(\gA) = \left(1-\frac{1}{2}\int_{\tilde{\vz}\in \Omega_{\vz}} \tilde{a}_{*,\vz}(\tilde{\vz})\tilde{Q}_{*,\vz}(\der \tilde{\vz})\right) + \frac{1}{2}\int_{\hat{\vz}\in \gA} \tilde{a}_{*,\vz}(\hat{\vz})\cdot\vone\left[\hat{\vz}\in \Omega_{\vz}\right]\tilde{Q}_{*,\vz}(\der \hat{\vz})
    \end{equation*}
    Under this condition, we have
    \begin{equation*}
        \begin{aligned}
            \left|\tilde{\gT}_{*,\vz}(\gA) - \tilde{\gT}_{*,\vz^\prime}(\gA) \right| \le & \underbrace{\max_{\hat{\vz}} \left(1-\frac{1}{2}\int_{\Omega_{\hat{\vz}}} \tilde{a}_{*,\hat{\vz}}(\tilde{\vz})\tilde{Q}_{*,\hat{\vz}}(\der \tilde{\vz})\right)}_{\mathrm{Term\ 1}}\\
            & + \underbrace{\frac{1}{2}\left|\int_{\hat{\vz}\in \gA} \tilde{a}_{*,\vz}(\hat{\vz})\cdot\vone\left[\hat{\vz}\in \Omega_{\vz}\right]\tilde{Q}_{*,\vz}(\der \hat{\vz}) - \tilde{a}_{*,\vz^\prime}(\hat{\vz})\cdot\vone\left[\hat{\vz}\in \Omega_{\vz^\prime}\right]\tilde{Q}_{*,\vz^\prime}(\der \hat{\vz})\right|}_{\mathrm{Term\ 2}}.
        \end{aligned}
    \end{equation*}
    \paragraph{Upper bound $\mathrm{Term\ 1}$.}
    We first consider to lower bound $\tilde{a}_{*,\hat{\vz}}(\tilde{\vz})$ in the following.
    According to Eq.~\ref{eq:idea_acc_rate}, we have
    \begin{equation*}
        \begin{aligned}
            \tilde{a}_{*,\hat{\vz}}(\tilde{\vz})\ge \exp\left(-g(\tilde{\vz}) - \frac{\left\|\hat{\vz} - \tilde{\vz} + \tau\grad g(\tilde{\vz})\right\|^2}{4\tau} + g(\hat{\vz}) + \frac{\left\|\tilde{\vz}-\hat{\vz} + \tau \grad g(\hat{\vz})\right\|^2}{4\tau}\right),
        \end{aligned}
    \end{equation*}
    which means
    \begin{equation*}
        \begin{aligned}
            4\ln \tilde{a}_{*,\hat{\vz}}(\tilde{\vz}) \ge & \underbrace{\tau\cdot \left(\left\|\grad g(\hat{\vz})\right\|^2 - \left\|\grad g(\tilde{\vz})\right\|^2\right)}_{\mathrm{Term\ 1.1}}\\
            & \underbrace{- 2 \cdot \left(g(\tilde{\vz}) - g(\hat{\vz}) - \left<\grad g(\hat{\vz}), \tilde{\vz}-\hat{\vz}\right>\right)}_{\mathrm{Term\ 1.2}}\\
            & \underbrace{+ 2 \cdot \left(g(\hat{\vz}) - g(\tilde{\vz}) - \left<\grad g(\tilde{\vz}), \hat{\vz}-\tilde{\vz}\right>\right)}_{\mathrm{Term\ 1.3}}.
        \end{aligned}
    \end{equation*}
    Since $\mathrm{Term\ 1.2}$ and $\mathrm{Term\ 1.3}$ are grouped to more easily apply the strong convexity and smoothness of $g$ (Lemma~\ref{lem:sc_sm_of_g}),
    it has
    \begin{equation*}
        \mathrm{Term\ 1.2}\ge -3L\left\|\hat{\vz} - \tilde{\vz}\right\|^2\quad \mathrm{and}\quad \mathrm{Term\ 1.3}\ge L\left\|\hat{\vz}-\tilde{\vz}\right\|^2\ge 0.
    \end{equation*}
    Besides, by requiring $\tau\le 1/3L$, we have
    \begin{equation*}
        \begin{aligned}
            \mathrm{Term\ 1.1}= & \tau \cdot \left<\grad g(\hat{\vz}) - \grad g(\tilde{\vz}), \grad g(\hat{\vz}) + \grad g(\tilde{\vz})\right>\\
            \ge & -\tau\cdot \left\|\grad g(\hat{\vz}) - \grad g(\tilde{\vz})\right\|\cdot\left\|\grad g(\hat{\vz}) + \grad g(\tilde{\vz})\right\|\\
            \ge & -3L\tau\left\|\hat{\vz}-\tilde{\vz}\right\| \cdot \left(2\left\|\grad g(\hat{\vz})\right\| + 3L\left\|\hat{\vz}-\tilde{\vz}\right\|\right) \ge -3L\tau^2 \left\|\grad g(\hat{\vz})\right\|^2 - 6L\left\|\hat{\vz}-\tilde{\vz}\right\|^2.
        \end{aligned}
    \end{equation*}
    Therefore, 
    \begin{equation*}
        \begin{aligned}
            4\ln \tilde{a}_{*,\hat{\vz}}(\tilde{\vz}) \ge & -3L\tau^2 \left\|\grad g(\hat{\vz})\right\|^2 - 9L\left\|\hat{\vz}-\tilde{\vz}\right\|^2 = -3L\tau^2 \left\|\grad g(\hat{\vz})\right\|^2 - 9L\left\|\tau\cdot \grad g(\hat{\vz}) + \sqrt{2\tau}\cdot \xi\right\|^2\\
            \ge & -21L\tau^2 \left\|\grad g(\hat{\vz})\right\|^2-36L\tau\left\|\xi\right\|^2,
        \end{aligned}
    \end{equation*}
    and
    \begin{equation*}
        \ln \tilde{a}_{*,\hat{\vz}}(\tilde{\vz}) \ge -6L\tau^2 \left\|\grad g(\hat{\vz})\right\|^2-9L\tau\left\|\xi\right\|^2 \ge -6L\tau^2\cdot\left(3LR + \|\grad g(\vzero)\|\right)^2 - 9L\tau\|\xi\|^2
    \end{equation*}
    where the last inequality follows from 
    \begin{equation*}
        \begin{aligned}
            \left\|\grad g(\hat{\vz})\right\| \le \left\|\grad g(\hat{\vz}) - \grad g(\vzero)\right\| + \left\|\grad g(\vzero)\right\|\le 3LR + \left\|\grad g(\vzero)\right\|.
        \end{aligned}
    \end{equation*}
    Under these conditions, we have
    \begin{equation}
        \label{ineq:lem65_term1_mid_bound}
        \begin{aligned}
            \mathrm{Term\ 1} \le & 1 - \frac{1}{2}\cdot \exp\left(-6L\tau^2 \left(3LR+\|\grad g(\vzero)\|\right)^2\right)\cdot \min \int_{\Omega_{\hat{\vz}}}\exp\left(-9L\tau \|\xi\|^2\right)\cdot \tilde{q}_{*,\hat{\vz}}(\tilde{\vz})\der \tilde{\vz}\\
            = & 1 - \frac{1}{2}\cdot \exp\left(-6L\tau^2 \left(3LR+\|\grad g(\vzero)\|\right)^2\right)\cdot \E_{\xi\sim \mathcal{N}(\vzero,\mI)}\left[\exp\left(-9L\tau \|\xi\|^2\right)\right]\\
            \le & 1 - 0.4\cdot \exp\left(-6L\tau^2 \left(3LR+\|\grad g(\vzero)\|\right)^2\right)\cdot \exp\left(-18L\tau d\right),
        \end{aligned}
    \end{equation}
    where the last inequality follows from the Markov inequality shown in the following
    \begin{equation*}
        \begin{aligned}
            \E_{\xi\sim \mathcal{N}(\vzero,\mI)}\left[\exp\left(-9L\tau \|\xi\|^2\right)\right] \ge & \exp\left(-18L\tau d\right)\cdot \mathbb{P}_{\xi\sim \mathcal{N}(\vzero,\mI)}\left[\exp\left(-9L\tau \|\xi\|^2\right) \ge \exp\left(-18L\tau d\right)\right]\\
            = & \exp\left(-18L\tau d\right)\cdot \mathbb{P}_{\xi\sim \mathcal{N}(\vzero,\mI)}\left[\|\xi\|^2\le 2d\right] \ge \exp\left(-18L\tau d\right)\cdot \left(1-\exp(-d/2)\right).
        \end{aligned}
    \end{equation*}
    Then, by choosing 
    \begin{equation}
        \label{ineq:lem65_tau_choice1}
        \tau \le \frac{1}{16\sqrt{L}\cdot \left(3LR + \left\|\grad g(\vzero)\right\|\right)},
    \end{equation}
    it has $6L\tau^2\left(3LR + \left\|\grad g(\vzero)\right\|\right)^2 \le 1/40$. 
    Besides by choosing
    \begin{equation}
        \label{ineq:lem65_tau_choice2}
        \tau\le  \frac{1}{L^2R^2} \le \frac{1}{40\cdot 18L \cdot \left(\sqrt{d} + \sqrt{\ln \frac{16S}{\epsilon}}\right)^2}  \le  \frac{1}{40\cdot 18L\cdot d },
    \end{equation}
    where the last inequality follows from the range of $R$ shown in Lemma~\ref{lem:proj_gap_lem}, 
    it has $18Ld\tau\le 1/40$.
    Under these conditions, considering Eq.~\ref{ineq:lem65_term1_mid_bound}, we have
    \begin{equation}
        \label{ineq:ac_lower_bound}
        \mathrm{Term\ 1}\le 1-0.5\cdot \min_{\hat{\vz}\in \Omega, \tilde{\vz}\in \Omega_{\hat{\vz}}} \int_{\Omega_{\hat{\vz}}} \tilde{a}_{*,\hat{\vz}}(\tilde{\vz}) \tilde{Q}_{*,\hat{\vz}}(\der\tilde{\vz})\le 1-0.4\cdot e^{-1/20}.
    \end{equation}
    Then, combining the step size choices of Eq.~\ref{ineq:lem65_tau_choice1}, Eq.~\ref{ineq:lem65_tau_choice2}, and Lemma~\ref{lem:proj_gap_lem}, 
    since the requirement
    \begin{equation*}
        \tau \le \frac{1}{16\sqrt{L}\cdot \left(3LR + \left\|\grad g(\vzero)\right\|\right)},\quad \tau\le \frac{1}{L^2R^2}  \quad \mathrm{and}\quad \tau\le \frac{d}{(3LR+\left\|\grad g(\vzero)\right\|+\epsilon_{\mathrm{score}})^2}     
    \end{equation*}
    can be achieved by
    \begin{equation}
        \label{ineq:tau_choice_all1}
        \tau\le {16^{-1}\cdot (3LR+\left\|\grad g(\vzero)\right\|+\epsilon_{\mathrm{score}})^{-2}},
    \end{equation}
    the range of $\tau$ can be determined.

    \paragraph{Upper bound $\mathrm{Term\ 2}$.} In This part, we use similar techniques as those shown in Lemma 6.5 of~\cite{zou2021faster}.
    According to the triangle inequality, we have
    \begin{equation}
        \label{ineq:lem65_term2}
        \begin{aligned}
            2\cdot \mathrm{Term\ 2} \le &\int_{\hat{\vz}\in \gA} \left(1-\tilde{a}_{*,\vz}(\hat{\vz})\right)\tilde{q}(\hat{\vz}|\vz)\vone\left[\hat{\vz}\in \Omega_{\vz}\right]\der \hat{\vz}+ \int_{\hat{\vz}\in \gA} \left(1-\tilde{a}_{*,\vz^\prime}(\hat{\vz})\right)\tilde{q}(\hat{\vz}|\vz^\prime)\vone\left[\hat{\vz}\in \Omega_{\vz^\prime}\right]\der \hat{\vz}\\
            & + \left|\int_{\hat{\vz}\in \gA} \left(\tilde{q}(\hat{\vz}|\vz)\vone\left[\hat{\vz}\in \Omega_{\vz}\right] - \tilde{q}(\hat{\vz}|\vz^\prime)\vone\left[\hat{\vz}\in \Omega_{\vz^\prime}\right]\right) \der\hat{\vz}\right|\\
            \le & 2\cdot \left(1- \min_{\hat{\vz}\in \Omega, \tilde{\vz}\in \Omega_{\hat{\vz}}} \int_{\Omega_{\hat{\vz}}} \tilde{a}_{*,\hat{\vz}}(\tilde{\vz}) \tilde{Q}_{*,\hat{\vz}}(\der\tilde{\vz})\right)\\
            &  + \underbrace{\left|\int_{\hat{\vz}\in \gA} \left(\tilde{q}(\hat{\vz}|\vz)\vone\left[\hat{\vz}\in \Omega_{\vz}\right] - \tilde{q}(\hat{\vz}|\vz^\prime)\vone\left[\hat{\vz}\in \Omega_{\vz^\prime}\right]\right) \der\hat{\vz}\right|}_{\mathrm{Term\ 2.1}}.
        \end{aligned}
    \end{equation}
    Then, we upper bound $\mathrm{Term\ 2.1}$ as follows
    \begin{equation*}
        \begin{aligned}
            \mathrm{Term\ 2.1}\le & \left|\int_{\hat{\vz}\in \gA} \vone\left[\hat{\vz}\in \Omega_{\vz^\prime}\right]\cdot \left(\tilde{q}_(\hat{\vz}|\vz) - \tilde{q}_(\hat{\vz}|\vz^\prime)\right)\right|+\left|\int_{\hat{\vz}\in \gA}\left(\vone\left[\hat{\vz}\in \Omega_{\vz}\right]-\vone\left[\hat{\vz}\in \Omega_{\vz^\prime}\right]\right) \cdot \tilde{q}(\hat{\vz}|\vz)\right|\\
            \le & \TVD{\tilde{Q}_{*,\vz}(\cdot)}{\tilde{Q}_{*,\vz^\prime}(\cdot)} + \max \left\{\int_{\hat{\vz}\in \Omega_{\vz^\prime}-\Omega_{\vz}}\tilde{q}(\hat{\vz}|\vz)\der \hat{\vz}, \int_{\hat{\vz}\in \Omega_{\vz}-\Omega_{\vz^\prime}} \tilde{q}(\hat{\vz}|\vz^\prime)\der \hat{\vz}\right\}\\
            \le & \TVD{\tilde{Q}_{*,\vz}(\cdot)}{\tilde{Q}_{*,\vz^\prime}(\cdot)} + \max \left\{\int_{\hat{\vz}\in \R^d-\Omega_{\vz}}\tilde{q}(\hat{\vz}|\vz)\der \hat{\vz}, \int_{\hat{\vz}\in \R^d-\Omega_{\vz^\prime}} \tilde{q}(\hat{\vz}|\vz^\prime)\der \hat{\vz}\right\}
        \end{aligned}
    \end{equation*}
    According to the definition, $\tilde{q}_{*,\vz}(\cdot)$ is Gaussian distribution with mean $\vz-\tau \grad g(\vz)$ and covariance matrix $2\tau \mI$, thus we have
    \begin{equation*}
        \begin{aligned}
            &\int_{\hat{\vz}\in \R^d - \Omega_{\vz}} \tilde{q}(\hat{\vz}|\vz)\der \hat{\vz} \le \mathbb{P}_{\hat{\rvz}\sim \chi_d^2}\left[\hat{\rvz}\ge \frac{1}{2}\left(r-\tau \left\|\grad g(\vz)\right\|\right)^2/\tau\right]\\
            & \int_{\hat{\vz}\in \R^d - \Omega_{\vz^\prime}} \tilde{q}(\hat{\vz}|\vz^\prime)\der \hat{\vz} \le \mathbb{P}_{\hat{\rvz}\sim \chi_d^2}\left[\hat{\rvz}\ge \frac{1}{2}\left(r-\tau \left\|\grad g(\vz)\right\| - \left\|\vz-\vz^\prime\right\|\right)^2/\tau\right].
        \end{aligned}
    \end{equation*}

    Then, we start to lower bound 
    \begin{equation*}
        r-\tau \left\|\grad g(\vz)\right\| - \left\|\vz-\vz^\prime\right\|.
    \end{equation*}
    Then, we require 
    \begin{equation}
        \label{ineq:zprob_lower_bound}
        \left\|\vz-\vz^\prime\right\|\le 0.1r\quad \mathrm{and}\quad \tau\le\frac{d}{35\cdot \left(3LR+G\right)^2}
    \end{equation}
    where the latter condition can be easily covered by the choice in Eq.~\ref{ineq:tau_choice_all1} when $d\ge 3$ without loss of generality.
    Under this condition, we have
    \begin{equation}
        \label{ineq:z_diff}
        \begin{aligned}
            \tau\le (0.17)^2\cdot \frac{d}{(3LR+G)^2}\quad \Leftrightarrow\quad \sqrt{\tau}\le \frac{0.17\sqrt{d}}{3LR+G}.
        \end{aligned}
    \end{equation}
    Since we have
    \begin{equation*}
        \left\|\grad g(\vz)\right\|=\left\|\grad g(\vz) - \grad g(\vzero) + \grad g(\vzero)\right\|\le 3L\cdot\left\|\vz\right\| + G\le 3LR+G,
    \end{equation*}
    by the smoothness, it has
    \begin{equation}
        \label{ineq:tau_gradg}
        \sqrt{\tau}\le \frac{0.17\sqrt{d}}{\left\|\grad g(\vz)\right\|}\quad \Leftrightarrow \quad \tau\left\|\grad g(\vz)\right\| \le 0.17\sqrt{\tau d}
    \end{equation}
    Plugging Eq.~\ref{ineq:tau_gradg} and Eq.~\ref{ineq:z_diff} into Eq.~\ref{ineq:zprob_lower_bound}, we have
    \begin{equation*}
        r-\tau \left\|\grad g(\vz)\right\| - \left\|\vz-\vz^\prime\right\|\ge 0.9r - 0.17\sqrt{\tau d}\ge \sqrt{6.4\tau d}
    \end{equation*}
    where the last inequality follows from the choice of $r$ shown in Lemma~\ref{lem:lem62_zou2021faster_gene}, i.e.,
    \begin{equation*}
        r = 3\cdot \sqrt{\tau d \log \frac{8S}{\epsilon}}\ge 3\cdot\sqrt{\tau d}.
    \end{equation*}
    Under these conditions, we have
    \begin{equation*}
        \max \left\{\int_{\hat{\vz}\in \R^d-\Omega_{\vz}}\tilde{q}(\hat{\vz}|\vz)\der \hat{\vz}, \int_{\hat{\vz}\in \R^d-\Omega_{\vz^\prime}} \tilde{q}(\hat{\vz}|\vz^\prime)\der \hat{\vz}\right\}\le \mathbb{P}_{\hat{\rvz}\sim \chi_d^2}\left(\|\rvz\|\ge 3.2d\right)\le 0.1.
    \end{equation*}
    Then combine the above results and apply Lemma~\ref{lem:lemB6_zou2021faster}, assume $\tau\le 1/(3L)$, we have
    \begin{equation*}
        \mathrm{Term\ 2.1} \le 0.1+\TVD{\tilde{Q}_{*,\vz}(\cdot)}{\tilde{Q}_{*,\vz^\prime}(\cdot)}\le 0.1+ \sqrt{2/\tau}\cdot \left\|\vz-\vz^\prime\right\|
    \end{equation*}
    
    Plugging the above into Eq.~\ref{ineq:lem65_term2}, we have
    \begin{equation*}
        \begin{aligned}
            \mathrm{Term\ 2} \le  &\left(1-\min_{\hat{\vz}\in \Omega, \tilde{\vz}\in \Omega_{\hat{\vz}}} \int_{\Omega_{\hat{\vz}}} \tilde{a}_{*,\hat{\vz}}(\tilde{\vz}) \tilde{Q}_{*,\hat{\vz}}(\der\tilde{\vz})\right) + \frac{1}{2}\cdot \left(0.1+\sqrt{\frac{2}{\tau}}\cdot \left\|\vz-\vz^\prime\right\|\right)\\
            \le & \left(1-0.8\cdot e^{-1/20}\right)+0.05 + (2\tau)^{-1/2}\cdot \|\vz-\vz^\prime\|,
        \end{aligned}
    \end{equation*}
    where the second inequality follows from Eq.~\ref{ineq:ac_lower_bound}.
    
    After upper bounding $\mathrm{Term\ 1}$ and $\mathrm{Term\ 2}$, we have
    \begin{equation*}
        \begin{aligned}
            \TVD{\tilde{\gT}_{*,\vz}(\cdot)}{\tilde{\gT}_{*,\vz^\prime}(\cdot)}\le & 1-0.4\cdot e^{-1/20} + \left(1-0.8\cdot e^{-1/20}\right)+0.05 + (2\tau)^{-1/2}\cdot \|\vz-\vz^\prime\|\\
            \le & 0.91 + (2\tau)^{-1/2}\cdot \|\vz-\vz^\prime\| \le 0.99
        \end{aligned}
    \end{equation*}
    where the last inequality can be established by requiring $\|\vz-\vz^\prime\|\le \sqrt{2\tau}$.
    Combining Lemma~\ref{lem:lem13_lee2018convergence}, the conductance of $\tilde{\mu}_{*}$ satisfies
    \begin{equation*}
        \phi \ge c_0\cdot \rho\sqrt{2\tau}.
    \end{equation*}
    Hence, the proof is completed.
\end{proof}

\paragraph{The connection between $\tilde{\mu}_S$ and $\tilde{\mu}_*$.}
With the conductance of truncated target distribution, we are able to find the convergence of the projected implementation of Alg.~\ref{alg:inner_mala}.
Besides, the gap between the truncated target $\tilde{\mu}_*$ and the true target $\mu_*$ can be upper bounded by controlling $R$ while such an $R$ will be dominated by the range of $R$ shown in Lemma~\ref{lem:proj_gap_lem}. 
In this section, we will omit several details since many of them have been proven in~\cite{zou2021faster}.

\begin{lemma}[Lemma 6.4 in~\cite{zou2021faster}]
    \label{lem:lem64_zou2021faster}
    Let $\tilde{\mu}_S$ be distributions of the outputs of the projected implementation of Alg.~\ref{alg:inner_mala}. 
    Under Assumption~\ref{ass:lips_score}--\ref{ass:second_moment}, if the transition kernel $\tilde{\gT}_{\vz}(\cdot)$ is $\delta$-close to $\tilde{\gT}_{*,\vz}$ with $\delta\le \min\left\{1-\sqrt{2}/2, \phi/16\right\}$ ($\phi$ denotes the conductance of $\tilde{\mu}_*$), then for any $\lambda$-warm start initial distribution with respect to $\tilde{\mu}_*$, it holds that
    \begin{equation*}
        \TVD{\tilde{\mu}_S}{\tilde{\mu}_*}\le \lambda\cdot \left(1-\phi^2/8\right)^S + 16\delta/\phi.
    \end{equation*}
\end{lemma}

\begin{lemma}[Lemma 6.6 in~\cite{zou2021faster}]
    \label{lem:lem66_zou2021faster}
    For any $\epsilon\in (0,1)$, set $R$ to make it satisfy
    \begin{equation*}
        \mu\left(\gB(\vzero, R)\right)\ge 1-\frac{\epsilon}{12},
    \end{equation*}
    and $\tilde{\mu}_*$ be the truncated target distribution of $\mu_*$.
    Then the total variation distance between $\mu_*$ and $\tilde{\mu}_*$ can be upper bounded by  $\TVD{\tilde{\mu}_*}{\mu_*}\le \epsilon/4$.
\end{lemma}

\subsection{Main Theorems of InnerMALA implementation}

\begin{lemma}
    \label{lem:inner_para_choice}
    Under Assumption~\ref{ass:lips_score}--\ref{ass:second_moment}, we can upper bound $G=\left\|\grad g(\vzero)\right\|$ as
    \begin{equation*}
        \left\|\grad g(\vzero)\right\| \le L\cdot \sqrt{2(d+m_2^2)} + 3L\cdot \left\|\vx_0\right\|.
    \end{equation*}
    Furthermore, we can reformulate $R$ as
    \begin{equation*}
        R = 63\cdot \sqrt{(d+m_2^2 + \|\vx_0\|^2)\cdot \log \frac{16S}{\epsilon}}
    \end{equation*}
    to make it satisfy the requirement shown in Lemma~\ref{lem:lem62_zou2021faster_gene}.
    Then, the range of inner step sizes, i.e., $\tau$, will satisfy
    \begin{equation*}
        \tau\le C_\tau\cdot \left(L^2\left(d+m_2^2+\|\vx_0\|^2\right)\cdot \log\frac{16S}{\epsilon}\right)^{-1},
    \end{equation*}
    where the absolute constant $C_\tau = 2^{-4}\cdot 3^{-8}\cdot 7^{-2}$.
\end{lemma}
\begin{proof}
    To make the bound more explicit, we control $R$ and $G$ in our previous analysis.
    For $G=\|\grad g(\vzero)\|$, according to Eq.~\ref{def:energy_inner_mala}, we have
    \begin{equation*}
        \grad g(\vz) = \grad f_{(K-k-1)\eta}(\vz) + \frac{e^{-2\eta}\vz - e^{-\eta}\vx_0}{(1-e^{-2\eta})},
    \end{equation*}
    which means
    \begin{equation*}
        \begin{aligned}
            \left\|\grad g(\vzero)\right\| \le & \left\|\grad f_{(K-k-1)\eta}(\vzero)\right\| + \left\|\frac{e^{-\eta}\vx_0}{1-e^{-2\eta}}\right\|\\
            \le & \left\|\grad f_{(K-k-1)\eta}(\vzero)\right\| + \sqrt{\frac{2L}{2L+1}}\cdot (2L+1) \cdot \left\|\vx_0\right\|\le \left\|\grad f_{(K-k-1)\eta}(\vzero)\right\| + (2L+1) \cdot \left\|\vx_0\right\|.
        \end{aligned}
    \end{equation*}
    Besides, we should note $f_{(K-k-1)\eta}$ is the smooth (Assumption~\ref{ass:lips_score}) energy function of $p_{(K-k-1)\eta}$ denoting the underlying distribution of time $(K-k-1)\eta$ in the forward OU process.
    Then, we have
    \begin{equation}
        \label{ineq:gradf0_ub}
        \begin{aligned}
            \left\|\grad f_{(K-k-1)\eta}(\vzero)\right\|^2 = & \E_{p_{(K-k-1)\eta}}\left[\left\|\grad f_{(K-k-1)\eta}(\vzero)\right\|^2\right]\\
            \le & 2\E_{p_{(K-k-1)\eta}}\left[\left\|\grad f_{(K-k-1)\eta}(\rvx)\right\|^2\right] + 2\E_{p_{(K-k-1)\eta}}\left[\left\|\grad f_{(K-k-1)\eta}(\rvx) - \grad f_{(K-k-1)\eta}(\vzero)\right\|^2\right]\\
            \le & 2Ld + 2L^2 \E_{p_{(K-k-1)\eta}}\left[\left\|\rvx\right\|^2\right] \le 2Ld + 2L^2\max\left\{d, m_2^2\right\}\le 2L^2(d+m_2^2)
        \end{aligned}
    \end{equation}
    where the first inequality follows from Lemma~\ref{lem:lem11_vempala2019rapid}, and the third inequality follows from Lemma~\ref{lem:lem10_chen2022sampling}.
    Under these conditions, we have
    \begin{equation}   
        \label{ineq:gradg0_range}
        \left\|\grad g(\vzero)\right\| \le L\cdot \sqrt{2(d+m_2^2)} + 3L\cdot \left\|\vx_0\right\|.
    \end{equation}
    Then, for $R$ defined as
    \begin{equation*}
        R \ge \max\left\{8\cdot \sqrt{\frac{\|\grad g(\vzero)\|^2}{L^2}+\frac{d}{L}}, 63\cdot \sqrt{\frac{d}{L}\log\frac{16S}{\epsilon}}\right\},
    \end{equation*}
    we can choose $R$ to be the upper bound of RHS. 
    Considering
    \begin{equation*}
        8\cdot \sqrt{\frac{\|\grad g(\vzero)\|^2}{L^2}+\frac{d}{L}} \le 8\cdot \sqrt{\frac{4L^2(d+m_2^2)+18L^2\|\vx_0\|^2}{L^2}+d} \le 63\cdot \sqrt{(d+m_2^2 + \|\vx_0\|^2)},
    \end{equation*}
    then we choose
    \begin{equation*}
        R = 63\cdot \sqrt{(d+m_2^2 + \|\vx_0\|^2)\cdot \log \frac{16S}{\epsilon}}.
    \end{equation*}
    After determining $R$, the choice of $\tau$ can be relaxed to
    \begin{equation*}
        \tau\le C_\tau\cdot \left(L^2\left(d+m_2^2+\|\vx_0\|^2\right)\cdot \log\frac{16S}{\epsilon}\right)^{-1},
    \end{equation*}
    where the absolute constant $C_\tau = 2^{-4}\cdot 3^{-8}\cdot 7^{-2}$, since we have
    \begin{equation*}
        \begin{aligned}
            & \left(3LR+G+\epsilon_{\mathrm{score}}\right)^2\le 9L^2R^2 + 4G^2 \\
            & \le 9L^2\cdot 63^2 \cdot \left(d+m_2^2+\|\vx_0\|^2\right)\cdot \log\frac{16S}{\epsilon} + 4 \left(4L^2\cdot \left(d+m_2^2\right) + 18 L^2 \|\vx_0\|^2\right)\\
            & \le 9\cdot 63^2\cdot L^2\left(d+m_2^2+\|\vx_0\|^2\right)\cdot \log\frac{16S}{\epsilon}.
        \end{aligned}
    \end{equation*}
    Hence, the proof is completed.
\end{proof}

\begin{theorem}
    \label{thm:inner_convergence}
    Under Assumption~\ref{ass:lips_score}--\ref{ass:second_moment}, for any $\epsilon\in(0,1)$, let $\tilde{\mu}_*(\vz)\propto \exp(-g(\vz))\vone[\vz\in\mathcal{B}(\vzero,R)]$ be the truncated target distribution in $\mathcal{B}(\vzero,R)$ with
    \begin{equation*}
        R = 63\cdot \sqrt{(d+m_2^2 + \|\vx_0\|^2)\cdot \log \frac{16S}{\epsilon}}=\tilde{\mathcal{O}} \left((d+m_2^2 + \|\vx_0\|^2)^{1/2}\right),
    \end{equation*}
    $r$ in Alg.~\ref{alg:inner_mala} satisfies
    \begin{equation*}
        r = 3\cdot \sqrt{\tau d \log \frac{8S}{\epsilon}} = \tilde{\mathcal{O}}(\tau^{1/2}d^{1/2})
    \end{equation*}
    and $\rho$ be the Cheeger constant of $\tilde{\mu}_*$.
    Suppose $\tilde{\mu}_0(\{\|\rvx\|\ge R/2\})\le \epsilon/16$, the step size satisfy
    \begin{equation*}
        \tau\le C_\tau\cdot \left(L^2\left(d+m_2^2+\|\vx_0\|^2\right)\cdot \log\frac{16S}{\epsilon}\right)^{-1}=\tilde{\mathcal{O}}(L^{-2}\cdot (d+m_2^2+\|\vx_0\|^2)^{-1}),
    \end{equation*}
    the score and energy estimation errors satisfy
    \begin{equation*}
        \epsilon_{\mathrm{score}}\le \frac{c_0\rho}{32\cdot 36\cdot \sqrt{d\log \frac{8S}{\epsilon}}} = \mathcal{O}(\rho d^{-1/2})\quad \mathrm{and}\quad \epsilon_{\mathrm{energy}}\le \frac{c_0\rho \sqrt{2\tau}}{32\cdot 5} = \mathcal{O}(\rho \tau^{1/2}),
    \end{equation*}
    then for any $\lambda$-warm start with respect to $\mu_*$ the output of both standard and projected implementation of Alg.~\ref{alg:inner_mala} satisfies
    \begin{equation*}
        \TVD{\mu_S}{\mu_*} = \frac{\epsilon}{2} + \lambda\left(1- \frac{c_0^2\rho^2}{4} \cdot \tau\right)^S + \tilde{\mathcal{O}}(d^{1/2}\rho^{-1}\epsilon_{\mathrm{score}}) + \mathcal{O}({\rho^{-1}\tau^{-1/2}\epsilon_\mathrm{energy}})
    \end{equation*}
\end{theorem}
\begin{proof}
    We characterize the condition on the step size $\tau$.
    Combining Lemma~\ref{lem:proj_gap_lem} and Corollary~\ref{cor:lem65_zou2021faster}, it requires the range of $\tau$ to satisfy
    \begin{equation*}
        \tau \le {16^{-1}\cdot (3LR+\left\|\grad g(\vzero)\right\|+\epsilon_{\mathrm{score}})^{-2}}.
    \end{equation*}
    Under this condition, we have
    \begin{equation*}
        \tau\le \frac{d\log \frac{8S}{\epsilon}}{(3LR+G)^2},\quad  \tau\le \frac{d\log \frac{8S}{\epsilon}}{\epsilon_{\mathrm{score}}^2},\quad \mathrm{and}\quad \tau\le \left(72^2\cdot \epsilon_{\mathrm{score}}^2\cdot d\log \frac{8S}{\epsilon}\right)^{-1}
    \end{equation*}
    which implies
    \begin{equation*}
        (3LR+G)\epsilon_{\mathrm{score}}\cdot\tau \le \epsilon_{\mathrm{score}}\sqrt{\tau}\cdot \sqrt{d\log\frac{8S}{\epsilon}}\quad \mathrm{and}\quad \epsilon_{\mathrm{score}}^2 \cdot\tau\le \epsilon_{\mathrm{score}}\sqrt{\tau}\cdot \sqrt{d\log\frac{8S}{\epsilon}}.
    \end{equation*}
    Then, we have 
    \begin{equation*}
        \begin{aligned}
            \delta = &   16 \cdot\left[\frac{3\epsilon_{\mathrm{score}}}{2}\cdot \sqrt{\tau d\log\frac{8S}{\epsilon}} + \frac{(3LR+G)\epsilon_{\mathrm{score}}\cdot \tau}{2} + \frac{\epsilon_{\mathrm{score}}^2\cdot \tau}{4}\right]\\
            \le & 16\cdot\left[\frac{3\epsilon_{\mathrm{score}}}{2}\cdot \sqrt{\tau d\log\frac{8S}{\epsilon}} + \frac{\epsilon_{\mathrm{score}}}{2}\cdot \sqrt{\tau d\log\frac{8S}{\epsilon}} + \frac{\epsilon_{\mathrm{score}}}{4}\cdot \sqrt{\tau d\log\frac{8S}{\epsilon}}\right]\\
            = & 36\epsilon_{\mathrm{score}}\cdot \sqrt{\tau d\log\frac{8S}{\epsilon}}\le \frac{1}{2}
        \end{aligned}
    \end{equation*}
    which matches the requirement of Lemma~\ref{lem:lem62_zou2021faster_gene}.
    Under this condition, if we require
    \begin{equation*}
        \epsilon_{\mathrm{score}}\le \frac{c_0\rho}{32\cdot 36\cdot \sqrt{d\log \frac{8S}{\epsilon}}} = \mathcal{O}(\rho d^{-1/2})\quad \mathrm{and}\quad \epsilon_{\mathrm{energy}}\le \frac{c_0\rho \sqrt{2\tau}}{32\cdot 5} = \mathcal{O}(\rho \tau^{1/2}),
    \end{equation*}
    it makes
    \begin{equation*}
        \delta + 5\epsilon_{\mathrm{energy}}\le 36\epsilon_{\mathrm{score}}\cdot \sqrt{\tau d\log\frac{8S}{\epsilon}} + 5\epsilon_{\mathrm{energy}}\le \frac{c_0\rho \sqrt{2\tau}}{16} \le \frac{\phi}{16}
    \end{equation*}
    and satisfies the requirements shown in Lemma~\ref{lem:lem64_zou2021faster}.
    
    Then, we are able to put the results of these lemmas together to establish the convergence of Alg.~\ref{alg:inner_mala}.
    Note that if $\mu_0$ is a $\lambda$-warm start to $\mu_*$, it must be a $\lambda$-warm start to $\tilde{\mu}_*$ since $\tilde{\mu}_*(\gA)\ge \mu_*(\gA)$ for all $\gA\in \Omega$.
    Combining Lemma~\ref{lem:proj_gap_lem}, Lemma~\ref{lem:lem64_zou2021faster} and Lemma~\ref{lem:lem66_zou2021faster}, we have
    \begin{equation*}
        \begin{aligned}
            \TVD{\mu_S}{\mu_*} \le & \TVD{\mu_S}{\tilde{\mu}_S} +\TVD{\tilde{\mu}_S}{\tilde{\mu}_*} + \TVD{\tilde{\mu}_*}{\mu_*}\\
            \le & \frac{\epsilon}{4} + \left(\lambda\cdot \left(1-\frac{\phi^2}{8}\right)^S + \frac{16(\delta+5\epsilon_{\mathrm{energy}})}{\phi}\right) + \frac{\epsilon}{4}\\
            \le & \frac{\epsilon}{2} + \lambda\left(1- \frac{c_0^2\rho^2}{4} \cdot \tau\right)^S + 408\epsilon_{\mathrm{score}}\cdot \frac{\sqrt{d\log\frac{8S}{\epsilon}}}{c_0\rho} + \frac{57\epsilon_{\mathrm{energy}}}{c_0\rho\sqrt{\tau}}\\
            = & \frac{\epsilon}{2} + \lambda\left(1- \frac{c_0^2\rho^2}{4} \cdot \tau\right)^S + \tilde{\mathcal{O}}(d^{1/2}\rho^{-1}\epsilon_{\mathrm{score}}) + \mathcal{O}({\rho^{-1}\tau^{-1/2}\epsilon_\mathrm{energy}}).
        \end{aligned}
    \end{equation*}
    After combining this result with the choice of parameters shown in Lemma~\ref{lem:inner_para_choice}, the proof is completed.
\end{proof}

\begin{lemma}
    \label{lem:inner_complexity}
    Under the same assumptions and hyperparameter settings made in Theorem~\ref{thm:inner_convergence}, we use Gaussian-type initialization 
    \begin{equation*}
        \frac{\mu_0(\der \vz)}{\der \vz} \propto \exp\left(-L\|\vz\|^2 - \frac{\left\|\vx_0 - e^{-\eta}\vz\right\|^2}{2(1-e^{-2\eta})}\right).
    \end{equation*}
    If we set the iteration number as
    \begin{equation*}
        S= \tilde{\mathcal{O}}\left(L\rho^{-2}\cdot \left(d+m_2^2\right)\tau^{-1}\right),
    \end{equation*}
    the standard and projected implementation of Alg.~\ref{alg:inner_mala} can achieve
    \begin{equation*}
        \TVD{\mu_S}{\mu_*} \le \frac{3\epsilon}{4} + \tilde{\mathcal{O}}(d^{1/2}\rho^{-1}\epsilon_{\mathrm{score}}) + \mathcal{O}({\rho^{-1}\tau^{-1/2}\epsilon_\mathrm{energy}}).
    \end{equation*}
\end{lemma}
\begin{proof}
    We reformulate the target distribution $\mu_*$ and the initial distribution $\mu_0$ as follows
    \begin{equation*}
        \begin{aligned}
            &\frac{\mu_*(\der \vz)}{\der \vz} \propto \exp\left[-\left(f_{(K-k-1)\eta}(\vz) + \frac{3L\|\vz\|^2}{2}\right) - \left(\frac{\left\|\vx_0 - e^{-\eta}\vz\right\|^2}{2(1-e^{-2\eta})} - \frac{3L\|\vz\|^2}{2}\right)\right]\coloneqq \exp\left(-\phi(\vz) - \psi(\vz)\right),\\
            &\frac{\mu_0(\der \vz)}{\vz} \propto \exp\left[-L\|\vz\|^2 - \frac{3L\|\vz\|^2}{2}-\left(\frac{\left\|\vx_0 - e^{-\eta}\vz\right\|^2}{2(1-e^{-2\eta})} - \frac{3L\|\vz\|^2}{2}\right)\right] = \exp\left[ - \frac{5L\|\vz\|^2}{2}-\psi(\vz)\right].
        \end{aligned}
    \end{equation*}
    Under this condition, we have
    \begin{equation}
        \label{ineq:lambda_mid}
        \begin{aligned}
            \frac{\mu_0(\der \vz)}{\mu_*(\der \vz)} \le \frac{\int_{\R^d} \exp\left(-\phi(\vz^\prime) - \psi(\vz^\prime)\right)\der \vz^\prime}{\int_{\R^d} \exp\left(-5L/2\cdot \|\vz^\prime\|^2 - \psi(\vz^\prime)\right)\der \vz^\prime}\cdot \exp\left(\phi(\vz) - \frac{5L\|\vz\|^2}{2}\right)
        \end{aligned}
    \end{equation}
    Due to Assumption~\ref{ass:lips_score}, we have
    \begin{equation*}
        \frac{L\mI}{2} \preceq \grad^2 f_{(K-k-1)\eta}(\vz^\prime) + \frac{3L}{2} = \grad^2 \phi(\vz^\prime) \preceq \frac{5L\mI}{2}, 
    \end{equation*}
    which means
    \begin{equation*}
        \begin{aligned}
            \phi(\vz) \le \phi(\vz_*) + \frac{5L}{4}\cdot \left\|\vz-\vz_*\right\|^2\le \phi(\vz_*) + \frac{5L\|\vz\|^2}{2} + \frac{5L\|\vz_*\|^2}{2}
        \end{aligned}
    \end{equation*}
    and
    \begin{equation}
        \label{ineq:lambda_left}
        \exp\left(\phi(\vz)-\frac{5L\|\vz\|^2}{2}\right)\le  \exp\left(\phi(\vz_*) + \frac{5L\|\vz_*\|^2}{2}\right).
    \end{equation}
    Since the function $\phi(\vz)$ is strongly log-concave, it satisfies
    \begin{equation*}
        \grad \phi(\vz)\cdot \vz \ge \frac{L\|\vz\|^2}{4} - \frac{\|\grad \phi(\vzero)\|}{L}\quad \mathrm{and}\quad \phi(\vz)\ge \frac{L\|\vz\|^2}{16} + \phi(\vz_*) - \frac{\|\grad \phi(\vzero)\|^2}{2L}
    \end{equation*}
    due to Lemma~\ref{lem:sc_to_dissp} and Lemma~\ref{lem:lemmaa1_zou2021faster}.
    Under these conditions, we have
    \begin{equation}
        \label{ineq:lambda_nume_ub}
        \begin{aligned}
            & \int \exp\left[-\phi(\vz^\prime) - \psi(\vz^\prime)\right] \der \vz^\prime \le \exp\left(-\phi(\vz_*) + \frac{\|\grad \phi(\vzero)\|^2}{2L}\right)\cdot \int \exp\left[-\frac{L\|\vz^\prime\|^2}{16} - \psi(\vz^\prime)\right]\der\vz^\prime\\
            & = \exp\left(-\phi(\vz_*) + \frac{\|\grad \phi(\vzero)\|^2}{2L}\right)\cdot \int \exp\left[-\frac{23L\|\vz^\prime\|^2}{16} - \frac{\left\|\vx_0 - e^{-\eta}\vz^\prime\right\|^2}{2(1-e^{-2\eta})}\right] \der\vz^\prime
        \end{aligned}
    \end{equation}
    Besides, we have
    \begin{equation*}
        \begin{aligned}
            & \int \exp\left[-\frac{5L\|\vz^\prime\|^2}{2} - \psi(\vz^\prime)\right]\der \vz^\prime = \int \exp\left[-L\|\vz^\prime\|^2 - \frac{\left\|\vx_0 - e^{-\eta}\vz^\prime\right\|^2}{2(1-e^{-2\eta})} \right] \der\vz^\prime,
        \end{aligned}
    \end{equation*}
    which implies
    \begin{equation}
        \label{ineq:lambda_domi_lb}
        \begin{aligned}
            &\int \exp\left[-\frac{5L\|\vz^\prime\|^2}{2} - \psi(\vz^\prime)\right]\der \vz^\prime \cdot \int \exp\left[-\frac{7L\|\vz^\prime\|^2}{16}\right]\der\vz^\prime\\
            &\ge \int \exp\left[-\frac{23L\|\vz^\prime\|^2}{16} - \frac{\left\|\vx_0 - e^{-\eta}\vz^\prime\right\|^2}{2(1-e^{-2\eta})}\right] \der\vz^\prime
        \end{aligned}
    \end{equation}
    Plugging Eq.~\ref{ineq:lambda_left}, Eq.~\ref{ineq:lambda_nume_ub} and Eq.~\ref{ineq:lambda_domi_lb} into Eq.~\ref{ineq:lambda_mid}, we have
    \begin{equation}
        \label{ineq:lambda_aft}
        \frac{\mu_0(\der \vz)}{\tilde{\mu}_*(\der \vz)} \le \exp\left( \frac{5L\|\vz_*\|^2}{2} + \frac{\|\grad \phi(\vzero)\|^2}{2L}\right)\cdot \int \exp\left[-\frac{7L\|\vz^\prime\|^2}{16}\right]\der\vz^\prime.
    \end{equation}
    Due to the strong convexity of $\phi$, it has
    \begin{equation*}
        \left\|\vz_*\right\|^2 \le \frac{4\|\grad \phi(\vzero) - \grad \phi(\vz_*)\|^2}{L^2}  = \frac{4\|\grad \phi(\vzero)\|^2}{L^2}
    \end{equation*}
    and
    \begin{equation*}
        \left\|\grad \phi(\vzero)\right\|^2 = \left\|\grad f_{(K-k-1)\eta}(\vzero)\right\|^2 \le 2L^2(d+m_2^2)  
    \end{equation*}
    where the inequality follows from Eq.~\ref{ineq:gradf0_ub}.
    Combining with the fact
    \begin{equation*}
        \int \exp\left[-\frac{7L\|\vz^\prime\|^2}{16}\right]\der\vz^\prime = \left(\frac{16\pi}{7L}\right)^{d/2},
    \end{equation*}
    Eq.~\ref{ineq:lambda_aft} can be relaxed to
    \begin{equation*}
        \lambda\le \max_{\vz} \frac{\mu_0(\der \vz)}{\tilde{\mu}_*(\der \vz)}\le \exp\left(22L\cdot (d+m_2^2)\right)\cdot \left(\frac{16\pi}{L}\right)^{d/2} = \exp\left(\mathcal{O}(L(d+m_2^2))\right)
    \end{equation*}
    which is independent on $\|\vx_0\|$.
    Then, In order to ensure the convergence of the total variation distance is smaller than $\epsilon$, it suffices to choose $\tau$ and $S$ such that    
    \begin{equation*}
         \lambda\left(1- \frac{c_0^2\rho^2}{4} \cdot \tau\right)^S \le \frac{\epsilon}{4}\quad \Leftrightarrow\quad S = \mathcal{O}\left(\frac{\log (\lambda/\epsilon)}{\rho^2 \tau}\right) = \tilde{\mathcal{O}}\left(L\rho^{-2}\cdot \left(d+m_2^2\right)\tau^{-1}\right),
    \end{equation*}
    where the last two inequalities follow from Theorem~\ref{thm:inner_convergence}.
    Hence, the proof is completed.
\end{proof}

\begin{theorem}
    \label{thm:nn_estimate_complexity_gene}
    Under Assumption~\ref{ass:lips_score}--\ref{ass:second_moment}, for Alg.~\ref{alg:rtk}, we choose 
    \begin{equation*}
        \eta= \frac{1}{2}\log \frac{2L+1}{2L} \quad \mathrm{and}\quad K = 4L\cdot \log \frac{(1+L^2)d+\left\|\grad f_*(\vzero)\right\|^2}{\epsilon^2}
    \end{equation*}
    and implement Step 3 of Alg.~\ref{alg:rtk} with projected Alg.~\ref{alg:inner_mala}.
    For the $k$-th run of Alg.~\ref{alg:inner_mala}, we use Gaussian-type initialization 
    \begin{equation*}
        \frac{\mu_0(\der \vz)}{\der \vz} \propto \exp\left(-L\|\vz\|^2 - \frac{\left\|\hat{\vx}_k - e^{-\eta}\vz\right\|^2}{2(1-e^{-2\eta})}\right).
    \end{equation*}
    If we set the hyperparameters as shown in Lemma~\ref{lem:inner_complexity}, it can achieve
    \begin{equation*}
        \TVD{\hat{p}_{K\eta}}{p_*}\le \epsilon + \tilde{\mathcal{O}}(L d^{1/2}\rho^{-1}\epsilon_{\mathrm{score}}) + \mathcal{O}(\hat{\tau}^{-1/2}\cdot L^2 (d^{1/2}+m_2+Z) {\rho^{-1}\epsilon_\mathrm{energy}})
    \end{equation*}
    with a gradient complexity as follows
    \begin{equation*}
        \tilde{\mathcal{O}}\left(L^4 \rho^{-2}\hat{\tau}^{-1}\cdot \left(d+m_2^2\right)^2 Z^2 \right)
    \end{equation*}
    for any $\hat{\tau}\in (0,1)$ where $Z$ denotes the maximal $l_2$ norm of particles appearing in outer loops (Alg.~\ref{alg:rtk}).
\end{theorem}
\begin{proof}
    According to Lemma~\ref{lem:tv_inner_conv}, we know that under the choice 
    \begin{equation*}
        \eta = \frac{1}{2}\ln\frac{2L+1}{2L},
    \end{equation*}
    it requires to run Alg.~\ref{alg:inner_mala} for $K$ times where
    \begin{equation*}
        K = 4L\cdot \log \frac{(1+L^2)d+\left\|\grad f_*(\vzero)\right\|^2}{\epsilon^2}.
    \end{equation*}
    For each run of Alg.~\ref{alg:inner_mala}, we require the total variation error to achieve
    \begin{equation*}
        \begin{aligned}
            \TVD{\hat{p}_{(k+1)\eta|k\eta}(\cdot|\hat{\vx})}{\bkwp_{(k+1)\eta|k\eta}(\cdot|\hat{\vx})} \le & \frac{\epsilon}{4L}\cdot \left[\log \frac{(1+L^2)d+\left\|\grad f_*(\vzero)\right\|^2}{\epsilon^2}\right]^{-1} \\
            & + \tilde{\mathcal{O}}(d^{1/2}\rho^{-1}\epsilon_{\mathrm{score}}) + \mathcal{O}({\rho^{-1}\tau_k^{-1/2}\epsilon_\mathrm{energy}}).
        \end{aligned}
    \end{equation*}
    Combining with Lemma~\ref{lem:inner_complexity}, we consider a step size
    \begin{equation*}
        \begin{aligned}
            \tau_k = & C_\tau\cdot \left(L^2\left(d+m_2^2+\|\hat{\vx}_k\|^2\right)\cdot \log\frac{48LS\log \frac{(1+L^2)d+\left\|\grad f_*(\vzero)\right\|^2}{\epsilon^2} }{\epsilon}\right)^{-1}\cdot \hat{\tau}\\
            =& \tilde{\mathcal{O}}(L^{-2}\cdot (d+m_2^2+\|\hat{\vx}_k\|^2)^{-1}\cdot \hat{\tau})
        \end{aligned}
    \end{equation*}
    where $\tau^\prime\in (0,1)$, to solve the $k$-th inner sampling subproblem.
    Then, the maximum iteration number will be
    \begin{equation*}
        S= \tilde{\mathcal{O}}\left(L^3\rho^{-2}\hat{\tau}^{-1}\cdot \left(d+m_2^2\right)^2\cdot \|\hat{\vx}_k\|^2\right).
    \end{equation*}
    This means that with the total gradient complexity
    \begin{equation*}
        \begin{aligned}
            K\cdot S = \tilde{\mathcal{O}}\left(L^4 \rho^{-2}\hat{\tau}^{-1}\cdot \left(d+m_2^2\right)^2 Z^2 \right)
        \end{aligned}
    \end{equation*}
    where $Z$ denotes the maximal $l_2$ norm of particles appearing in outer loops (Alg.~\ref{alg:rtk}),
    we can obtain
    \begin{equation*}
        \begin{aligned}
            \TVD{\hat{p}_{K\eta}}{p_*}\le & \epsilon + \tilde{\mathcal{O}}(K d^{1/2}\rho^{-1}\epsilon_{\mathrm{score}}) + \mathcal{O}(K L (d+m_2^2+Z^2)^{1/2} \hat{\tau}^{-1/2} {\rho^{-1}\epsilon_\mathrm{energy}}) \\
            = & \epsilon + \tilde{\mathcal{O}}(L d^{1/2}\rho^{-1}\epsilon_{\mathrm{score}}) + \mathcal{O}(\hat{\tau}^{-1/2}\cdot L^2 (d^{1/2}+m_2+Z) {\rho^{-1}\epsilon_\mathrm{energy}}).
        \end{aligned}
    \end{equation*}
    Hence, the proof is completed.
\end{proof}

\begin{lemma}
    \label{lem:particles_z_bound}
    Suppose we implement Alg.~\ref{alg:inner_mala} with its projected version, we have 
    \begin{equation*}
        Z^2 \le  \tilde{\mathcal{O}}\left(L^3(d+m_2^2)^2 \rho^{-2}\right).
    \end{equation*}
    where $Z$ denotes the maximal $l_2$ norm of particles appearing in outer loops (Alg.~\ref{alg:rtk})
\end{lemma}
\begin{proof}
    Suppose we implement Alg.~\ref{alg:inner_mala} with its projected version, where each update will be projected to a ball with a ratio $r$ shown in Lemma~\ref{lem:proj_gap_lem}.
    Under these conditions, we have
    \begin{equation*}
        \begin{aligned}
            \left\|\hat{\vx}_K\right\|^2 = & \left\|\hat{\vx}_0 + \sum_{i=1}^{K} (\hat{\vx}_{i} - \hat{\vx}_{i-1})\right\|^2\le (K+1)\left\|\hat{\vx}_0\right\|^2 + (K+1)\cdot \sum_{i=1}^K \left\|\hat{\vx}_i - \hat{\vx}_{i-1}\right\|^2
        \end{aligned}
    \end{equation*}
    For each $i\in\{1,2,\ldots K\}$, we have
    \begin{equation*}
        \left\|\hat{\vx}_i - \hat{\vx}_{i-1}\right\|^2 = \left\|\vz_S - \vz_0\right\|^2 \le (S+1)\cdot \sum_{j=1}^S\left\|\vz_{j} - \vz_{j-1}\right\|^2 \le 2S\cdot r^2.
    \end{equation*}
    Follows from Lemma~\ref{lem:inner_complexity}, it has
    \begin{equation*}
        S\cdot r^2 = \mathcal{O}\left(\frac{\log (\lambda/\epsilon)}{\rho^2\tau}\right) \cdot \tilde{\mathcal{O}}\left(\tau d\right) = \tilde{\mathcal{O}}\left(\frac{d\log (\lambda/\epsilon)}{\rho^2}\right) =  \tilde{\mathcal{O}}\left(L(d+m_2^2)^2 \rho^{-2}\right).
    \end{equation*}
    Then, we have
    \begin{equation*}
        Z^2 \le \mathcal{O}(K^2)\cdot \tilde{\mathcal{O}}\left(L(d+m_2^2)^2 \rho^{-2}\right) = \tilde{\mathcal{O}}\left(L^3(d+m_2^2)^2 \rho^{-2}\right),
    \end{equation*}
    Hence, the proof is completed.
\end{proof}

\subsection{Control the error from Energy Estimation}

\begin{corollary}
    \label{cor:complexity_19}
    Suppose the diffusion model $\vs_{\vtheta}$ satisfies 
    \begin{equation*}
        \left\|\hat{\vs}_{\vtheta}(\vx,t) + \grad\log p_t(\vx)\right\|_{\infty} \le \frac{\rho\epsilon}{Ld^{1/2}},
    \end{equation*}
    and another parameterized model $\hat{l}_{\hat{\vtheta}}(\vx, t)$ is used to estimate the log-likelihood of $p_t(\vx)$ satisfying
    \begin{equation*}
        \left\|\hat{l}_{\vtheta^\prime}(\vx,t) + \log p_{t}(\vx)\right\|_{\infty}\le \frac{\rho\epsilon}{L^2\cdot (d^{1/2}+m_2+Z)}.
    \end{equation*}
    If we implement Alg.~\ref{alg:rtk} with the projected version of Alg.~\ref{alg:inner_mala}, it has
    \begin{equation*}
        \TVD{\hat{p}_{K\eta}}{p_*}\le \tilde{\mathcal{O}}(\epsilon)
    \end{equation*}
    with the following gradient complexity
    \begin{equation*}
        \tilde{\mathcal{O}}\left(L^4 \rho^{-2}\cdot \left(d+m_2^2\right)^2 Z^2 \right).
    \end{equation*}
\end{corollary}
\begin{proof}
    Since we have highly accurate scores and energy estimation, we can construct $\vs_{\vtheta}$ and $r_{\vtheta^\prime}$ (shown in Eq.~\ref{def:energy_score_estimation}) for the $k$-th inner loop as follows
    \begin{equation*}
        \begin{aligned}
            &\vs_{\vtheta}(\vz) = \hat{\vs}_{\vtheta}(\vz, (K-k-1)\eta) + \frac{e^{-2\eta}\vz - e^{-\eta}\hat{\vx}_k}{1-e^{-2\eta}}\\
            & r_{\vtheta^\prime}(\vz, \vz^\prime) = \hat{l}_{\vtheta^\prime}(\vz, (K-k-1)\eta) + \frac{\left\|\hat{\vx}_k - e^{-\eta}\cdot \vz\right\|^2}{2(1-e^{-2\eta})} \\
            & \qquad\qquad\quad - \left(\hat{l}_{\vtheta^\prime}(\vz^\prime, (K-k-1)\eta) + \frac{\left\|\hat{\vx}_k - e^{-\eta}\cdot \vz^\prime\right\|^2}{2(1-e^{-2\eta})} \right).
        \end{aligned}
    \end{equation*}
    Under these conditions, we have
    \begin{equation*}
        \epsilon_{\mathrm{energy}}\le \frac{\rho\epsilon}{L^2\cdot (d^{1/2}+m_2+Z)}\quad \mathrm{and}\quad \epsilon_{\mathrm{score}}\le \frac{\rho\epsilon}{Ld^{1/2}}.
    \end{equation*}
    Plugging these results into Theorem~\ref{thm:nn_estimate_complexity_gene} and setting $\hat{\tau}=1/2$, we have
    \begin{equation*}
        \TVD{\hat{p}_{K\eta}}{p_*}\le \tilde{\mathcal{O}}(\epsilon)
    \end{equation*}
    with the following gradient complexity
    \begin{equation*}
        \tilde{\mathcal{O}}\left(L^4 \rho^{-2}\cdot \left(d+m_2^2\right)^2 Z^2 \right).
    \end{equation*}
\end{proof}

\begin{corollary}
    \label{cor:high_order_non_par_loglikelihood}
    Suppose the score estimation is extremely small, i.e.,
    \begin{equation*}
        \left\|\hat{\vs}_{\vtheta}(\vx,t) + \grad\log p_t(\vx)\right\|_{\infty} \ll \frac{\rho\epsilon}{Ld^{1/2}},
    \end{equation*}
    and the log-likelihood function of $p_t$ has a bounded $3$-order derivative, e.g.,
    \begin{equation*}
        \left\|\grad^{(3)} f(\vz)\right\|\le L,
    \end{equation*}
    we have a non-parametric estimation for log-likelihood to make we have $\TVD{\hat{p}_{K\eta}}{p_*}\le \tilde{\mathcal{O}}(\epsilon)$ with 
    \begin{equation*}
        \tilde{\mathcal{O}}\left(L^4 \rho^{-3}\cdot \left(d+m_2^2\right)^2 Z^3\cdot \epsilon\right).
    \end{equation*}
    gradient calls.
\end{corollary}
\begin{proof}
    Combining the Alg.~\ref{alg:inner_mala} and the definition of $\epsilon_{\mathrm{energy}}$ shown in Lemma~\ref{lem:lem62_zou2021faster}, we actually require to control
    \begin{equation*}
        \epsilon_{\mathrm{energy}}\coloneqq \left(g(\tilde{\vz}_s)-g(\vz_s)\right) - r_{\theta}(\tilde{\vz}_s, \vz_s)
    \end{equation*}
    for any $s\in [0, S-1]$.
    Then, we start to construct $r_{\theta}(\tilde{\vz}_s, \vz_s)$.
    Since we have
    \begin{equation*}
        g(\tilde{\vz}_s) - g(\vz_s) =  f_{(K-k-1)\eta}(\tilde{\vz}_s)+\frac{\left\|\vx_0 - \tilde{\vz}_s\cdot e^{-\eta}\right\|^2}{2(1-e^{-2\eta})} - f_{(K-k-1)\eta}({\vz}_s) - \frac{\left\|\vx_0 - {\vz}_s\cdot e^{-\eta}\right\|^2}{2(1-e^{-2\eta})},
    \end{equation*}
    we should only estimate the difference of the energy function $f_{(K-k-1)\eta}$ which will be presented as $f$ for abbreviation. 
    Besides, we define the following function
    \begin{equation*}
        h(t) = f\left( \left(\tilde{\vz}_s - \vz_s\right)\cdot t + \vz_s\right),
    \end{equation*}
    which means
    \begin{equation*}
        \begin{aligned}
            & h^{(1)}(t)\coloneqq \frac{\der h(t)}{\der t} =  \grad f\left( \left(\tilde{\vz}_s - \vz_s\right)\cdot t + \vz_s\right) \cdot \left(\tilde{\vz}_s - \vz_s\right)\\
            & h^{(2)}(t)\coloneqq \frac{\der^2 h(t)}{(\der t)^2} = \left(\tilde{\vz}_s - \vz_s\right)^\top \grad^2  f\left( \left(\tilde{\vz}_s - \vz_s\right)\cdot t + \vz_s\right) \left(\tilde{\vz}_s - \vz_s\right) 
        \end{aligned}
    \end{equation*}
    Under the high-order smoothness condition, i.e.,
    \begin{equation*}
        \left\|\grad^{3} f(\vz)\right\|\le L
    \end{equation*}
    where $\|\cdot\|$ denotes the nuclear norm, 
    then we have
    \begin{equation*}
        \left|h(1) - h(0)\right|\le \sum_{i=1}^{2} \frac{h^{(i)}(0)}{i !} + \frac{L\cdot \left\|\tilde{\vz}_s -\vz_s\right\|^3}{3!}\le \sum_{i=1}^{2} \frac{h^{(i)}(0)}{i !} + \frac{Lr^3}{3!}.
    \end{equation*}
    It means we need to approximate $h^{(i)}$ with high accuracy.
    
    For $i=1$, the ground truth $h^{(1)}(0)$ is
    \begin{equation*}
        h^{(1)}(0) = \frac{\der h(t)}{\der t} =  \grad f\left( \vz_s\right) \cdot \left(\tilde{\vz}_s - \vz_s\right)
    \end{equation*}
    we can approximate it numerically as 
    \begin{equation*}
        \tilde{h}^{(1)}(0) \coloneqq s_{\theta}(\vz_s ) \cdot \left(\tilde{\vz}_s - \vz_s\right)
    \end{equation*}
    since we have score approximation. 
    Then it has
    \begin{equation}
        \label{ineq:error_1_order}
        \delta^{(1)}(0) = h^{(1)}(0) - \tilde{h}^{(1)}(0) \le \left\|\grad f(\vz_s ) - s_{\theta}(\vz_s)\right\|\cdot \left\|\tilde{\vz}_s - \vz_s\right\| \le \epsilon_{\mathrm{score}}\cdot r.
    \end{equation}
    Then, for $i=2$, we obtain the ground truth $h^{(2)}(0)$ by
    \begin{equation*}
        h^{(1)}(t) - h^{(1)}(0) = \int_0^t h^{(2)}(\tau)\der\tau = t h^{(2)}(0) + \int_0^t h^{(2)}(\tau) - h^{(2)}(0)\der \tau,
    \end{equation*}
    which means
    \begin{equation*}
        h^{(2)}(0) = \frac{h^{(1)}(t) - h^{(1)}(0)}{t} + \frac{1}{t}\cdot \int_0^t h^{(2)}(\tau) - h^{(2)}(0)\der\tau.
    \end{equation*}
    If we use the differential to approximate $h^{(2)}(0)$, i.e.,
    \begin{equation*}
        \tilde{h}^{(2)}(0)\coloneqq \frac{\tilde{h}^{(1)}(t) - \tilde{h}^{(1)}(0)}{t},
    \end{equation*}
    we find the error term will be
    \begin{equation}
        \label{ineq:error_2_order_mid}
        \delta^{(2)}(0) = \left|h^{(2)}(0) -\tilde{h}^{(2)}(0)\right| = \left|\frac{2\delta^{(1)}}{t}+ \frac{1}{t}\cdot \int_0^t h^{(2)}(\tau) - h^{(2)}(0)\der\tau\right|.
    \end{equation}
    If we use smoothness to relax the integration term, we have
    \begin{equation*}
        \left|h^{(2)}(\tau) - h^{(2)}(0) \right| \le \left\|\grad^2  f\left( \left(\tilde{\vz}_s - \vz_s\right)\cdot \tau + \vz_s\right) - \grad^2  f(\vz_s)\right\| \cdot \left\|\tilde{\vz}_s - \vz_s\right\|^2 \le L\tau\cdot \left\|\tilde{\vz}_s - \vz_s\right\|^3,
    \end{equation*}
    which means 
    \begin{equation}
        \label{ineq:error_2_order_sec}
        \frac{1}{t}\cdot \int_0^t h^{(2)}(\tau) - h^{(2)}(0)\der\tau \le \frac{L\left\|\tilde{\vz}_s - \vz_s\right\|^3}{t}\cdot \int_0^t \tau \der\tau \le \frac{tLr^3}{2}.
    \end{equation}
    Combining Eq.~\ref{ineq:error_1_order}, Eq.~\ref{ineq:error_2_order_mid} and Eq.~\ref{ineq:error_2_order_sec}, we have
    \begin{equation*}
        \delta^{(2)}(0) \le \frac{2\epsilon_{\mathrm{score}}r}{t}+ \frac{Lr^3t}{2},
    \end{equation*}
    which means the final energy estimation error will be
    \begin{equation}
        \label{ineq:total_delta_3sm}
        \begin{aligned}
            &\left|h(1) - h(0) - \left(\tilde{h}^{(1)}(0) + \frac{\tilde{h}^{(1)}(t) - \tilde{h}^{(1)}(0)}{2t}\right)\right|\\
            & \le \frac{\delta^{(1)}(0)}{1} + \frac{\delta^{(2)}(0)}{2} + \frac{Lr^3}{3!} = \underbrace{\epsilon_{\mathrm{score}}\cdot r}_{\mathrm{Term\ 1}} + \underbrace{\frac{1}{2}\cdot\left(\frac{2\epsilon_{\mathrm{score}}r}{t} + \frac{Lr^3t}{2}\right)}_{\mathrm{Term\ 2}} + \frac{Lr^3}{6}.
        \end{aligned}
    \end{equation}
    Considering $\epsilon_{\mathrm{score}}$ is extremely small (compared with the output performance error tolerance $\epsilon$), we can choose $t$ depending on $\epsilon_{\mathrm{score}}$, e.g., $t = \sqrt{\epsilon_{\mathrm{score}}}$, to make Term 1 and Term 2 in Eq.~\ref{ineq:total_delta_3sm} diminish.
    Under this condition, the term $Lr^3/6$ will dominate RHS of Eq.~\ref{ineq:total_delta_3sm}.
    Besides, we have
    \begin{equation*}
        r = 3\cdot \sqrt{\tau d \log \frac{8S}{\epsilon}} = \tilde{O}(\tau^{1/2}d^{1/2}),
    \end{equation*}
    then we have
    \begin{equation*}
        \begin{aligned}
            \epsilon_{\mathrm{energy}} = \mathcal{O}(Lr^3) = \mathcal{O}(Ld^{3/2}\tau^{3/2}) = \tilde{\mathcal{O}}\left(L^{-2}d^{3/2}\left(d+m_2^2 + \|\hat{\vx}_k\|^2\right)^{-3/2}\hat{\tau}^{3/2}\right)
        \end{aligned}
    \end{equation*}
    where the last equation follows from the choice of $\tau$ shown in Theorem~\ref{thm:inner_convergence}.
    Then, plugging this result into Theorem~\ref{thm:nn_estimate_complexity_gene} and considering $\epsilon_{\mathrm{score}}\ll \epsilon$, we have
    \begin{equation*}
        \begin{aligned}
            \TVD{\hat{p}_{K\eta}}{p_*}\le & \epsilon + \tilde{\mathcal{O}}(L d^{1/2}\rho^{-1}\epsilon_{\mathrm{score}}) + \mathcal{O}(\hat{\tau}^{-1/2}\cdot L^2 (d^{1/2}+m_2+Z) {\rho^{-1}\epsilon_\mathrm{energy}})\\
            \le & \tilde{\mathcal{O}}(\epsilon) + \tilde{\mathcal{O}}\left(\hat{\tau}(d^{1/2}+m_2+Z)\rho^{-1}\right)
        \end{aligned}
    \end{equation*}
    with a gradient complexity as follows
    \begin{equation*}
        \tilde{\mathcal{O}}\left(L^4 \rho^{-2}\hat{\tau}^{-1}\cdot \left(d+m_2^2\right)^2 Z^2 \right).
    \end{equation*}
    Then, by choosing 
    \begin{equation*}
        \tau = \frac{\epsilon\rho}{d^{1/2}+m_2+Z},
    \end{equation*}
    we have $\TVD{\hat{p}_{K\eta}}{p_*}\le \tilde{\mathcal{O}}(\epsilon)$ with 
    \begin{equation*}
        \tilde{\mathcal{O}}\left(L^4 \rho^{-3}\cdot \left(d+m_2^2\right)^2 Z^3\cdot \epsilon\right).
    \end{equation*}
    Hence, the proof is completed.
\end{proof}
\begin{remark}
    \label{remark:high_order_appro}
    If we consider more high-order smooth, i.e.,
    \begin{equation*}
        \left\|\grad^{(u)} f(\vz)\right\|\le L,
    \end{equation*}
    with similar techniques shown in Corollary~\ref{cor:high_order_non_par_loglikelihood}, we can have the following bound, i.e.,
    \begin{equation*}
        \epsilon_{\mathrm{energy}} = \mathcal{O}(Lr^{u})
    \end{equation*}
    when $\epsilon_{\mathrm{score}}$ is extremely small.
    Under this condition, since it has
    \begin{equation*}
        r = 3\cdot \sqrt{\tau d \log \frac{8S}{\epsilon}} = \tilde{O}(\tau^{1/2}d^{1/2}),
    \end{equation*}
    we have
    \begin{equation*}
        \begin{aligned}
            \epsilon_{\mathrm{energy}} = \mathcal{O}(Lr^u) = \mathcal{O}(Ld^{u/2}\tau^{u/2}) = \tilde{\mathcal{O}}\left(L^{-u+1}d^{u/2}\left(d+m_2^2 + \|\hat{\vx}_k\|^2\right)^{-u/2}\hat{\tau}^{u/2}\right) = \tilde{\mathcal{O}}(L^{-u+1}\hat{\tau}^{u/2}).
        \end{aligned}
    \end{equation*}
    Then, plugging this result into Theorem~\ref{thm:nn_estimate_complexity_gene} and considering $\epsilon_{\mathrm{score}}\ll \epsilon$, we have
    \begin{equation*}
        \begin{aligned}
            \TVD{\hat{p}_{K\eta}}{p_*}\le & \epsilon + \tilde{\mathcal{O}}(L d^{1/2}\rho^{-1}\epsilon_{\mathrm{score}}) + \mathcal{O}(\hat{\tau}^{-1/2}\cdot L^2 (d^{1/2}+m_2+Z) {\rho^{-1}\epsilon_\mathrm{energy}})\\
            = & \tilde{\mathcal{O}}(\epsilon) + \tilde{\mathcal{O}}\left(\hat{\tau}^{(u-1)/2}L^{-u+3}(d^{1/2}+m_2+Z) \rho^{-1}\right)\\
            = & \tilde{\mathcal{O}}(\epsilon) + \tilde{\mathcal{O}}\left(\hat{\tau}^{(u-1)/2}(d^{1/2}+m_2+Z) \rho^{-1}\right)
        \end{aligned}
    \end{equation*}
    where we suppose $L\ge 1$ in the last equation without loss of generality.
    Then, by supposing 
    \begin{equation*}
        \hat{\tau}  = \frac{\epsilon^{2/(u-1)}\rho}{d^{1/2}+m_2+Z}
    \end{equation*}
    we have $\TVD{\hat{p}_{K\eta}}{p_*}\le \tilde{\mathcal{O}}(\epsilon)$ with 
    \begin{equation*}
        \tilde{\mathcal{O}}\left(L^4 \rho^{-3}\cdot \left(d+m_2^2\right)^2 Z^3\cdot \epsilon^{-2/(u-1)}\cdot 2^{u}\right)
    \end{equation*}
    where the last $2^u$ appears since the estimation of high-order derivatives requires an exponentially increasing call of score estimations.
\end{remark}

\section{Implement RTK inference with ULD}

In this section, we consider introducing a ULD to sample from $\bkwp_{k+1|k}(\vz|\vx_0)$. 
To simplify the notation, we set 
\begin{equation}
    \label{def:energy_inner_uld}
    g(\vz) \coloneqq f_{(K-k-1)\eta}(\vz)+\frac{\left\|\vx_0 - \vz\cdot e^{-\eta}\right\|^2}{2(1-e^{-2\eta})}
\end{equation}
and consider $k$ and $\vx_0$ to be fixed.
Besides, we set 
\begin{equation*}
    \bkwp(\vz|\vx_0) \coloneqq \bkwp_{k+1|k}(\vz|\vx_0) \propto \exp(-g(\vz))
\end{equation*}
According to Corollary~\ref{lem:kl_inner_conv} and Corollary~\ref{lem:tv_inner_conv}, when we choose
\begin{equation*}
    \eta= \frac{1}{2}\log \frac{2L+1}{2L},
\end{equation*}
the log density $g$ will be $L$-strongly log-concave and $3L$-smooth.

For the underdamped Langevin dynamics, we utilize a form similar to that shown in~\cite{zhang2023improved}, i.e.,
\begin{equation}
    \label{sde:prac_uld}
    \begin{aligned}
        &\der \hat{\rvz}_t = \hat{\rvv}_t\der t\\
        &\der \hat{\rvv}_t = -\gamma \hat{\rvv}_t\der t - \vs_{\vtheta}(\hat{\rvz}_{s\tau})\der t + \sqrt{2\gamma}\der \mB_t
    \end{aligned}
\end{equation}
with a little abuse of notation for $t\in[s\tau, (s+1)\tau)$.
We denote the underlying distribution of $(\hat{\rvz}_t, \hat{\rvv}_t)$ as $\hat{\pi}_t$, and the exact continuous SDE
\begin{equation*}
    \begin{aligned}
        &\der \rvz_t = \rvv_t\der t\\
        &\der \rvv_t = -\gamma \rvv_t\der t - \grad g(\rvz_t)\der t + \sqrt{2\gamma}\der \mB_t
    \end{aligned}
\end{equation*}
has the underlying distribution $(\rvz_t,\rvv_t)\sim \pi_t$.
The stationary distribution of the continuous version is defined as
\begin{equation*}
    \pi^{\gets}(\vz,\vv|\vx_0)\propto \exp\left(-g(\vz) - \frac{\left\|\vv\right\|^2}{2}\right)
\end{equation*}
where the $\vz$-marginal of $\pi^{\gets}(\cdot|\vx_0)$ is $p^{\gets}(\cdot|\vx_0)$ which is the desired target distribution of inner loops.
Therefore, by taking a small step size for the discretization and a large number of iterations, ULD will yield an approximate sample from $p^\gets(\cdot|\vx_0)$.
Besides, in the analysis of ULD, we usually consider an alternate system of coordinates 
\begin{equation*}
    (\phi, \psi) \coloneqq \mathcal{M}(\vz,\vv) \coloneqq (\vz, \vz+\frac{2}{\gamma}\vv),
\end{equation*}
their distributions of the continuous time iterates $\pi_t^{\gM}$ and the target in these alternate coordinates $\pi^{\gM}$, respectively.
Besides, we need to define log-Sobolev inequality as follows
\begin{definition}[Log-Sobolev Inequality]
    The target distribution $p_*$ satisfies the following inequality 
    \begin{equation*}
        \begin{aligned}
            &\E_{p_*}\left[g^2\log g^2\right]-\E_{p_*}[g^2]\log \E_{p_*}[g^2]
            \le 2C_{\mathrm{LSI}} \E_{p_*}\left\|\grad g\right\|^2
        \end{aligned}
    \end{equation*}
    with a constant $C_{\mathrm{LSI}}$ for all smooth function $g\colon \R^d\rightarrow \R$ satisfying $\E_{p_*}[g^2]<\infty$.
\end{definition}
\begin{remark}
    Log-Sobolev inequality is a milder condition than strong log-concavity. 
    Suppose $p$ satisfies $m$-strongly log-concavity, it satisfies $1/m$ LSI, which is proved in Lemma~\ref{lem:strongly_lsi}.
\end{remark}
\begin{definition}[Poincar\'e Inequality]
    The target distribution $p$ satisfies the following inequality 
    \begin{equation*}
        \begin{aligned}
            &\E_{\rvx\sim p}\left[\left\|g(\rvx) - \E_{\rvx\sim p} [g(\rvx)]\right\|^2\right] \le C_{\mathrm{PI}} \E_{p}\left\|\grad g\right\|^2
        \end{aligned}
    \end{equation*}
    with a constant $C_{\mathrm{PI}}$ for all smooth function $g\colon \R^d\rightarrow \R$ satisfying $\E_{p_*}[g^2]<\infty$.
\end{definition}

In the following, we mainly follow the idea of proof shown in~\cite{zhang2023improved}, which provides the convergence of KL divergence for ULD, to control the error from the sampling subproblems.

\begin{lemma}[Proposition 14 in~\cite{zhang2023improved}]
    \label{lem:prop14_zhang2023improved}
    Let $\pi_t^{\gM}$  denote the law of the continuous-time underdamped Langevin diffusion with $\gamma = c\sqrt{3L}$ for $c\ge \sqrt{2}$ in the $(\phi, \psi)$ coordinates. 
    Suppose the initial distribution $\pi_0$ has a log-Sobolev (LSI) constant (in the altered coordinates) $C_{\mathrm{LSI}}(\pi_0^{\gM})$, then $\{\pi_t^{\gM}\}$ satisfies LSI with a constant that can be uniformly upper bounded by
    \begin{equation*}
        C_{\mathrm{LSI}}(\pi_t^{\gM})\le \exp\left(-\sqrt{\frac{2L}{3}}\cdot t\right)\cdot C_{\mathrm{LSI}}(\pi_0^{\gM}) + \frac{2}{L}.
    \end{equation*}
\end{lemma}

\begin{lemma}[Adapted from Proposition 1 of~\cite{ma2021there}]
    \label{lem:prop1_ma2021there}
    Consider the following Lyapunov functional
    \begin{equation*}
        \gF(\pi^\prime, \pi^{\gets}) \coloneqq \KL{\pi^\prime}{\pi^{\gets}} + \E_{\pi^\prime}\left[\left\|\mathfrak{M}^{1/2}\grad \log \frac{\pi^\prime}{\pi^\gets}\right\|^2\right],\quad \mathrm{where}\quad \mathfrak{M} = \left[
            \begin{matrix}
                \frac{1}{12L} & \frac{1}{\sqrt{6L}}\\
                \frac{1}{\sqrt{6L}} & 4
            \end{matrix}
        \right] \otimes \mI_d.
    \end{equation*}
    For targets $\pi^\gets\propto\exp(-g)$ which are $3L$-smooth and satisfy LSI with constant $1/L$, let $\gamma=2\sqrt{6L}$. 
    Then the law $\pi_t$ of ULD satisfies
    \begin{equation*}
        \partial_t \gF(\pi_t, \pi^{\gets})\le -\frac{\sqrt{L}}{10\sqrt{6}}\cdot \gF(\pi_t, \pi^{\gets}).
    \end{equation*}
\end{lemma}

\begin{lemma}[Variant of Lemma 4.8 in~\cite{altschuler2023faster}]
    \label{lem:lem48_altschuler2023faster}
    Let $\hat{\pi}_t$ denote the law of SDE.~\ref{sde:prac_uld} and $\pi_t$ denote the law of the continuous time underdamped Langevin diffusion with the same initialization, i.e., $\hat{\pi}_0 = \pi_0$.
    If $\gamma\asymp \sqrt{L}$ and the step size $\tau$ satisfies
    \begin{equation*}
        \tau= \tilde{\mathcal{O}} \left(L^{-3/2}d^{-1/2}T^{-1/2}\right)
    \end{equation*}
    then we have 
    \begin{equation*}
        \chi^2(\hat{\pi}_T\|\pi_T) \lesssim L^{3/2}d\tau^2 T + \epsilon^2_{\mathrm{score}}L^{-1/2}T
    \end{equation*}
\end{lemma}
\begin{proof}
    The main difference of this discretization analysis is whether the score $\grad \log p_t$ can be exactly obtained or only be approximated by $\vs_{\vtheta}$.
    Therefore, in this proof, we will omit various steps the same as those shown in~\cite{altschuler2023faster}.

    We consider the following difference
    \begin{equation*}
        \begin{aligned}
            G_T\coloneqq &\frac{1}{\sqrt{2\gamma}} \sum_{s=0}^{S-1}\int_{s\tau}^{(s+1)\tau} \left<\grad g(\vz_t) - \vs_{\vtheta}(\vz_{s\tau}),\der\mB_t\right>\\
            & - \frac{1}{4\gamma}\sum_{s=0}^{S-1}\int_{s\tau}^{(s+1)\tau} \left\|\grad g(\vz_t) - \vs_{\vtheta}(\vz_{s\tau}) \right\|^2 \der t.
        \end{aligned}
    \end{equation*}
    From Girsanov’s theorem, we obtain immediately using It\^o's formula
    \begin{equation*}
        \begin{aligned}
            \E_{\pi_T}\left[\left(\frac{\der\hat{\pi}_T}{\der \pi_T}\right)^2\right] - 1 = & \E\left[\exp\left(2G_T\right)\right] - 1= \frac{1}{2\gamma} \E_{\pi_T}\sum_{s=0}^{S-1} \left[\int_{s\tau}^{(s+1)\tau} \exp(2G_t)\left\|\grad g(\vz_t) - \vs_{\vtheta}(\vz_{s\tau})\right\|^2\right]\\
            \le & \frac{1}{\gamma} \cdot \sum_{s=0}^{S-1} \int_{s\tau}^{(s+1)\tau} \sqrt{\E\left[\exp(4G_t)\right] \cdot \E\left[\left\|\grad g(\vz_t) - \vs_{\vtheta}(\vz_{s\tau})\right\|^4\right]}\der t\\
            \le & \frac{4}{\gamma} \sum_{s=0}^{S-1} \cdot \int_{s\tau}^{(s+1)\tau} \sqrt{\E\left[\exp(4G_t)\right] \cdot \E\left[\left\|\grad g(\vz_t) - \grad g(\vz_{s\tau})\right\|^4\right]}\der t\\
            & + \frac{4\epsilon^2_{\mathrm{score}}}{\gamma} \sum_{s=0}^{S-1}\int_{s\tau}^{(s+1)\tau} \sqrt{\E\left[\exp(4G_t)\right]} \der t
        \end{aligned}
    \end{equation*}
    According to Corollary 20 of~\cite{zhang2023improved}, we have
    \begin{equation}
        \label{ineq:exp_G_upb}
        \begin{aligned}
            \E\left[\exp(4G_t)\right] \le & \sqrt{\E\left[\exp\left(\frac{16}{\gamma} \sum_{s=0}^{S-1} \int_{s\tau}^{(s+1)\tau\wedge t} \left\|\grad g(\vz_r) - \vs_{\vtheta}(\vz_{s\tau})\right\|^2\der r\right)\right] }\\
            \le & \sqrt{\E \exp\left[\frac{32}{\gamma}\cdot \sum_{s=0}^{S-1} \left(\int_{s\tau}^{(s+1)\tau \wedge t}  \left\|\grad g(\vz_r) - \grad g(\vz_{s\tau})\right\|^2\der r + \int_{s\tau}^{(s+1)\tau \wedge t} \epsilon^2_{\mathrm{score}}\der r\right)\right]}\\
            = & \exp\left(\frac{16t\epsilon^2_{\mathrm{score}}}{\gamma}\right)\cdot \sqrt{\E \exp\left[\frac{32}{\gamma}\cdot \sum_{s=0}^{S-1} \int_{s\tau}^{(s+1)\tau\wedge t}  \left\|\grad g(\vz_r) - \grad g(\vz_{s\tau})\right\|^2\der r\right]}\\
            \le & 3\cdot \sqrt{\E \exp\left[\frac{32}{\gamma}\cdot \sum_{s=0}^{S-1} \int_{s\tau}^{(s+1)\tau \wedge t}  \left\|\grad g(\vz_r) - \grad g(\vz_{s\tau})\right\|^2\der r\right]},
        \end{aligned}
    \end{equation}
    where the last inequality can be established by requiring
    \begin{equation*}
        \epsilon_{\mathrm{score}}= \mathcal{O}\left(\gamma^{1/2}T^{-1/2}\right)\quad \Rightarrow\quad  \frac{16t\epsilon^2_{\mathrm{score}}}{\gamma}\le 1
    \end{equation*}
    since $\exp(u)\le 1+2u$ for any $u\in[0,1]$.

    With similar techniques utilized in Lemma 4.8 of~\cite{altschuler2023faster}, we know that if
    \begin{equation*}
        \gamma\asymp\sqrt{3L},\quad  \tau\lesssim \frac{\gamma^{1/2}}{6L\cdot d^{1/3} T^{1/2} (\log S)^{1/2}},\quad \mathrm{and}\quad T\gtrsim \frac{\sqrt{3L}}{L} = \sqrt{\frac{3}{L}},
    \end{equation*}
    it holds that
    \begin{equation*}
        \E \exp\left[\frac{32}{\gamma}\cdot \sum_{s=0}^{S-1} \int_{s\tau}^{(s+1)\tau \wedge t}  \left\|\grad g(\vz_r) - \grad g(\vz_{s\tau})\right\|^2\der r\right] \le \exp\left(\mathcal{O}\left(L^{3/2}d\tau^2 T \log S\right)\right).
    \end{equation*}
    Furthermore, for
    \begin{equation*}
        \tau\lesssim L^{-3/2}d^{-1/2}T^{-1/2}(\log S)^{-1/2},
    \end{equation*}
    it has
    \begin{equation*}
        \sup_{t\in[0,T]}\ \E\left[\exp(4G_t)\right]\lesssim 1.
    \end{equation*}
     Then, still with similar techniques utilized in Lemma 4.8 of~\cite{altschuler2023faster}, we have
     \begin{equation*}
         \begin{aligned}
             \sqrt{\E\left[\left\|\grad g(\vz_t) - \grad g(\vz_{s\tau})\right\|^4\right]}\le (3L)^2\sqrt{\E\left[\left\|\vz_t - \vz_{s\tau}\right\|^4\right]}\lesssim L^{2}d\tau^2.
         \end{aligned}
     \end{equation*}
     In summary, we have
     \begin{equation*}
         \E_{\pi_T}\left[\left(\frac{\der\hat{\pi}_T}{\der \pi_T}\right)^2\right] - 1 \lesssim \frac{L^2d\tau^2 T}{\gamma} + \frac{\epsilon^2_{\mathrm{score}}T}{\gamma},
     \end{equation*}
     and the proof is completed.
\end{proof}

\begin{corollary}
    Under the same assumptions and hyperparameter settings made in Lemma~\ref{lem:lem48_altschuler2023faster}. 
    If the step size $\tau$ and the score estimation error $\epsilon_{\mathrm{score}}$ satisfies
    \begin{equation*}
        \tau = \tilde{\Theta}\left( \frac{\epsilon}{L^{3/4}d^{1/2}T^{1/2}}\right)\quad \mathrm{and}\quad \epsilon_{\mathrm{score}} = \mathcal{O}\left(T^{-1/2}\epsilon\right)
    \end{equation*}
    Then we have $\chi^2(\hat{\pi}_T\|\pi_T) \lesssim \epsilon^2$.
\end{corollary}
\begin{proof}
    We can easily obtain this result by plugging the choice of $\tau$ and $\epsilon$ into Lemma~\ref{lem:lem48_altschuler2023faster}. Noted that we suppose $L\ge 1$ without loss of generality.
\end{proof}

\begin{theorem}[Variant of Theorem 6 in~\cite{zhang2023improved}]
    \label{thm:inner_convergence_uld}
    Under Assumption~\ref{ass:lips_score}--\ref{ass:second_moment}, for any $\epsilon\in(0,1)$, we require Gaussian-type initialization and high-accurate score estimation, i.e.,
    \begin{equation*}
        \hat{\pi}_0 = \mathcal{N}(\vzero, e^{2\eta}-1)\otimes \mathcal{N}(\vzero,\mI)\quad \mathrm{and}\quad \epsilon_{\mathrm{score}}=\tilde{\mathcal{O}}(\epsilon).
    \end{equation*}
    If we set the step size and the iteration number as
    \begin{equation*}
         \begin{aligned}
             \tau & = \tilde{\Theta}\left(\epsilon d^{-1/2} L^{-1/2} \cdot \left(\log \left[\frac{L(d+m_2^2+\|\vx_0\|^2)}{\epsilon^2}\right]\right)^{-1/2}\right)\\
            S &=\tilde{\Theta}\left(\epsilon^{-1}d^{1/2}\cdot \left(\log \left[\frac{L(d+m_2^2+\|\vx_0\|^2)}{\epsilon^2}\right]\right)^{1/2}\right).
         \end{aligned}
     \end{equation*}
    the marginal distribution of output particles $\hat{p}_T$ will satisfy $\KL{\hat{p}_T}{p^{\gets}(\cdot|\vx_0)}\le \mathcal{O}(\epsilon^2)$.
\end{theorem}
\begin{proof}
    Consider the underlying distribution of the twisted coordinates $(\phi,\psi)$ for SDE.~\ref{sde:prac_uld}, the decomposition of the KL using Cauchy–Schwarz:
    \begin{equation}
        \label{ineq:kl_decompo}
        \begin{aligned}
            \KL{\hat{\pi}_T^{\gM}}{\pi^{\gM}} = & \int \log\frac{\hat{\pi}_T^{\gM}}{\pi^{\gM}}\der \hat{\pi}_T^{\gM} = \KL{\hat{\pi}_T^{\gM}}{\pi_T^{\gM}} + \int \log\frac{\pi_T^{\gM}}{\pi^{\gM}}\der \hat{\pi}_T^{\gM}\\
            = & \KL{\hat{\pi}_T^{\gM}}{\pi_T^{\gM}} + \KL{\pi_T^{\gM}}{\pi^{\gM}} + \int \log\frac{\pi_T^{\gM}}{\pi^{\gM}}\der (\hat{\pi}_T^{\gM}-\pi_T^{\gM})\\
            = & \KL{\hat{\pi}_T^{\gM}}{\pi_T^{\gM}} + \KL{\pi_T^{\gM}}{\pi^{\gM}} + \sqrt{\chi^2\left(\hat{\pi}_T^{\gM}\|\pi_{T}^{\gM}\right)\times \mathrm{var}_{\pi_T^{\gM}} \left(\log \frac{\pi_T^{\gM}}{\pi^{\gM}}\right) }.
        \end{aligned}
    \end{equation}
    Using LSI of the iterations via Lemma~\ref{lem:prop14_zhang2023improved}, we have
    \begin{equation*}
        \mathrm{var}_{\pi_T^{\gM}} \left(\log \frac{\pi_T^{\gM}}{\pi^{\gM}}\right)\le C_{\mathrm{LSI}}(\pi_T^{\gM})\cdot \E_{\pi_T^{\gM}}\left[\left\|\grad \log \frac{\pi_T^{\gM}}{\pi^{\gM}}\right\|^2\right]\lesssim \frac{1}{L}\cdot \E_{\pi_T^{\gM}}\left[\left\|\grad \log \frac{\pi_T^{\gM}}{\pi^{\gM}}\right\|^2\right].
    \end{equation*}
    Then, we start to upper bound the relative Fisher information.
    Since $\pi^{\gM} = \gM_{\#}\pi^{\gets}(\cdot|\vx_0)$, then 
    \begin{equation*}
        \pi^{\gM}(\phi,\psi) \propto \pi^{\gets}(\gM^{-1}(\phi,\psi)|\vx_0).
    \end{equation*}
    Therefore, we have
    \begin{equation*}
        \grad \log \pi^{\gM}  = (\gM^{-1})^\top \grad \log \pi^{\gets}(\cdot|\vx_0)\circ \gM^{-1},
    \end{equation*}
    and similarly for $\grad \log \pi_T^{\gM}$.
    This yields the expression
    \begin{equation}
        \label{ineq:fisher_bound_1}
        \E_{\pi_T^{\gM}}\left[\left\|\grad \log \frac{\pi_T^{\gM}}{\pi^{\gM}}\right\|^2\right] = \E_{\pi_T}\left[\left\|(\gM^{-1})^\top \grad\log \frac{\pi_T}{\pi^{\gets}}\right\|^2\right].
    \end{equation}
    According to the definition of $\gM$, we have
    \begin{equation*}
        \gM^{-1}(\gM^{-1})^\top = 
        \left[
            \begin{matrix}
            1 & -\gamma/2\\
            -\gamma/2 & \gamma^2/2
            \end{matrix}
        \right].
    \end{equation*}
    For any $c_0>0$ and 
    \begin{equation*}
        \mathfrak{M}\coloneqq 
        \left[
            \begin{matrix}
                \frac{1}{12L} & \frac{1}{\sqrt{6L}}\\
                \frac{1}{\sqrt{6L}} & 4
            \end{matrix}
        \right] \otimes \mI_d,
    \end{equation*}
    we have
    \begin{equation*}
        L\mathfrak{M}-c_0\gM^{-1}(\gM^{-1})^\top  = 
        \left[
            \begin{matrix}
                1/4-c_0 & \sqrt{3L}(1/\sqrt{2}+c_0\sqrt{2})\\
                \sqrt{3L}(1/\sqrt{2}+c_0\sqrt{2}) & 3L(4-c_0)
            \end{matrix}
        \right].
    \end{equation*}
    The determinant is 
    \begin{equation*}
        3L\cdot\left[\left(\frac{1}{4}- c_0\right)\cdot (4-c_0) - \left(\frac{1}{\sqrt{2}}+c_0\sqrt{2}\right)^2\right]>0
    \end{equation*}
    for $c_0>0$ sufficiently small, which means that
    \begin{equation*}
        \gM^{-1}(\gM^{-1})^\top \preceq c_0^{-1}L\mathfrak{M}.
    \end{equation*}
    Therefore, Eq.~\ref{ineq:fisher_bound_1} becomes
    \begin{equation*}
        \E_{\pi_T^{\gM}}\left[\left\|\grad \log \frac{\pi_T^{\gM}}{\pi^{\gM}}\right\|^2\right] \lesssim 3L\cdot \E_{\pi_T}\left[\left\|\mathfrak{M}^{1/2}\grad\log \frac{\pi_T}{\pi^{\gets}}\right\|^2\right].
    \end{equation*}
    According to Lemma~\ref{lem:prop1_ma2021there},  the  decay of the Fisher information requires us to set
    \begin{equation}
        \label{ineq:inner_mixing_uld}
        T\gtrsim L^{-1/2}\cdot \log\left[\epsilon^{-2}\cdot \left(\KL{\pi_0}{\pi^{\gets}} + \E_{\pi_0}\left(\left\|\mathfrak{M}^{1/2}\grad\log \frac{\pi_0}{\pi^{\gets}}\right\|^2\right)\right)\right],
    \end{equation}
    which yields $\KL{\pi_T^{\gM}}{\pi^{\gM}}\le \epsilon^2$.
    Besides, we can easily have
    \begin{equation*}
        \E_{\pi_0}\left(\left\|\mathfrak{M}^{1/2}\grad\log \frac{\pi_0}{\pi^{\gets}}\right\|^2\right)\lesssim \frac{1}{3L}\cdot \FI{\pi_0}{\pi^{\gets}} = \frac{1}{3L}\cdot \E_{\pi_0}\left(\left\|\grad\log \frac{\pi_0}{\pi^{\gets}}\right\|^2\right).
    \end{equation*}
    According to the definition of LSI, we also have
    \begin{equation*}
        \KL{\pi_0}{\pi^{\gets}} \le \frac{C_{\mathrm{LSI}}}{2}\cdot \FI{\pi_0}{\pi^{\gets}} = \frac{1}{2L}\cdot \E_{\pi_0}\left(\left\|\grad\log \frac{\pi_0}{\pi^{\gets}}\right\|^2\right).
    \end{equation*}
    Recall as well that this requires $\gamma\asymp \sqrt{3L}$ in SDE.~\ref{sde:prac_uld}.
     For the remaining $\KL{\hat{\pi}_T^{\gM}}{\pi_T^{\gM}}$ and $\chi^2\left(\hat{\pi}_T^{\gM}\|\pi_{T}^{\gM}\right)$ in Eq.~\ref{ineq:kl_decompo},  we invoke Lemma~\ref{lem:lem48_altschuler2023faster} with the value $T=S\tau$ specified and desired accuracy $\epsilon$, , which consequently yields
     \begin{equation}
        \label{ineq:inner_iter_num_uld}
         \tau = \tilde{\Theta}\left(\frac{\epsilon}{L^{3/4}d^{1/2}T^{1/2}}\right)\quad \mathrm{and}\quad S=\tilde{\Theta}\left(\frac{T^{3/2}L^{3/4}d^{1/2}}{\epsilon}\right).
     \end{equation}
     Under this condition, we start to consider the initialization error.
     Suppose we have $\pi_0= \mathcal{N}(\vzero, e^{2\eta}-1)\otimes \mathcal{N}(\vzero,\mI)$, which implies
     \begin{equation*}
         \begin{aligned}
             \FI{\pi_0}{\pi^{\gets}} \lesssim & \E_{\pi_0}\left[\left\|\grad f_{(K-k-1)\eta}(\rvz) - \grad f_{(K-k-1)\eta}(\vzero)+\grad f_{(K-k-1)\eta}(\vzero) - \frac{e^{-\eta}\vx_0}{1-e^{-2\eta}} \right\|^2\right]\\
             \le & 3L^2\E_{\pi_0}[\|\rvz\|^2] + 3\left\|\grad f_{(K-k-1)\eta}(\vzero)\right\|^2 + \frac{3e^{-2\eta}}{(1-e^{-2\eta})^2}\cdot \left\|\vx_0\right\|^2\\
             = & 3L^2\cdot \left(e^{2\eta}-1\right)+ 3\left\|\grad f_{(K-k-1)\eta}(\vzero)\right\|^2 +  \frac{3e^{-2\eta}}{(1-e^{-2\eta})^2}\cdot \left\|\vx_0\right\|^2
         \end{aligned}
     \end{equation*}
     Following the $\eta$ setting, i.e., 
     \begin{equation*}
         \eta= \frac{1}{2}\log \frac{2L+1}{2L}\quad \Leftrightarrow\quad e^{2\eta} = \frac{2L+1}{2L},
     \end{equation*}
     which yields
     \begin{equation}
        \label{ineq:error_init_uld}
         \begin{aligned}
             \FI{\pi_0}{\pi^{\gets}}  \lesssim & L + \left\|\grad f_{(K-k-1)\eta}(\vzero)\right\|^2 + L^2\|\vx_0\|^2\\
             \lesssim & L+ L^2(d+m_2^2) + L^2 \|\vx_0\|^2
         \end{aligned}
     \end{equation}
     where the inequality follows from Eq.~\ref{ineq:gradf0_ub}.
     Therefore, combining Eq.~\ref{ineq:error_init_uld}, Eq.~\ref{ineq:inner_iter_num_uld} and Eq.~\ref{ineq:inner_mixing_uld}, we have
     \begin{equation*}
         T^{1/2} \gtrsim L^{-1/4}\cdot \left(\log \left[\frac{L(d+m_2^2+\|\vx_0\|^2)}{\epsilon^2}\right]\right)^{1/2} \gtrsim L^{1/4}\cdot \left(\log \left[\frac{\E_{\pi_0}\left(\left\|\grad\log \frac{\pi_0}{\pi^{\gets}}\right\|^2\right)}{L \epsilon^2}\right]\right)^{1/2}, 
     \end{equation*}
     which implies
     \begin{equation*}
         \begin{aligned}
             \tau & = \tilde{\Theta}\left(\epsilon d^{-1/2} L^{-1/2} \cdot \left(\log \left[\frac{L(d+m_2^2+\|\vx_0\|^2)}{\epsilon^2}\right]\right)^{-1/2}\right)\\
            S &=\tilde{\Theta}\left(\epsilon^{-1}d^{1/2}\cdot \left(\log \left[\frac{L(d+m_2^2+\|\vx_0\|^2)}{\epsilon^2}\right]\right)^{1/2}\right).
         \end{aligned}
     \end{equation*}
     In this condition, the score estimation error is required to be
     \begin{equation*}
         \epsilon_{\mathrm{score}} = \mathcal{O}\left(\gamma^{1/2}T^{-1/2}\cdot \epsilon\right) = \tilde{\mathcal{O}}\left(\epsilon/\sqrt{L}\right).
     \end{equation*}
     Hence, the proof is completed.
\end{proof}

\begin{theorem}
    \label{thm:uld_outer_complexity_gene}
    Under Assumption~\ref{ass:lips_score}--\ref{ass:second_moment}, for Alg.~\ref{alg:rtk}, we choose 
    \begin{equation*}
        \eta= \frac{1}{2}\log \frac{2L+1}{2L} \quad \mathrm{and}\quad K = 4L\cdot \log \frac{(1+L^2)d+\left\|\grad f_*(\vzero)\right\|^2}{\epsilon^2}
    \end{equation*}
    and implement Step 3 of Alg.~\ref{alg:rtk} with projected Alg.~\ref{alg:inner_uld}.
    For the $k$-th run of Alg.~\ref{alg:inner_uld}, we require Gaussian-type initialization and high-accurate score estimation, i.e.,
    \begin{equation*}
        \hat{\pi}_0 = \mathcal{N}(\vzero, e^{2\eta}-1)\otimes \mathcal{N}(\vzero,\mI)\quad \mathrm{and}\quad \epsilon_{\mathrm{score}}=\tilde{\mathcal{O}}(\epsilon).
    \end{equation*}
    If we set the hyperparameters as shown in Lemma~\ref{thm:inner_convergence_uld}, it can achieve $\TVD{\hat{p}_{K\eta}}{p_*}\lesssim \epsilon$
    with an $\tilde{\mathcal{O}}\left(L^2 d^{1/2}\epsilon^{-1} \right)$ gradient complexity.
\end{theorem}

\begin{proof}
    According to Corollary~\ref{lem:kl_inner_conv}, we know that under the choice 
    \begin{equation*}
        \eta = \frac{1}{2}\ln\frac{2L+1}{2L},
    \end{equation*}
    it requires to run Alg.~\ref{alg:inner_uld} for $K$ times where
    \begin{equation*}
        K = 4L\cdot \log \frac{(1+L^2)d+\left\|\grad f_*(\vzero)\right\|^2}{\epsilon^2}.
    \end{equation*}
    For each run of Alg.~\ref{alg:inner_uld}, we require the KL divergence error to achieve
    \begin{equation*}
        \begin{aligned}
            \KL{\hat{p}_{(k+1)\eta|k\eta}(\cdot|\hat{\vx})}{\bkwp_{(k+1)\eta|k\eta}(\cdot|\hat{\vx})} \le & \frac{\epsilon^2}{4L}\cdot \left[\log \frac{(1+L^2)d+\left\|\grad f_*(\vzero)\right\|^2}{\epsilon^2}\right]^{-1}.
        \end{aligned}
    \end{equation*}
    Combining with Theorem~\ref{thm:inner_convergence_uld}, we consider a step size
    \begin{equation*}
        \begin{aligned}
            \tau_k = & \tilde{\mathcal{O}}\left(L^{-1}d^{-1/2}\epsilon\cdot (\log\left[L^2\cdot(d+m_2^2+\|\hat{\vx}_k\|^2)\right])^{-1/2}\right)
        \end{aligned}
    \end{equation*}
    then the iteration number will be
    \begin{equation*}
        S_k= \tilde{\mathcal{O}}\left(L^{1/2}d^{1/2}\epsilon^{-1}\cdot (\log\left[L^2\cdot(d+m_2^2+\|\hat{\vx}_k\|^2)\right])^{1/2}\right).
    \end{equation*}
    For an expectation perspective, we have
    \begin{equation*}
        \E_{\hat{p}_{k\eta}}\left[\log (L^2\|\hat{\rvx}_k\|^2)\right]\le \log\left[\E_{\hat{p}_{k\eta}}(\|\hat{\rvx}_k\|^2)\right] = \tilde{\mathcal{O}}(L)
    \end{equation*}
    where the last inequality follows from Lemma~\ref{lem:second_moment_bound}.
    This means that with the total gradient complexity
    \begin{equation*}
        \begin{aligned}
            K\cdot S = \tilde{\mathcal{O}}\left(L^2 d^{1/2}\epsilon^{-1} \right)
        \end{aligned}
    \end{equation*}
    Hence, the proof is completed.
\end{proof}

\section{Auxiliary Lemmas}

\begin{lemma}[Theorem 4 in~\cite{vempala2019rapid}]
    \label{thm4_vempala2019rapid}
    Suppose $p \propto \exp(-f)$ defined on $\R^d$ satisfies LSI with constant $\mu>0$. 
    Along the Langevin dynamics, i.e.,
    \begin{equation*}
        \der \rvx_t = -\grad f(\rvx)\der t + \sqrt{2}\der \mB_t,
    \end{equation*}
    where $\rvx_t\sim p_t$, then it has
    \begin{equation*}
        \KL{p_t}{p} \le \exp\left(-2\mu t\right)\cdot \KL{p_0}{p}.
    \end{equation*}
\end{lemma}

\begin{lemma}
    \label{lem:init_error_bound}
    Suppose $p \propto \exp(-f)$ defined on $\R^d$ satisfies LSI with constant $\mu>0$ where $f$ is L-smooth, i.e.,
    \begin{equation*}
        \left\|\grad f(\vx^\prime)- \grad f(\vx)\right\|\le L \left\|\vx^\prime - \vx\right\|.
    \end{equation*}
    If $p_0$ is the standard Gaussian distribution defined on $\R^d$, then we have
    \begin{equation*}
        \KL{p_0}{p} \le \frac{(1+2L^2)d+2\left\|\grad f(\vzero)\right\|^2}{\mu}.
    \end{equation*}
\end{lemma}
\begin{proof}
    According to the definition of LSI, we have
    \begin{equation*}
        \begin{aligned}
            \KL{p_0}{p}\le & \frac{1}{2\mu}\int p_0(\vx)\left\|\grad \log \frac{p_0(\vx)}{p(\vx)}\right\|^2 \der \vx = \frac{1}{2\mu}\int p_0(\vx)\left\|-\vx + \grad f(\vx)\right\|^2 \der \vx\\
            \le & \mu^{-1}\cdot \left[\int p_0(\vx)\|\vx\|^2\der\vx + \int  p_0(\vx)\|\grad f(\vx)-\grad f(\vzero)+\grad f(\vzero)\|^2 \der \vx\right] \\
            \le & \mu^{-1}\cdot \left[(1+2L^2)\int p_0(\vx)\|\vx\|^2\der\vx + 2\left\|\grad f(\vzero)\right\|^2\right]\\
            = & \frac{(1+2L^2)d+2\left\|\grad f(\vzero)\right\|^2}{\mu}
        \end{aligned}
    \end{equation*}
    where the third inequality follows from the $L$-smoothness of $f_*$ and the last equation establishes since $\mathbb{E}_{p_0}[\|\vx\|^2]=d$ is for the standard Gaussian distribution $p_0$ in $\R^d$.
\end{proof}

\begin{lemma}[Variant of Lemma B.1 in~\cite{zou2021faster}]
    \label{lem:sc_to_dissp}
    Suppose $f\colon \R^d\rightarrow \R$ is a $m$-strongly convex function and satisfies $L$-smooth. Then, we have
    \begin{equation*}
        \grad f(\vx) \cdot \vx \ge \frac{m\left\|\vx\right\|^2}{2} - \frac{\|\grad f(\vzero)\|^2}{2m}
    \end{equation*}
    where $\vx_*$ is the global optimum of the function $f$.
\end{lemma}
\begin{proof}
    According to the definition of strongly convex, the function $f$ satisfies
    \begin{equation*}
        \begin{aligned}
            f(\vzero) - f(\vx) \ge \grad f(\vx)\cdot \left(\vzero - \vx\right) + \frac{m}{2}\cdot \left\|\vx\right\|^2\ \Leftrightarrow\ \grad f(\vx)\cdot \vx \ge f(\vx) - f(\vzero) + \frac{m}{2}\cdot \left\|\vx\right\|^2.
        \end{aligned}
    \end{equation*}
    Besides, we have
    \begin{equation*}
        f(\vx)-f(\vzero)\ge \grad f(\vzero)\cdot \vx + \frac{m}{2}\cdot \|\vx\|^2 \ge \frac{m}{2}\cdot \|\vx\|^2 - \frac{m}{2}\cdot \|\vx\|^2 - \frac{\|\grad f(\vzero)\|^2}{2m} = -\frac{\|\grad f(\vzero)\|^2}{2m}.
    \end{equation*}
    Combining the above two inequalities, the proof is completed.
\end{proof}

\begin{lemma}[Lemma A.1 in~\cite{zou2021faster}]
    \label{lem:lemmaa1_zou2021faster}
    Suppose a function $f$ satisfy 
    \begin{equation*}
        \grad f(\vx)\cdot \vx \ge \frac{m\|\vx\|^2}{2} - \frac{\|\grad f(\vzero)\|}{2m},
    \end{equation*}
    then we have
    \begin{equation*}
        f(\vx) \ge \frac{m}{8}\|\vx\|^2 + f(\vx_*) - \frac{\|\grad f(\vzero)\|^2}{4m}.
    \end{equation*}
\end{lemma}

\begin{lemma} [Lemma 1 in~\cite{huang2023monte}]
\label{lem:lem1_huang2023monte}
Consider the Ornstein-Uhlenbeck forward process
\begin{equation*}
    \der \rvx_t = -\rvx_t \der t + \sqrt{2}\der \mB_t,
\end{equation*}
and denote the underlying distribution of the particle $\rvx_t$ as $p_t$.
Then, the score function can be rewritten as
\begin{equation}
    \begin{aligned}
        &\grad_{\vx} \ln p_t(\vx) =  \mathbb{E}_{\vx_0 \sim q_t(\cdot|\vx)}\frac{ e^{-t}\vx_0-\vx}{\left(1-e^{-2t}\right)},\\&q_{t}(\vx_0|\vx)  \propto \exp\left(-f_{*}(\vx_0)-\frac{\left\|\vx - e^{-t}\vx_0\right\|^2}{2\left(1-e^{-2t}\right)}\right).
    \end{aligned}
\end{equation}
\end{lemma}

\begin{lemma}[Lemma 11 in~\cite{vempala2019rapid}]
    \label{lem:lem11_vempala2019rapid}
    Assume $p\propto \exp(-f)$ and the energy function $f$ is $L$-smooth. 
    Then
    \begin{equation*}
        \E_{\rvx\sim p}\left[\left\|\grad f(\rvx)\right\|^2\right]\le Ld
    \end{equation*}
\end{lemma}

\begin{lemma}[Lemma 10 in~\cite{chen2022sampling}]
    \label{lem:lem10_chen2022sampling}
    Suppose that Assumption~\ref{ass:lips_score}--\ref{ass:second_moment} hold. 
    Let $\{\rvx_t\}_{t\in[0,T]}$ denote the forward process, i.e., Eq.~\ref{sde:ou}, for all $t\ge 0$, 
        \begin{equation*}
            \E\left[\left\|\rvx\right\|^2\right]\le \max\left\{d, m_2^2\right\}.
        \end{equation*}
\end{lemma}

\begin{lemma}
    \label{lem:lsi_var_bound}
    Suppose $q$ is a distribution which satisfies LSI with constant $\mu$, then its variance satisfies
    \begin{equation*}
        \int q(\vx) \left\|\vx - \mathbb{E}_{\tilde{q}}\left[\rvx\right]\right\|^2 \der \vx  \le \frac{d}{\mu}.
    \end{equation*}
\end{lemma}
\begin{proof}
    It is known that LSI implies Poincar\'e inequality with the same constant, i.e., $\mu$, which means if for all smooth
    function $g\colon \R^d \rightarrow \R$,
    \begin{equation*}
        \mathrm{var}_{q}\left(g(\rvx)\right)\le \frac{1}{\mu}\mathbb{E}_{q}\left[\left\|\grad g(\rvx)\right\|^2\right].
    \end{equation*}
    In this condition, we suppose $\vb=\mathbb{E}_{q}[\rvx]$, and have the following equation
    \begin{equation*}
        \begin{aligned}
            &\int q(\vx) \left\|\vx - \mathbb{E}_{q}\left[\rvx\right]\right\|^2 \der \vx = \int q(\vx) \left\|\vx - \vb\right\|^2 \der \vx\\
            =& \int \sum_{i=1}^d q(\vx) \left(\vx_{i}-\vb_i\right)^2 \der \vx =\sum_{i=1}^d \int q(\vx) \left( \left<\vx, \ve_i\right> - \left<\vb, \ve_i\right> \right)^2 \der\vx\\
            =& \sum_{i=1}^d \int q(\vx)\left(\left<
            \vx,\ve_i\right> - \mathbb{E}_{q}\left[\left<\rvx, \ve_i\right>\right]\right)^2 \der \vx  =\sum_{i=1}^d \mathrm{var}_{q}\left(g_i(\rvx)\right)
        \end{aligned}
    \end{equation*}
    where $g_i(\vx)$ is defined as $g_i(\vx) \coloneqq \left<\vx, \ve_i\right>$
    and $\ve_i$ is a one-hot vector ( the $i$-th element of $\ve_i$ is $1$ others are $0$).
    Combining this equation and Poincar\'e inequality, for each $i$, we have
    \begin{equation*}
        \mathrm{var}_{q}\left(g_i(\rvx)\right) \le \frac{1}{\mu} \mathbb{E}_{q}\left[\left\|\ve_i\right\|^2\right]=\frac{1}{\mu}.
    \end{equation*}
    Hence, the proof is completed.
\end{proof}

\begin{lemma}[Variant of Lemma 10 in~\cite{cheng2018convergence}]
    \label{lem:strongly_lsi}
    Suppose $-\log p_*$ is $m$-strongly convex function, for any distribution with density function $p$, we have
    \begin{equation*}
        \KL{p}{p_*}\le \frac{1}{2m}\int p(\vx)\left\|\grad\log \frac{p(\vx)}{p_*(\vx)}\right\|^2\der\vx.
    \end{equation*}
    By choosing $p(\vx)=g^2(\vx)p_*(\vx)/\mathbb{E}_{p_*}\left[g^2(\rvx)\right]$ for the test function $g\colon \R^d\rightarrow \R$ and  $\mathbb{E}_{p_*}\left[g^2(\rvx)\right]<\infty$, we have
    \begin{equation*}
        \mathbb{E}_{p_*}\left[g^2\log g^2\right] - \mathbb{E}_{p_*}\left[g^2\right]\log \mathbb{E}_{p_*}\left[g^2\right]\le \frac{2}{m} \mathbb{E}_{p_*}\left[\left\|\grad g\right\|^2\right],
    \end{equation*}
    which implies $p_*$ satisfies $1/m$-log-Sobolev inequality.
\end{lemma}

\newpage
\end{document}